\documentclass[final,12pt]{colt2026} 

\usepackage{tikz}
\usepackage{amsmath}
\usetikzlibrary{arrows.meta, positioning}
\usetikzlibrary{calc}

\usepackage{etoc}
\usepackage{graphicx}
\usepackage{hyperref}
\usepackage{booktabs}
\usepackage[tikz]{bclogo}
\usepackage{xcolor}
\usepackage{fancybox}
\usepackage[utf8]{inputenc}
\usepackage{natbib}
\usepackage[space]{grffile}
\usepackage[T1]{fontenc}
\usepackage{newtxtext}
\usepackage[english]{babel}
\let\coltvec\vec
\let\vec\hat
\usepackage{accents}
\let\vec\coltvec\usepackage{caption}
\usepackage{mathtools}
\usepackage[capitalize]{cleveref}
\usepackage{bm}
\usepackage{algorithm}
\usepackage{algorithmic}
\usepackage{bbm}
\usepackage{lipsum}
\usepackage{makecell}

\usepackage{mathabx}
\usepackage{algorithm2e}
\RestyleAlgo{ruled}
\SetAlgoVlined

\DeclareMathOperator*{\argmax}{arg\,max}

\title[Entropic Best Policy Identification]{Tight Sample Complexity Bounds for \\
Entropic Best Policy Identification}
\usepackage{times}

\coltauthor{%
 \Name{Amer Essakine} \Email{amer.essakine@ens-paris-saclay.fr}\\
 \addr ENS Paris Saclay
 \AND
 \Name{Claire Vernade} \Email{claire.vernade@utn.de}\\
 \addr University of Technology Nuremberg %
}

\begin{document}

\maketitle

\begin{abstract}%
 We study best-policy identification for finite-horizon risk-sensitive reinforcement learning under the entropic risk measure. Recent work established a constant gap in the exponential horizon dependence between lower and upper bounds on the number of samples required to identify an approximately optimal policy. Precisely, known lower bounds scale in $\Omega(e^{|\beta| H})$ where $H$ is the horizon of the MDP, while the state-of-the-art upper bound achieves at best $O(e^{2|\beta| H})$ \citep{MortensenTalebi2025} using a generative model. We show that this extra exponential factor can be traced to overly loose concentration control for exponential utilities. To close this open gap, we revisit the analysis of this problem through a forward-model based algorithm building on KL-based exploration bonuses that we adapt to the entropic criterion. The improvement we get is due to two main novel technical innovations. We leverage the smoothness properties of the exponential utility to derive sharper concentration bounds, and we propose a new stopping rule that exploits further this tightness to obtain a sample complexity that matches the lower bound.
\end{abstract}


\section{Introduction}

Risk-sensitive Reinforcement Learning (RL) studies the problem of learning a policy whose return distribution, rather than only its expectation, satisfies desirable properties, typically with respect to downside risk or tail events. In many applications ranging from finance to robotics \citep{charpentier2021reinforcement,polydoros2017survey}, it may be more important for a decision maker to certify with high confidence that the return does not fall below a critical threshold than to maximize the expected reward.

A prominent example of a dynamically consistent risk criterion is the Entropic Risk Measure \citep{marthe2023beyond,HowardMatheson1972}, which generalizes the expectation through an exponential utility function parameterized by a scalar $\beta \in \mathbb{R}$. The sign and magnitude of $\beta$ control the risk sensitivity of the decision maker, interpolating between risk-seeking and risk-averse behaviors. Crucially, the entropic risk measure satisfies a risk-sensitive dynamic programming principle, yielding a Bellman-type recursion for finite-horizon Markov Decision Processes (MDPs) \citep{sutton2018reinforcement}. As shown by \citet{marthe2023beyond}, this property in fact characterizes the entropic family among utility-based risk measures satisfying dynamic consistency.

In RL, the MDP is unknown and the learner must rely on sampled trajectories to estimate the model and plan under uncertainty. \citet{MortensenTalebi2025} study this problem in a discounted infinite horizon setting and show that the sample complexity of identifying an $\epsilon$-optimal policy must depend exponentially on the horizon. More precisely, for a horizon $H$ and entropic parameter $\beta$, they prove a lower bound showing that any algorithm requires at least
\[
\Omega\!\left(\exp\!\left(|\beta|H\right)\right)
\]
samples (up to polynomial factors in the number of states $S$, actions $A$, and $1/\epsilon$) to identify an $\epsilon$-optimal policy. In the same work, the authors derive an upper bound in $
\exp\!\left(2|\beta|H\right)$ 
using access to a generative model and a careful application of the simulation lemma (see Section~\ref{sec:setting_background} for precise statements). While this gap may look benign at first, we show that this constant term in the exponential is due to a fundamentally loose handling of exponential utilities in exploration bonuses. 

Beyond simply closing this open gap, we provide a new set of tools to design algorithms for risk-sensitive RL. Our algorithm builds on the forward-model framework \citep{Menard2021Fast} and tailors the exploration bonuses and stopping-time to the entropic risk measure. 
We also revisit previous analyses and extract a problem-dependent quantity, the \emph{maximal achievable reward} $G_{\max}$, that gives a sharper dependence of the bounds on the MDP than directly using its upper bound $H$. 


The main contributions of the paper are the following:
\begin{itemize}
    \item First, we derive a lower bound for the best policy identification problem for the forward model. To the best of our knowledge, this is the first lower bound for BPI in this setting. The lower bound is expressed in terms of the maximal achievable reward $G_{\max}$  that we argue is better suited for the entropic risk measure.
    \item We present \textsc{Entropic-BPI}, an algorithm that outputs a policy $\pi$ that achieves $(\varepsilon,\delta)$-PAC with optimal sample complexity 
    up to logarithmic factors and up to $e^{2\max\{0,\beta\}}$ which is a constant for $\varepsilon$ small enough. The exponential dependence is $e^{|\beta| G_{\text{max}}(\mathcal{M})} \leq e^{|\beta| H}$ which removes the extra exponential factor with respect to previous work. 
\end{itemize}






\section{Risk-Sensitive Episodic Reinforcement Learning}
\label{sec:setting_background}
We consider a finite-horizon MDP $\mathcal{M} = \bigl(\mathcal S, \mathcal A, H, \{r_h\}_{h=1}^H, \{p_h\}_{h=1}^H\bigr)$ where $\mathcal S$ and $\mathcal A$ denote the finite state and action spaces respectively, $H \in \mathbb N$ is the fixed horizon length, $p_h(\cdot \mid s,a)$ is the transition kernel at step $h$ for each $h \in \{1,\dots,H\}$, and $r_h : \mathcal S \times \mathcal A \to [0,1]$ is the deterministic reward function at step $h$, which we assume to be bounded in $[0,1]$.

In the learning problem, the transition kernels $\{p_h\}_{h=1}^H$ are unknown to the learner. Through interactions, the learner must devise a policy $\pi$ that is a (possibly non-stationary) mapping prescribing which action to take in each state and step, with the goal of maximizing a suitable performance criterion. In the literature, this problem is addressed under two distinct interaction models.
\begin{itemize}
    \item \textbf{Generative model} \citep{ICML2012Gheshlaghi_638,KearnsMansourNg2002}: In the generative model (also called a sample oracle), the learner has access to an oracle such that for any query $(s,a,h) \in \mathcal S \times \mathcal A \times \{1,\dots,H\}$, the oracle returns a sample $(r,s')$ where $r = r_h(s,a)$ and $s' \sim p_h(\cdot \mid s,a)$. Each call to the oracle is independent of the past, and the learner may choose its queries adaptively based on all previously observed samples.
    \item \textbf{Online (forward) model} \citep{Dann2015,StrehlLiLittman2009}: In the forward interaction model, the learner interacts with the MDP over episodes. At the beginning of each episode $t$, an initial state $s_1^t$ is drawn from a fixed initial distribution $\mu_1$ on $\mathcal S$. At each step $h \in \{1,\dots,H\}$ of episode $t$, the learner observes the current state $s_h^t$, selects an action $a_h^t \in \mathcal A$ according to some policy $\pi_t$, then receives a reward $r_h(s_h^t,a_h^t)$ and the environment goes to the next state $ s_{h+1}^t \sim p_h(\cdot \mid s_h^t, a_h^t)$.
\end{itemize}
In this work, we focus exclusively on the online (forward) interaction model. Moreover, rather than the standard risk-neutral objective where we optimize the total expected reward along the trajectory, we adopt a risk-sensitive performance criterion based on the entropic risk measure which reflects the agent’s attitude toward uncertainty, allowing for both risk-seeking and risk-averse behavior.

\paragraph{Policy \& value functions} A deterministic policy $\pi$ is a collection of functions $\pi_h : \mathcal S \to \mathcal A$
for all $h \in \{1,...,H\}$, where every $\pi_h$ maps each state to a single action.
Let $(S_h,A_h)_{h=1}^H$ be the random trajectory induced by the policy $(\pi_h)_{h \in \{1,...,H\}}$ and the MDP dynamics,
i.e., $A_h = \pi_h(S_h)$ and $S_{h+1}\sim p_h(\cdot\mid S_h,A_h)$ starting from a fixed initial state $s_1$\footnote{As explained in \citep{Fiechter1994}, if the first state is sampled randomly as $s_1 \sim p_0$, we simply add an artificial first state $s_0$ such that for any action $a$, the transition probability is defined as the distribution $p_0(s_0, a) \triangleq p_0$. This augments the state space by one and the horizon by one, and the bounds carry over with only this constant-size modification.}. Define the (random) cumulative return from step $h$ by:
$$
R_h^\pi = \sum_{i=h}^H r_i(S_i,A_i)
$$
The entropic value function of $\pi$, denoted by $V_h^\pi$ is defined as:
$$
V_h^\pi(s)
\triangleq
\frac{1}{\beta}
\log \mathbb E\left[
\exp\left(
\beta \sum_{i = h}^H r_{i}(s_{i}, a_{i})
\right)
\mid S_h = s
\right]
$$
where $a_{i} \triangleq \pi_{i}(s_{i})$ and
$s_{i+1} \sim p_{i}(\cdot \mid s_{i}, a_{i})$

The optimal entropic value functions are defined as
$
V_h^\star(s) \triangleq \sup_{\pi} V_h^\pi(s)$ for $h \in [H]$. 
As for the expected return, both $V_h^\pi$ and $V_h^\star$ satisfy the Bellman equation \citep{borkar2001,sutton2018reinforcement,marthe2023beyond} and can thus be computed efficiently.

To simplify notation, we introduce a couple of operators that allow us to write concisely the effect of transition kernels or policies on our functionals of interest. 
For any bounded function $f : \mathcal S \to \mathbb R$, we denote by
$
(p_h f)(s,a) \triangleq \mathbb E_{S' \sim p_h(\cdot \mid s,a)}[f(S')]
$ the action of the Markov kernel $p_h$ on $f$. For any function $g : \mathcal S \times \mathcal A \to \mathbb R$ and (possibly stochastic) policy $\pi_h$, we write
$(\pi_h g)(s) \triangleq \mathbb E_{A \sim \pi_h(\cdot \mid s)}[g(s,A)]$ for the composition with the policy at step $h$.
Finally, we introduce a useful quantity that appears naturally in the analysis of the sample complexity.

\begin{definition}[Maximum achievable reward in a trajectory]\label{Gmax}
For an episodic MDP $\mathcal{M}$, 
define the maximal achievable return (a deterministic quantity depending only on $\mathcal{M}$) as
$$
G_{\max}(\mathcal{M}) = \sup_{\pi}\ \sup_{\omega}\ R_1^\pi(\omega)
$$
where $\omega$ covers all sources of randomness (initial state, transitions, policy’s own randomization).
\end{definition}
Intuitively, $G_{\max}(\mathcal{M})$ is the largest total reward that can occur in a single episode under the most favorable sequence of outcomes. We will express both the lower bound and the upper bound in the quantity $G_{\max}(\mathcal{M})$ rather than the horizon in previous works. It is the natural scale for the entropic risk measure objective where the hardness comes from the exponential amplification of rewards. Remark that we have the inequality $G_{\max}(\mathcal{M}) \leq H$ so our results can also be expressed (more loosely) as a function of the horizon.

 \paragraph{\textbf{Empirical MDP.}}
Let $(s_h^t, a_h^t, s_{h+1}^t)\in \mathcal{S}\times\mathcal{A}\times \mathcal{S}$ be a tuple observed by the algorithm at step $h$ of episode $t$.
For any $h \in \{1,...,H\}$ and $(s,a)\in \mathcal{S}\times\mathcal{A}$, define the visitation counts
$$
n_h^t(s,a) \triangleq \sum_{i=1}^{t} \mathbbm{1}\{(s_h^i,a_h^i)=(s,a)\}
\qquad
n_h^t(s,a,s') \triangleq \sum_{i=1}^{t} \mathbbm{1}\{(s_h^i,a_h^i,s_{h+1}^i)=(s,a,s')\}
$$
These quantities induce the empirical transition probabilities
$$
\widehat{p}_h^{\,t}(s' | s,a) \triangleq
\begin{cases}
\dfrac{n_h^t(s,a,s')}{n_h^t(s,a)} & \text{if } n_h^t(s,a)>0,\\[6pt]
\dfrac{1}{|\mathcal{S}|} & \text{otherwise}
\end{cases}
$$
We denote $\overline{n}_h^t(s,a) = \mathbb{E}[n_h^t(s,a)]$ the expected number of visits, which is called the pseudo-counts.

\paragraph{Best Policy Identification (BPI) under entropic risk}
Our objective is to identify a (near) optimal policy with high probability. In each episode $t$, the agent follows a policy $\pi^{t}$ (the \emph{sample rule}) based only on the information collected up to and including episode $t-1$. 
At the end of each episode, the agent can decide to stop collecting data according to a \emph{stopping rule} (we denote by $\tau$ its random stopping time) and outputs a guess $\hat{\pi}$ for the optimal policy.

A BPI algorithm is therefore made of a triple $\big((\pi^{t})_{t\in\mathbb{N}},\tau,\hat{\pi}\big)$. 
The goal is to build an $(\varepsilon,\delta)$-PAC algorithm according to the following definition, for which the sample complexity, that is the number of exploration episodes $\tau$, is as small as possible.

\begin{definition}[PAC algorithm for BPI]\label{BPI}
An algorithm is $(\varepsilon,\delta)$-PAC for best policy identification if it returns a policy $\hat{\pi}$ after some number of episodes $\tau$ that satisfies
$$
\mathbb{P} \left( V^{\star}_{1}(s_{1})-V^{\hat{\pi}}_{1}(s_{1}) \leq \varepsilon \right) \ge 1-\delta 
$$
\end{definition}
\paragraph{Prior bounds. }
For the entropic risk measure, the best policy identification problem is harder\footnote{In the sense that it admits higher lower bounds, see Sec.~\ref{sec:lower_bound}} because it magnifies tail outcomes—inflating high rewards when $\beta>0$ and penalizing low rewards when $\beta<0$. Moreover, as this line of work is relatively recent and most of the results aim to establish regret bounds, the results on PAC bounds remain scarce; to our knowledge, there are no reward-free algorithms \citep{Jin2020RewardFree,Kaufmann2021ALT} with theoretical guarantees specifically tailored to this criterion.\\

\paragraph{Sample complexity with the generative model. } In the discounted infinite-horizon, \citet{MortensenTalebi2025} proved a lower bound on the number of oracle calls needed for an $\varepsilon-$optimal policy. In particular, they showed that an exponential dependency on the risk parameter $\beta$ and on the horizon $H$ is unavoidable: Any algorithm that outputs an $\varepsilon$-optimal policy with probability at least $1-\delta$ must make at least $ \Omega\Bigg(\frac{SA\gamma^{2}}{c_{1}\varepsilon^{2}} \frac{e^{|\beta|^{\frac{1}{1-\gamma}}}-3}{|\beta|^{2}} \log\left(\frac{S}{c_{2}\delta}\right)\Bigg) $ calls to the oracle. They also provided two model-based algorithms with respective sample complexity $\tilde{\mathcal{O}}\Bigg(\frac{SA\gamma^{2}}{c_{1}(1-\gamma)^4\varepsilon^{2}} \frac{\Big(e^{2|\beta|{\frac{1}{1-\gamma}}}-1\Big)^2}{|\beta|^{2}} \log\left(\frac{SA}{c_{2}\delta}\right)\Bigg) \quad$ 
This gives the first explicit sample complexity characterization for the entropic risk measure objective with a generative model. It can be mapped back to our finite-horizon setting using $H = \frac{1}{1-\gamma}$, the effective horizon.

\paragraph{Forward model:} There are no known BPI sample-complexity bounds for the entropic risk measure in this model. However, there exist bounds on the regret, a more forgiving metric that sums $V^{\pi_t} - V^*$ over episodes that can be connected back to BPI as we explain next.  The first non-asymptotic results are for model-free optimistic algorithms RSVI/RSQ \citep{Fei2020RS}, establishing $\tilde{\mathcal{O}}\big(\lambda(|\beta|H^{2})\sqrt{H^{3}S^{2}AT}\big)$ (RSVI) and $\tilde{\mathcal{O}}\big(\lambda(|\beta|H^{2})\sqrt{H^{4}SAT}\big)$ (RSQ) regret with $\lambda(u)=(e^{3u}-1)/u$. \citet{FeiYCW21} introduce the exponential Bellman equation and a doubly-decaying bonus, removing an extra $e^{|\beta|H^{2}}$ factor and yielding a regret of 
$\tilde{\mathcal{O}}\Big(\frac{e^{|\beta|H}-1}{|\beta|H}
\sqrt{H^{4} S^{2} A K}
\Big)$. More recently, \citet{hu23b} adapt UCB-ADVANTAGE \citep{zhang2020almost} to the exponential-utility setting and derive a worst-case problem-independent regret bound
$\tilde{\mathcal O}\Big(\frac{e^{|\beta|H}-1}{|\beta|}\sqrt{H^{2}SAK}\Big)$
along with a tighter problem-dependent bound that matches the information-theoretic lower bound up to logarithmic factors on a class of MDPs.

These results are not directly related to our problem, though in general there is a connection between controlling regret and BPI sample complexity. As noticed by \citet{Jin2020RewardFree}, a straightforward approach to convert the regret bound to a sample complexity bound is to output a random policy among the sequence of policies returned by the regret algorithm. This idea was later shown to be outperformed (for the expected return) by \citet{Kaufmann2021ALT}, and the same applies in the case of entropic risk measures. It can be easily checked that a naive conversion of the best regret upper bounds above would yield a sample complexity in $e^{2|\beta| H}$ and a dependence in $\frac{1}{\delta^2}$. This upper bound can thus be considered the best achievable sample complexity so far for this problem.

\section{Lower bound}\label{sec:lower_bound}
Under the entropic risk criterion, the exponential transform $e^{\beta r}$ amplifies tail behavior, placing disproportionate weight on rare, high-return trajectories when $\beta>0$. Therefore, if achieving near-optimal performance requires hitting a ``hard'' state--action--time triple $(s^\star,a^\star,h^\star)$ that is reached with tiny probability, the learner must repeatedly experience these rare transitions to accurately evaluate their contribution since a small number of tail episodes can dominate the entropic objective. As a result, learning becomes exponentially hard in the horizon $H$ and the risk parameter $\beta$.

\begin{theorem}\label{th:lower_bound_theorem}
Fix $S \geq 6, A \geq 2, H \in \mathbb{N}, \beta \in \mathbb{R}^\star$, and $ \delta \in (0, 1/16)$ and $\varepsilon$ small enough (See condition ~\eqref{condition A}). There exists an MDP $\mathcal{M}_0$ with $S$ states, $A$ actions, horizon $H$, and rewards in $[0, 1]$ and nonstationary transitions such that for every algorithm $\mathcal{A}$ outputs a policy $\hat{\pi}$ that is $(\varepsilon, \delta)$-PAC for the entropic risk measure after sampling $\tau$ trajectories we have: 
$$\mathbb{E}_{\mathcal{M}_0}[\tau] \geq \frac{1}{1650} \frac{(e^{|\beta| G_{\max}(\mathcal{M}_0)} - 1)^2}{e^{|\beta| G_{\max}(\mathcal{M}_0)}} \frac{e^{2\min \{\beta,0\}\varepsilon}SAH}{(e^{|\beta| \varepsilon} - 1)^2} \log \left( \frac{1}{\delta} \right)$$
\end{theorem}
\begin{proof}[Sketch of proof]
We follow the hard episodic, stage-dependent MDP class introduced by \citet{domingues2021episodic}. 
At a high level, these instances behave like a single multi-armed bandit with 
$\Theta(HSA)$ arms, where an ``arm'' corresponds to a triple (time, leaf, action). 
The agent starts in a waiting state $s_w$ and can play a special action $a_w$ to remain in $s_w$ up to some stage $\bar H$; once it stops waiting (or after $\bar H$), it is forced to leave $s_w$ and then traverses a full $A$-ary tree deterministically (each action selects the corresponding child), reaching a leaf after $d$ steps. 
From a leaf state $s_i\in\mathcal L$, at stage $h$ and action $a$, the process transitions to an absorbing good state $s_g$ with probability $p_h(s_g|s_i,a)$
and to an absorbing bad state $s_b$ otherwise. Rewards are obtained only in $s_g$ for the rest of the horizon. 

Consequently, achieving the optimal value requires (i) leaving $s_w$ so as to hit the correct stage $h^\star$ at the leaves, (ii) reaching the correct leaf $s_{\ell^\star}$ (via the deterministic actions along the tree), and (iii) playing the correct action $a^\star$ at that leaf; only then does the probability of reaching $s_g$ increase from $p_-$ to $p_+$. We denote $u = (h^\star,l^\star,a^\star)$ for the rest of the proof.

The lower bound proof then compares a reference instance $\mathcal{M}_0$ (where no unique action is optimal) to instances $\mathcal{M}_u$ (where only one triplet has the favorable probability $p_+$), and applies standard change-of-measure arguments to show that distinguishing these close Bernoulli transition models forces a large number of episodes visiting the special triplet. 

The construction of $p_-$ for the entropic risk measure differs from the risk-neutral setting (where $p_- = \frac{1}{2})$ to reflect the hardness induced by the entropic risk measure:
\begin{itemize}
    \item For $\beta >0$, the entropic criterion overweights rare high-return trajectories, so we make success (reaching the good state) rare by choosing $p_- \sim  e^{-\beta H}$
    \item For $\beta<0$, the entropic criterion is especially sensitive to adverse tail events, so we instead make failure rare by choosing $p_- \sim 1 - e^{-|\beta|H}$
\end{itemize}
Then, we chose the gap $\Delta = p_+ - p_-$ so that, in instance $M_u$, any policy that does not identify the special triplet is $\varepsilon$-suboptimal for the entropic objective. A change of measure argument \citep{kauffman2016} then gives the lower bound. The full proof together with details on the MDP construction can be found in Appendix~\ref{lower_bound}
\end{proof}
To the best of our knowledge, this establishes the first lower bound for Best Policy Identification (BPI) in the forward model setting with non-stationary transitions.

When $|\beta|$ goes to $0$, the lower bound approaches:$$\Omega\left( \frac{G_{\max}(\mathcal{M}_0)^2}{\varepsilon^2} SAH \log \left( \frac{1}{\delta} \right)\right) = \Omega\left( \frac{SAH^3}{\varepsilon^2} \log \left( \frac{1}{\delta} \right)\right)$$
The second equality holds because, in the construction of $\mathcal{M}_0$, the maximum return $G_{\max}(\mathcal{M}_0)$ scales linearly with $H$. Thus, we recover the standard lower bound for the risk-neutral case. 

Moreover, for sufficiently small $\varepsilon$, the bound simplifies to:
$$\Omega \left( \frac{SAH}{\varepsilon^2} e^{|\beta|G_{\max}(\mathcal{M}_0)}\right)$$
This matches the lower bound derived by \citep{MortensenTalebi2025} in the generative model setting up to an additional factor of $H$. This extra factor is inherent to the non-stationary transition dynamics of the finite-horizon setting. Another difference is that the exponential dependence in our bound scales with $G_{\max}(\mathcal{M}_0)$ rather than the horizon $H$ (or the effective horizon $\frac{1}{1-\gamma}$). This distinction is intuitive: since the hardness of the entropic risk measure comes from exponentially reweighting trajectory returns, the exponential dependence is naturally governed by the maximum cumulative reward ($G_{\max}$) rather than the length of the episode.\label{explainGmaxlower}


\section{Algorithm and Matching Upper Bound}
\label{sec:algorithm}





We now present an algorithm whose sample complexity matches the lower bound (Theorem~\ref{th:lower_bound_theorem}) up to logarithmic factors. A central difficulty with the entropic risk measure is that, unlike the risk-neutral objective, it is neither additive nor sub-additive, which complicates both algorithm design and analysis. In particular, standard UCB-style approaches rely on concentration inequalities for additive returns, whereas the entropic criterion involves a log-moment generating function and does not directly fit into standard risk-neutral concentration frameworks.

A common workaround is to use the Lipschitz continuity of the logarithm to reduce the analysis to risk-neutral quantities; however, this can yield coarse bounds and, more importantly, may break the dynamic-programming structure. As observed by \citet{FeiYCW21}, losing a Bellman-type recursion can cause error terms to compound over the horizon, leading to substantially worse dependence on $H$ in the order of $e^{2\beta H^2}$. Similarly to \citet{FeiYCW21}, we instead work directly in the exponential space induced by the criterion. This restores a Bellman recursion for the exponentiated value functions and enables sharper control of uncertainty. In particular, Lemma~\ref{entropicvariance} derives Bellman-type variance identities in this space, which are unavailable in the original entropic-value
space due to the lack of sub-linearity\footnote{As shown by \citet{rowland2019statistics}, the variance is a Bellman-closed statistic and can be computed by dynamic programming. We derive this result for the variance of exponential utilities}. This identity highlights that the exponential parameterization is the natural domain for controlling uncertainty under the entropic criterion.

For a policy $\pi$, the exponential transforms of the value and the state-value function are:
$$
Z_h^\pi(s) \triangleq \exp(\beta V_h^\pi(s)),\qquad
U_h^\pi(s,a) \triangleq \exp(\beta Q_h^\pi(s,a)).
$$
We introduce two novel techniques. We adapt the KL-based exploration bonuses introduced in \citep{Menard2021Fast} to the entropic criterion to get bonuses that admit variance-sensitive control in the exponential space defined by these exponential transforms. We also propose an entropic stopping rule that yields sharper horizon dependence, improving over bounds that incur an additional factor of order $H^2e^{|\beta|H}$ \citep{MortensenTalebi2025}.

Similarly to \citet{azar17a}, \citet{Zanette2019} and \citet{Menard2021Fast} we define optimistic (and pessimistic, see \eqref{optimism_equation_positive}) state-value functions for $\beta >0$ on the exponential transform $U_h^\star$ of $Q_h^\star$ (see \eqref{optimism_equation_negative} for $\beta < 0$): fix $\widetilde{U}_{H+1}^t(s,a) = 1$ and recursively,
\begin{align}\label{optimism_positive}
    \widetilde{U}_h^t(s,a) &= \min \Bigg \{e^{\beta(H - h -  1)},e^{\beta r_h(s,a)} \Bigg[ \widehat{p}_h^t \widetilde Z_h^t (s,a) + b_h^t(s,a) + \frac{1}{H} \widehat{p}_h^t\Big( \widetilde{Z}_{h+1}^t - \underaccent{\widetilde}
    {Z}_{h+1}^t\Big)(s,a)\Bigg]\Bigg\}, 
\end{align}
where $\widetilde Z_h^t$ and $\underaccent{\widetilde}
    {Z}_{h+1}^t(s,a)$ denote respectively the optimistic and pessimistic exponential value functions\footnote{We use upper and lower tilde accents to indicate optimism and pessimism.}: (see again \eqref{optimism_equation_positive} for pessimistic versions)
\begin{align}
   \widetilde{Z}_{H+1}^t(s) &= 1, \quad \text{and recursively }\;\; \widetilde{Z}_h^t(s) = \max_{a \in \mathcal{A}} \widetilde{U}(s,a), 
\end{align}
and the bonus term is defined as:
\begin{equation}\label{bonus_positive}
b_h^t(s,a) = 2\sqrt{2}\sqrt{\operatorname{Var}_{\widehat{p}_h^t}(\widetilde Z_{h+1}^t)\frac{\alpha^\star(n_h^t(s,a))}{n_h^t(s,a)}} + 5 e^{\beta  (H - h)} \frac{\alpha(n_h^t(s,a))}{n_h^t(s,a)} +  4 He^{\beta  (H - h)}\frac{\alpha^\star(n_h^t(s,a))}{n_h^t(s,a)}\end{equation}
with the exploration rates $\alpha,\alpha^*$ (see \cref{eq:bonuses} for exact definitions) being of the form  $n,\delta \mapsto O(\log(SAH/\delta) + \log(n))$
and coming directly from the Bernstein inequality. 
As stated before, our algorithm acts greedily on this optimistic value function. Specifically, at step $h$ in episode $t$, given state $s_{h,t}$, \textsc{Entropic-BPI} executes
\begin{align}
    \label{eq:policy_def}
    a_{h,t} = \pi_t(s_{h,t}) = \argmax_{a} \widetilde{U}_h^t(s_{h,t},a)
\end{align}

\paragraph{Entropic certificate}
Following related work, we call certificate an upper bound on the width of the confidence interval for a policy $\pi$. We define it by  backward recursion:
\begin{align}
\pi_{h+1}^{t+1} G_h^t(s)\! =\! \min\left\{e^{\beta(H-h-1)}\!, e^{\beta r_h(s, \pi_h^{t+1}(s))} \!\left( 3b_h^t(s, \pi_h^{t+1}(s)) \!+\! \left(\!1\! + \!\frac{3}{H}\!\right)\! \widehat{p}_h^t (\pi_{h+1}^{t+1} G_{h+1}^t)(s, \pi_h^{t+1}(s)) \right)\right\}
\end{align}
with terminal $G_{H+1}^t \equiv 0$. In particular, on the high-probability good event, the performance gap of the greedy policy $\pi_{t+1}$ is controlled by the start-state certificate.
Concretely, lemma ~\ref{certificate_positive} shows that with high probability $1-\delta$:
\begin{equation}\label{certificate_equation}
Z_1^\star(s_1)-Z_1^{\pi_{t+1}}(s_1)\leq \widetilde{Z}_1(s_1) - Z_1^{\pi_{t+1}}(s_1) \leq  \pi_{t+1,1}G_1^t(s_1)
\end{equation}
\paragraph{Stopping rule}\label{stopping_rule_paragrph}
We derive a stopping criterion based on the certificate $G_h^t$ by relating the value function gap to the ratio of partition functions. The difference in value functions can be expressed in the exponential space as:
\begin{align*}
    (V^\star_1 - V^{\pi^{t+1}}_1)(s_1) &= \frac{1}{\beta} \log\left( \frac{Z^\star_1}{Z_1^{\pi^{t+1}}}\right)(s_1) \\
    &= \frac{1}{\beta} \log \left(1 + \frac{Z_1^\star - Z_1^{\pi^{t+1}}}{Z_1^{\pi^{t+1}}} \right)(s_1).
\end{align*}
To ensure the policy is $\varepsilon$-optimal, it suffices to satisfy at stopping time $\tau$:
\begin{equation}\label{true_optimality_condition}
    \pi_1^t G_1(s_1) \leq \left(e^{\beta \varepsilon} - 1\right) Z_1^{\pi^{t}}(s_1)
\end{equation}
However, $Z_1^{\pi^t}$ is unknown as it relies on the true dynamics. We therefore substitute it with a computable lower bound. Using \eqref{certificate_equation}, we have:
\begin{equation*}
    Z_1^{\pi_{t+1}}(s_1) \geq \widetilde{Z}_1^t(s_1) - \pi_{t+1,1}G_1^t(s_1).
\end{equation*}
Substituting this lower bound into the optimality condition \eqref{true_optimality_condition} yields a stronger, computable stopping rule:
\begin{equation*}
    \pi_1^t G_1(s_1) \leq \left(e^{\beta \varepsilon} - 1\right) \left(\widetilde{Z}_1^t(s_1) - \pi_{1}^t G_1(s_1)\right).
\end{equation*}
Solving for $\pi_1^t G_1(s_1)$, we obtain the final stopping criterion:
\begin{equation}\label{computabe_stopping_condition}
    \pi_1^t G_1(s_1) \leq \frac{e^{\beta \varepsilon} - 1}{e^{\beta \varepsilon}} \widetilde{Z}_1^t(s_1) 
\end{equation}
where both $\pi_1^t G_1(s_1)$ and $\widetilde{Z}_1^t(s_1)$ can be computed efficiently by dynamic programming. 

\paragraph{Insight on the stopping rule}\label{insight} To guarantee $\varepsilon$-optimality, we need $\pi_1 G_1(s_1)/Z_1^{\pi_1^t}(s_1)$ to be smaller than a threshold $e^{\beta \varepsilon} - 1$. Our analysis reveals that (see proof below)
\begin{equation}\label{insight_equation}
\pi_1 G_1(s_1)/Z_1^{\pi_1^t}(s_1) \lesssim \mathcal{O}\left( \sqrt{ \frac{\operatorname{Var}\Big(e^{\beta R_1^\pi}\big|S_1 = s_1\Big)}{\mathbb{E}\Big[e^{\beta R_1^\pi}\big|S_1 = s_1\Big]^2} \mathbb{E}^\pi\left[\sum_{h=1}^H \frac{1}{n_h^t(s,a)}\Bigg|s_1\right]}\right)\end{equation}
The bound consists of a decreasing visitation term and the constant
$
\tfrac{\operatorname{Var}\Big(e^{\beta R_1^\pi}\big|S_1 = s_1\Big)}{\mathbb{E}\Big[e^{\beta R_1^\pi}\big|S_1 = s_1\Big]^2}$,
which governs the asymptotic rate in contrast to $\operatorname{Var(R_1^\pi\big|S_1 = s_1})$ in the risk-neutral setting. 

It is insightful  to make a short comment on the connection of this quantity with the probability space we are working with. Let us denote $\mathbb{P}^\pi$  the probability distribution of a random trajectory $(s_1,a_1,...,s_H,a_H)$ in the MDP, and consider the tilted law $\mathbb{P}_\beta^\pi$ defined by the Radon-Nykodim derivative 
$\dfrac{d\mathbb{P}_\beta^\pi}{\mathbb{P}^\pi}(\omega) = \frac{e^{\beta R_1^\pi(\omega)}}{Z_1^\pi(s_1)}$. 
We can see the tilted measure as the law of a trajectory on a twisted version of the original MDP that biases transitions towards states with high future exponential return. It can be easily computed that:
$$\frac{\operatorname{Var}\Big(e^{\beta R_1^\pi}\big|S_1 = s_1\Big)}{\mathbb{E}\Big[e^{\beta R_1^\pi}\big|S_1 = s_1\Big]^2} = \chi^2(\mathbb{P}_\beta^\pi,\mathbb{P}^\pi)$$
In other words, the constant leading the convergence speed in the upper bound ~\eqref{insight_equation} is the $\chi_2$ divergence mismatch between the tilted trajectory distribution and the nominal trajectory distribution: It measures the extent to which optimizing the entropic objective amounts to learning under an implicit twisted dynamics that over-samples trajectories with high exponentiated reward. This mismatch is precisely what gets introduced by maximizing the entropic risk measure and what drives the extra constant $\frac{e^{\beta H}}{\beta^2}$ introduced in the sample complexity in contrast to the risk-neutral setting. 
Finally, notice that for $\beta \approx 0$, the $\chi^2$-divergence admits the expansion
$
\chi^2(\mathbb{P}_\beta^\pi , \mathbb{P}^\pi)
= \beta^2\operatorname{Var}_{\mathbb{P}^\pi}\big(R_1^\pi\big|S_1 = s_1\big) + \mathcal{O}(\beta^3)
$, which reduces the term in Eq.\eqref{insight_equation} to the variance term that governs the risk-neutral case.

\begin{algorithm}[htb]
\caption{Entropic-BPI}\label{algorithm}
\label{alg:entropic-bpi-both-stopping}
\begin{algorithmic}[1]
\STATE \textbf{Input:} $\beta\neq 0$, $\delta\in(0,1)$, $\varepsilon>0$.
\STATE Initialize counts $n_h^0(\cdot)=0$ and $\widehat p_h^0(\cdot|s,a)=1/S$.

\FOR{$t=0,1,2,\dots$}

  \STATE \textbf{Terminal init:} set $\widetilde Z_{H+1}^t(s)=1$, $\underaccent{\widetilde }{Z}_{H+1}^t(s)=1$, and $G_{H+1}^t(s)=0$ for all $s$.

  \FOR{$h=H,H-1,\dots,1$}

    \STATE Compute the bonus $b_h^t(\cdot)$:
    use~\eqref{bonus positive} if $\beta>0$, and~\eqref{bonus_negative} if $\beta<0$
    
    \STATE \textbf{Backup:} for all $(s,a)$ compute the optimistic and pessimistic quantities:
    use~\eqref{optimism_equation_positive} if $\beta>0$, and~\eqref{optimism_equation_negative} if $\beta<0$.

    \STATE \textbf{Greedy action:} for all $s$ set
    \vspace{-0.5cm}
    $$
      \pi_h^{t+1}(s)\in
      \begin{cases}
      \arg\max_{a\in\mathcal A}\widetilde U_h^t(s,a), & \beta>0,\\
      \arg\min_{a\in\mathcal A}\widetilde U_h^t(s,a), & \beta<0,
      \end{cases}
    $$
    \vspace{-0.3cm}
    \STATE \textbf{Certificate:} Compute $\pi_{h}^{t+1}G_h^t$:
    use~\eqref{width_certificate_positive} if $\beta>0$, and~\eqref{width_certificate_negative} if $\beta<0$.

  \ENDFOR

  \STATE \textbf{Stopping test:}
  \IF{$\beta>0$ \AND $(\pi_1^{t+1}G_1^t)(s_1)\le \frac{(e^{\beta\varepsilon}-1)}{e^{\beta \varepsilon}}\,\widetilde Z_1^t(s_1)$}
    \STATE \textbf{stop} and output $\pi^{t+1}$.
  \ENDIF
  \IF{$\beta<0$ \AND $(\pi_1^{t+1}G_1^t)(s_1)\le (1-e^{\beta\varepsilon})\,\underaccent{\widetilde}{Z}_1^t(s_1)$}
    \STATE \textbf{stop} and output $\pi^{t+1}$.
  \ENDIF

  \STATE Execute episode $t+1$ with $\pi^{t+1}$, update counts and $\widehat p_h^{t+1}$.

\ENDFOR
\end{algorithmic}
\end{algorithm}
\paragraph{Algorithm and sample complexity upper bounds. } We summarize the elements described above into an algorithm we call \textsc{Entropic-BPI} (\cref{alg:entropic-bpi-both-stopping}). Using the optimistic proxies in Eq.~\eqref{optimism_equation_positive} for $\beta>0$ and in Eq.~\eqref{optimism_equation_negative} for $\beta <0$, it builds exploratory trajectories until our stopping criterion (Eq.\ref{computabe_stopping_condition} for $\beta >0$ and Eq.\eqref{computabe_stopping_condition} for $\beta<0$) is reached. We prove a sample complexity bound for \textsc{Entropic-BPI} in the following theorem. This complexity bound is valid for both $\beta > 0$ (proof in Appendix~\ref{beta_positive}) and for $\beta < 0$ (proof in Appendix~\ref{beta_negative})
\begin{theorem}[sample complexity]\label{th:upper_bound}
For any $\delta \in [0,1]$ and $\varepsilon > 0$ small enough and for any finite MDP $\mathcal{M}$, \textsc{Entropic-BPI} (\cref{alg:entropic-bpi-both-stopping}) outputs a policy that is $(\varepsilon,\delta)$-PAC for best policy identification problem for the entropic risk measure after $\tau$ episodes. Moreover, with probability $1 - \delta$:
\begin{align*}
\tau = \mathcal{O}\left( \frac{e^{2|\beta|\varepsilon}}{\big(e^{|\beta| \varepsilon} - 1\big)^2}\frac{(e^{|\beta| G_{\text{max}}(\mathcal{M})} - 1)^2}{e^{|\beta| G_{\text{max}}(\mathcal{M})}} SAH  \log(\frac{3SAH}{\delta})\right)
\end{align*}
\end{theorem}
The algorithm upper bound matches the lower bound derived in \cref{th:lower_bound_theorem} up to a factor $e^{2 |\beta| \varepsilon}$ which is a constant when $\varepsilon$ is small enough and comes from having a stopping rule using comptable proxies instead of $Z_h^\pi$.
When $|\beta|$ goes to $0$, the upper bound approaches :
$$\tilde{\mathcal{O}}\left(\frac{G_{\text{max}}(\mathcal{M})^2 SAH}{\varepsilon^2}\right)= \tilde{\mathcal{O}}\left(\frac{SAH^3}{\varepsilon^2}\right)$$
and we recover the optimal sample complexity for the risk-neutral setting \citep{Menard2021Fast}.

Also remark that using the elementary inequality $\log(1 + x) \geq \frac{x}{2}$ for $|\beta| \varepsilon \in [0,1]$ and using\footnote{since the rewards are in $[0,1]$} that $\frac{(e^{|\beta| G_{\text{max}}(\mathcal{M})} - 1)^2}{e^{|\beta| G_{\text{max}}(\mathcal{M})}} \leq e^{|\beta| G_{\text{max}}(\mathcal{M})} -1 \leq e^{|\beta| H} - 1$  we have an upper bound of the order of :
$$\tau = \tilde{\mathcal{O}}\left( \frac{e^{|\beta| H} - 1}{\beta^2}\frac{SAH}{\varepsilon^2}  \right)$$

This matches the lower bound of \citet{MortensenTalebi2025} when mapped to the finite-horizon setting, up to an additional factor $H$ which is unavoidable as it is inherent to the non-stationary finite-horizon setting (with $H$ separate kernels). 

Note that the upper bound is stated in terms of $G_{\max}(\mathcal{M})$ rather than $H$. Since rewards lie in $[0,1]$, $G_{\max}(\mathcal{M})$ can be interpreted as an effective reward horizon, i.e., the maximal cumulative reward that can be accrued along a trajectory. This choice is natural for the entropic risk criterion, whose difficulty comes from the exponential amplification of accumulated rewards; using $H$ may overestimate this effect in problems where rewards are sparse or concentrated near the end of the episode. Finally, our algorithm does not require prior knowledge of $G_{\max}(\mathcal{M})$. 

\paragraph{Experiments. }  To illustrate the gains in sample complexity, we propose a simple 8-state MDP and compare \textsc{Entropic-BPI} with regret algorithms in the literature. The results are discussed in Appendix~\ref{app:experiments}

\section{Proof of Theorem~\ref{th:upper_bound}}

We present the main ideas of the proof for $\beta>0$. The full proof is given in \cref{proofsupperbound}.
We first control the concentration events (Lemma~\ref{good event}) and work on the good event $\mathcal{E}^+$ for $\beta>0$, which holds with probability at least $1-\delta$. As explained in paragraph~\nameref{stopping_rule_paragrph}, when the algorithm stops, it outputs by design a policy $\pi^\tau$ that is $\varepsilon$-optimal with high probability $1-\delta$, this proves the first statement of Theorem ~\ref{th:upper_bound}.

To upper bound the stopping time, we first show that $\widetilde{U}_h^t$ and $\underaccent{\widetilde}{U}_h^t$ are indeed optimistic and pessimistic, respectively, for the exponential transform of the value functions $U_h^\star$ (Lemma~\ref{optimism positive}). Then, following a similar approach to \citep{Menard2021Fast,DannLB17}, we bound the width certificate by a computable recursive upper bound, which serves as the stopping rule for our algorithm~\eqref{certificate_equation}. Lemma~\ref{certificate_positive} shows that, with probability at least $1-\delta$, this width certificate upper bounds the optimality gap. Since the stopping condition is not met for episodes $t=\{1,\ldots,\tau\}$, we have:
\[
\tau \frac{e^{\varepsilon \beta}-1}{e^{\varepsilon \beta}} \leq \sum_{t=1}^{\tau} \frac{\pi_1 G_1^t(s_1)}{\widetilde{Z}^t_1(s_1)}
\]

We bound the right-hand side for episode $t$ by replacing the empirical transition probabilities with the true ones. We then unroll the resulting recursive inequality for $\pi_{h+1}^t G_h^t$ under the true dynamics (see \ref{proof:upper_bound} for details):
\begin{align*}
\frac{(\pi_1 G_1^{t})(s_1)}{\widetilde{Z}_1^t(s_1)}&
\leq e^{13}
\mathbb{E}^{\pi^{t+1}}\Bigg[
\sum_{h=1}^{H} 
\exp\Big(\beta\sum_{i=1}^{h} r_i(s_i,a_i)\Big)
\Bigg(
36\sqrt{\frac{\operatorname{Var}_{p_h}\big(Z_{h+1}^{\pi^{t+1}}\big)}{(Z_1^{\pi^{t+1}})^2}\alpha^\star\big(n_h^t(s_h,a_h) \wedge 1\big)}
\\&+
81He^{\beta(H-h)}\alpha\big(n_h^t(s_h,a_h) \wedge 1\big)
\Bigg)
\Bigg| s_1
\Bigg]
\end{align*}
We bound the first term using Cauchy-Schwartz inequality:
\begin{align*}
    \mathbb{E}^{\pi^{t+1}} & \left[\sum_{h=1}^{H}  
\exp\Big(\beta \sum_{i=1}^{h} r_i(s_i,a_i)\Big)
\Bigg(
\sqrt{\frac{\operatorname{Var}_{p_h}\big(Z_{h+1}^{\pi^{t+1}}\big)}{(Z_1^{\pi^{t+1}})^2}\big(\frac{\alpha^\star\big(n_h^t(s,a)}{n_h^t(s,a)} \wedge 1 \big)}\Big) \Bigg| s_1\right]\\ &\leq \sqrt{\mathbb{E}^{\pi^{t+1}}\Big[\exp\Big(2\beta \sum_{i=1}^{h} r_i(s_i,a_i)\Big) \frac{\operatorname{Var}_{p_h}\big(Z_{h+1}^{\pi^{t+1}}\big)}{(Z_1^{\pi^{t+1}})^2}\Big]} \sqrt{\mathbb{E}^{\pi^{t+1}}\Big[\frac{\alpha^\star(n_h^t(s,a))}{n_h^t(s,a)}}\Big]
\end{align*}
Using lemma~\ref{entropicvariance} and lemma~\ref{bound} we bound the first term on the right-hand side as:
\begin{align*}
    \sum_{h=1}^H \mathbb{E}^{\pi^{t+1}} \!\!\left[ \exp\Big(2 \beta \sum_{i=1}^{h} r_i(s_i,a_i)\Big)\frac{\operatorname{Var}_{p_h}\big(Z_{h+1}^{\pi^{t+1}}\big)}{(Z_1^{\pi^{t+1}})^2} \right] \!=\! \frac{\operatorname{Var}\big(e^{\beta R_1^{\pi^{t+1}}}(S_1)\big|S_1=s_1\Big)}{\Big(\mathbb{E}\Big[e^{\beta R_1^{\pi^{t+1}}(S_1)}\big|S_1 = s_1 \Big]\Big)^2} \!\leq \!\frac{\left(e^{\beta G_{\text{max}}(\mathcal{M})} - 1\right)^2}{e^{\beta G_{\text{max}}(\mathcal{M})}}
\end{align*}
For the second term, we bound it loosely by directly upper bounding the per-step reward by $1$:
\begin{align*}
   \mathbb{E}^{\pi^{t+1}} & \left[
\sum_{h=1}^{H}
\exp\Big(\beta\sum_{i=1}^{h} r_i(s_i,a_i)\Big)
He^{\beta(H-h)}\big(\frac{\alpha\big(n_h^t(s,a)}{n_h^t(s,a)} \wedge 1 \big)
\Bigg)
\middle| s_1
\right] \leq  H e^{\beta H} \mathbb{E} \Big[\frac{\alpha\big(n_h^t(s,a)}{n_h^t(s,a)} \wedge 1 \Big]
\end{align*}
We sum over all episodes $t =1 , ...,\tau$. Then using a standard counting argument we have: 
\begin{align*}
    \sum_{t=1}^{\tau -1}\sum_{h=1}^H \mathbb{E}^{\pi^{t+1}}\Big[\frac{\alpha^\star\big(n_h^t(s,a)}{n_h^t(s,a)} \wedge 1 \Big] \leq 3 SAH\alpha^\star(\tau-1,\delta) \log(\tau + 1)
\end{align*}
And we have a similar bound for $\alpha$, Hence:
\begin{align*}
    \tau \frac{e^{\beta \varepsilon} - 1}{e^{\beta \varepsilon}} &\leq  36 e^{13}  \sqrt{\frac{(e^{\beta G_{\text{max}}(\mathcal{M})} - 1)^2}{e^{\beta G_{\text{max}}(\mathcal{M})}}\tau SAH\alpha^\star(\tau-1,\delta) \log(\tau + 1)} \\
    &\quad \quad + 84 e^{13} e^{\beta H}SAH\alpha(t-1,\delta) \log(t + 1)
\end{align*}
Solving this inequality using lemma~\ref{inequality}, we find the exact upper bound on $\tau$ with probability $1-\delta$.

\section{Regret bounds}
Although the main focus of this work is best-policy identification, the
techniques developed above also have implications for regret minimization. In particular, the certificate analysis for Entropic-BPI can be used to control the regret of a non-stopping variant of the algorithm.

We emphasize that this is not a direct reduction from BPI: BPI evaluates only the final policy, whereas regret evaluates the full learning trajectory. Nevertheless, both problems are driven by the need to identify rare transitions whose contribution is exponentially amplified by the entropic criterion.

We assume in this section that $\beta >0$ but the same analysis follows for $\beta < 0$. We consider the non-stopping variant of algorithm ~\ref{algorithm}. At each episode, the algorithm outputs the optimistic  model from the current empirical counts, plays the greedy policy, updates the counts and continues indefinitely.

We define the cumulative regret after $K$ episodes as:
\begin{equation}
    R(K) = \sum_{t = 0}^{K-1} (V_1^\star(s_1) - V_1^{\pi^{t+1}}(s_1))
\end{equation}

The key observation is that the same certificate $G^t$ used in the stopping rule controls each instantaneous regret term through the normalized ratio $\pi_1 G_1(s_1)/Z_1^{\pi_1^t}(s_1)$. Controlling this normalized ratio is what allows us to recover the sharper exponential dependence

\begin{theorem}\label{th:regretupper}
On the same high-probability event as theorem~\ref{th:upper_bound}, the non-stopping variant of Algorithm~\ref{algorithm} satisfies, for every $K \geq 2$,
$$R(K) = \mathcal{O}\left( \frac{1}{|\beta|} \sqrt{ \frac{(e^{|\beta |G_{\max}(\mathcal{M})} - 1)^2}{e^{|\beta| G_{\max}(\mathcal{M})}} KSAH \log\left( \frac{SAHK}{\delta} \right) } + \frac{e^{|\beta| H}}{\beta} S^2 AH \log^2\left( \frac{SAHK}{\delta} \right) \right).$$\end{theorem}
In particular when $G_{\max}(\mathcal{M})$ is of the order of $H$, the exponential dependency becomes $\frac{e^{\beta H/2}}{\beta}$, improving over existing regret bounds whose leading exponential dependence is $\frac{e^{\beta H}}{\beta}$.

The proof follows the sample-complexity analysis, but instead of stopping when the certificate becomes smaller than the desired accuracy threshold, we sum the normalized certificates over the first $K$ episodes. The variance sensitive part is again controlled through the exponential Bellman variance identity, which yields the factor $\frac{(e^{\beta G_{\max}(\mathcal{M})} - 1)^2}{e^{\beta G_{\max}(\mathcal{M})}}$. The cumulative exploration terms are then bounded using the same pseudo-count argument as in the stopping-time analysis.

We also evaluate the non-stopping variant of ENTROPIC-BPI on the toy MDP described in Appendix ~\ref{app:experiments}. The comparison includes several risk-sensitive regret algorithms from the literature. Across the tested values of $\beta$, our method achieves smaller cumulative regret, and the improvement becomes more noticeable as the risk sensitivity increases, supporting the sharper exponential dependence in Theorem~\ref{th:regretupper}.

\paragraph{Regret lower bound} The lower bound construction of section~\ref{sec:lower_bound} can be adapted to get the lower bound for the regret setting. Recall that the hard family contains $\Theta(SAH)$ possible hard triplets $u = (h^\star, \ell^\star, a^\star)$. In instance $M_u$, only the triplet $u$ has the favorable transition probability $p_+$, while all other triplets have probability $p_-$. Thus, before the learner gathers enough information to identify $u$, it incurs regret whenever it fails to visit the correct triplet.

The regret lower bound is obtained by choosing the value gap $\varepsilon$ as a function of the time horizon $K$. The information gained from one visit to the special triplet is of order :
$\frac{\beta^2 \varepsilon^2}{\frac{(e^{|\beta| G_{\max}(\mathcal{M})} - 1)^2}{e^{|\beta| G_{\max}(\mathcal{M})}}}.$
Since there are $\Theta(SAH)$ possible hard triplets and only $K$ episodes, the learner cannot identify the correct triplet when $\varepsilon$ is the order of $\frac{1}{|\beta|} \sqrt{\frac{\frac{(e^{|\beta| G_{\max}(\mathcal{M})} - 1)^2}{e^{|\beta| G_{\max}(\mathcal{M})}} SAH}{K}}.$ Consequently, the expected regret satisfies:
$$\mathbb{E}[R(K)] \geq \frac{c_0}{|\beta|} \sqrt{\frac{(e^{|\beta| G_{\max}(\mathcal{M})} - 1)^2}{e^{|\beta| G_{\max}(\mathcal{M})}} SAHK}$$

This regret lower bound improves on existing lower bounds in the literature. Combining the regret upper bound with the corresponding lower-bound argument shows that the non-stopping optimistic variant of Entropic-BPI is minimax-optimal in its leading $\sqrt K$ term, up to logarithmic factors and lower-order additive terms. In particular, in worst-case instances where $G_{\max}(\mathcal M)\asymp H$, the leading exponential factor is $e^{|\beta|H/2}/|\beta|$. Thus, at the level of the leading regret term, the analysis closes the gap between previous $e^{|\beta|H}/|\beta|$-type upper bounds and the $e^{|\beta|H/2}/|\beta|$ lower-bound dependence.

\section{Discussions and Conclusions}

We provide a new approach to entropic best arm identification that resolves a known suboptimality gap. Our approach builds a successful optimism-driven framework for the forward model, relying on a tight control of the variance of the estimators of the entropic value function, and on a specifically tailored stopping rule.  

Indeed, the dependence of the sample complexity on the horizon remains exponential, indicating one more time that learning exponential utilities in MDPs is a significantly harder problem than the standard expected return due to the focus on tail (rare) events. However, we show that the real MDP-dependent term of interest in the exponential is the maximal return, which could be constant in some sparse reward problems, making the problem more amenable. Such problem-specific investigations could be an avenue for future work. More broadly, this exponential dependence should be interpreted as the intrinsic statistical price of entropic risk sensitivity rather than as evidence that the criterion is unsuitable for learning. Indeed, the entropic risk measure has an important structural advantage: it preserves an exact Bellman recursion on the original MDP state space and is essentially characterized by dynamic consistency among utility-based objectives \citep{marthe2023beyond}. By contrast, other tail-risk criteria such as CVaR can enjoy favorable statistical rates, including near-minimax regret guarantees with polynomial dependence on the tail parameter \citep{wang23m}. However, these guarantees rely on additional assumptions on the return distributions, and direct optimization of CVaR typically requires state augmentation and optimization over an additional risk variable, often handled through discretization which may lead to inefficient runtime.

Recently, \citet{Marthe2025Efficient} showed that there is a fundamental connection between the entropic risk measure and more practically-used metrics such as the (conditional) Value at Risk. 
Our forward-model approach could be combined with such optimization improvements to propose new algorithms for (C)VaR optimization in RL. 
\acks{C. Vernade is funded by the Deutsche Forschungsgemeinschaft (DFG) under both the project 468806714 of the Emmy Noether Programme and under Germany’s Excellence Strategy – EXC number 2064/1 – Project number 390727645. CV also gratefully acknowledges funding from the European Union (ERC, ConSequentIAL, 101165883). Views and opinions expressed are however those of the author(s) only and do not necessarily reflect those of the European Union or the European Research Council. Neither the European Union nor the granting authority can be held responsible for them). CV also thanks the international Max Planck Research School for Intelligent Systems (IMPRS-IS).}

\bibliography{refs2}
\clearpage
\appendix

\part*{Appendix}
\addcontentsline{toc}{part}{Appendix} 

\etocsettocstyle{\section*{Contents of Appendix}}{}
\etocsetnexttocdepth{subsubsection}
\localtableofcontents
\clearpage


\section{Concentration events}\label{Concentration_events}
Following the ideas of \citep{Menard2021Fast} we define the following quantities:
\begin{align}
\alpha(n, \delta) &= \log(3SAH/\delta) + S \log(8e(n + 1)) \label{eq:bonuses}\notag\\
\alpha^{\text{cnt}}(\delta) &= \log(3SAH/\delta) \quad \text{and} \notag\\
\alpha^\star(n, \delta) &= \log(3SAH/\delta) + \log(8e(n + 1))
\end{align}
We also define the KL-divergence concentration event as: 
\begin{align*}
    \mathcal{E}_{KL} = \left\{ \forall t \in \mathbb{N}, \forall h \in \{1,...,H\}, \forall (s, a) \in \mathcal{S} \times \mathcal{A} : D_{KL}(\widehat{p}_h^t(s,a), p_h(s,a)) \leq \alpha(n_h^t(s,a),\delta) \right\}
\end{align*}
and the Bernoulli concentration event for a series of function $(f_h)_{h \in [H+1]}$ in $[0,b]$:
\begin{align*}
    \mathcal{E}_{f} &= \Bigg\{ \forall t \in \mathbb{N}, \forall h \in \{1,...,H\}, \forall (s, a) \in \mathcal{S} \times \mathcal{A} : \\&\left| (\widehat{p}_h^t - p_h) f_{h+1}(s, a) \right| < \sqrt{ 2 \operatorname{Var}_{p_h}(f_{h+1})(s, a) \frac{\alpha^\star(n_h^t(s, a), \delta)}{n_h^t(s, a)} }+ 3 b \frac{\alpha^\star(n_h^t(s, a), \delta)}{n_h^t(s, a)} \Bigg\}
\end{align*}
And the counts concentration event:
$$\mathcal{E}^{\text{cnt}} = \left\{ \forall t \in \mathbb{N}, \forall h \in \{1,...,H\}, \forall (s, a) \in \mathcal{S} \times \mathcal{A} : n_h^t(s, a) \geq \frac{1}{2} \bar{n}_h^t(s, a) - \alpha^{\text{cnt}}(\delta) \right\}$$
Using this, we define the good event for our algorithm analysis for $\beta > 0$ and $\beta <0$: for $\beta >0$
\begin{align*}
    \mathcal{E}^+ = \mathcal{E}_{KL} \cap  \mathcal{E}_{V^*} \cap \mathcal{E}^{\text{cnt}}
\end{align*}
And for $\beta <0$ we have almost the same thing but the definition of $\mathring{V}$ and the range of the functions changes:
\begin{align*}
    \mathcal{E}^- = \mathcal{E}_{KL} \cap  \mathcal{E}_{V^*} \cap \mathcal{E}^{\text{cnt}}
\end{align*}

\begin{lemma}\label{good event}
    It holds that :
    \begin{align*}
        \mathbb{P}(\mathcal{E}_{KL}) \geq 1 - \frac{\delta}{3}, \qquad \mathbb{P}(\mathcal{E}^{\text{cnt}}) \geq 1 - \frac{\delta}{3}, \qquad \textit{and} \qquad \text{for any $f$} \quad \mathbb{P}(\mathcal{E}_f) \geq 1 - \frac{\delta}{3}
    \end{align*}
    Consequently,
    $$
    \mathbb{P}(\mathcal{E}^+) \geq 1 - \delta \quad \text{and} \quad \mathbb{P}(\mathcal{E}^-) \geq 1 - \delta
    $$
    \end{lemma}
\begin{proof}

\textbf{The KL concentration event}: 

For $(h,s,a)$ fixed, we apply lemma~\ref{Sanov} with confidence level $\delta_{h,s,a} =\frac{\delta}{3SAH}$ and then do a union bound over $h,s,a$ to get a concentration inequality that holds uniformly.

\textbf{The Bernstein concentration event}:

Let $(f_h)_{h\in[H+1]}$ be a sequence of function. Fix $(h,s,a)$ and let $(\mathcal F_\tau)_{\tau\ge0}$ be the history filtration. Define
$$
w_\tau=\mathbf 1\{(H_\tau,S_\tau,A_\tau)=(h,s,a)\},\qquad
Y_\tau=f_{h+1}(S_{\tau+1})-\mathbb{E}[f_{h+1}(S_{\tau+1})\mid \mathcal F_{\tau-1}]
$$
Then $(w_\tau)$ is predictable, $\mathbb{E}[Y_\tau\mid\mathcal F_{\tau-1}]=0$, $|Y_\tau|\le H$, and
$$
\mathbb{E}[Y_\tau^2\mid\mathcal F_{\tau-1}]=\operatorname{Var}_{p_h}(f_{h+1})(s,a)\quad\text{on }\{w_\tau=1\}
$$
Let
$$
S_t=\sum_{\tau=1}^t w_\tau Y_\tau,\quad
V_t=\sum_{\tau=1}^t w_\tau^2\mathbb{E}[Y_\tau^2\mid\mathcal F_{\tau-1}],\quad
W_t=\sum_{\tau=1}^t w_\tau=n_h^t(s,a)
$$
Then $V_t=W_t\operatorname{Var}_{p_h}(f_{h+1})(s,a)$ and, for $W_t\ge1$,
$$
(\widehat p_h^t-p_h)f_{h+1}(s,a)=\frac{S_t}{W_t}
$$
Applying lemma~\ref{Bernstein} with $b$ and confidence $\delta_{h,s,a}$ yields (simultaneously for all $t$)
$$
\Big|(\widehat p_h^t-p_h)f_{h+1}(s,a)\Big|
\leq
\sqrt{2\operatorname{Var}_{p_h}(f_{h+1})(s,a)\frac{\alpha^\star(n_h^t(s,a),\delta)}{n_h^t(s,a)}}
+
3B\frac{\alpha^\star(n_h^t(s,a),\delta)}{n_h^t(s,a)}
$$
where $\alpha^\star(n,\delta)=\log\big(\frac{4e(2n+1)}{\delta_{h,s,a}}\big)$ (using $n_h^t(s,a)\le t$ and monotonicity of the log term).
Finally choose $\delta_{h,s,a}=\frac{\delta}{SAH}$ and we apply a union bound over $(h,s,a)$ to get the result for any state-action pair

\textbf{The counts concentration event:}

The proof follows from lemma~\ref{Bernoulli} applied to the Bernoulli random variable $\mathbbm{1}\left\{ (s_h^i, a_h^i) = (s, a) \right\}$ for $\delta_{h,s,a} = \frac{\delta}{SAH}$ and then doing a union bound over $h,s,a$
\end{proof}
\section{Algorithm analysis}\label{Algorithm analysis}
\label{proofsupperbound}
Here we provide a detailed analysis of our algorithm. Our method is a UCB-style algorithm that plans over a KL confidence region, following the approach of \citet{Menard2021Fast} for the risk-neutral objective. At each step $t$ and stage $h$, we construct a confidence set around the true transition kernel:
\begin{align*}
    \mathcal{C}_h^t(s, a) \triangleq \left\{ q \in \Sigma_S : \operatorname{KL}\left(\widehat{p}_h^t(s, a), q(s, a)\right) \le \frac{\alpha(n_h^t(s, a), \delta)}{n_h^t(s, a)} \right\}
\end{align*}
The algorithm then acts optimistically by selecting, among all transition models $q$ such that $q(.|s,a) \in \mathcal{C}_h^t(s, a) $, the one that yields the highest value function, and plans accordingly.

For the entropic risk measure, we follow the same principle, but carry out optimistic planning in the exponential (log-moment-generating) space induced by the entropic criterion. As noted by \citet{FeiYCW21}, working in this exponential space allows us to exploit a Bellman-type recursion that is generally lost if one applies Lipschitz arguments directly in the original value space. This means that the upper and lower confidence bounds on the optimal exponential transformation of the state-value function $U^*$ and value function $Z^*$ for $\beta > 0$ are given by: 
\begin{align*}
    \overline{U}_h^t(s, a) &\triangleq e^{\beta r_h(s,a)} \max_{\overline{p}_h \in \mathcal{C}_h^t(s,a)} \overline{p}_h \overline{Z}_{h+1}^t(s, a) & 
    \underline{U}_h^t(s, a) &\triangleq e^{\beta r_h(s,a)}\min_{\underline{p}_h \in \mathcal{C}_h^t(s,a)} \underline{p}_h \underline{Z}_{h+1}^t)(s, a) \\
    \overline{Z}_h^t(s) &\triangleq \max_{a} \overline{U}_h^t(s, a) & 
    \underline{Z}_h^t(s) &\triangleq \max_{a} \underline{U}_h^t(s, a) \\
    \overline{Z}_{H+1}^t(s) &\triangleq 0 & 
    \underline{Z}_{H+1}^t(s) &\triangleq 0 \\
    \overline{p}_h^t(s, a) &\in \operatorname*{arg\,max}_{\overline{p} \in \mathcal{C}_h^t(s,a)} \overline{p}_h \overline{Z}_{h+1}^t(s, a) & 
    \underline{p}_h^t(s, a) &\in \operatorname*{arg\,min}_{\underline{p} \in \mathcal{C}_h^t(s,a)} \underline{p}_h \underline{Z}_{h+1}^t(s, a) \\
    \overline{\pi}_h^t(s, a) &\in \operatorname*{arg\,max}_{a \in \mathcal{A}} \overline{U}_h^t(s, a) & 
    \underline{\pi}_h^t(s, a) &\in \operatorname*{arg\,max}_{a \in \mathcal{A}} \underline{U}_h^t(s, a).
\end{align*}
For $\beta < 0$, $ \overline{U}$ will correspond to the pessimistic $\underline{Q}$ via the log-transformation. As such, maximizing $\overline{Q}$ to define the policy $\overline{\pi}$ is equivalent to minimizing $\underline{U}$. Similarly, finding the best action at each stage to define $\overline{Z}$ and $\underline{Z}$ corresponds to minimizing $\overline{U}$ and $\underline{U}$ respectively:
\begin{align*}
    \overline{U}_h^t(s, a) &\triangleq e^{\beta r_h(s,a)} \min_{\overline{p}_h \in \mathcal{C}_h^t(s,a)} \overline{p}_h \overline{Z}_{h+1}^t(s, a) & 
    \underline{U}_h^t(s, a) &\triangleq e^{\beta r_h(s,a)}\max_{\underline{p}_h \in \mathcal{C}_h^t(s,a)} \underline{p}_h \underline{Z}_{h+1}^t(s, a) \\
    \overline{Z}_h^t(s) &\triangleq \min_{a} \overline{U}_h^t(s, a) & 
    \underline{Z}_h^t(s) &\triangleq \min_{a} \underline{U}_h^t(s, a) \\
    \overline{Z}_{H+1}^t(s) &\triangleq 0 & 
    \underline{Z}_{H+1}^t(s) &\triangleq 0 \\
    \overline{p}_h^t(s, a) &\in \operatorname*{arg\,min}_{\overline{p} \in \mathcal{C}_h^t(s,a)} \overline{p}_h \overline{Z}_{h+1}^t(s, a) & 
    \underline{p}_h^t(s, a) &\in \operatorname*{arg\,max}_{\underline{p} \in \mathcal{C}_h^t(s,a)} \underline{p}_h \underline{Z}_{h+1}^t(s, a) \\
    \overline{\pi}_h^t(s, a) &\in \operatorname*{arg\,min}_{a \in \mathcal{A}} \overline{U}_h^t(s, a) & 
    \underline{\pi}_h^t(s, a) &\in \operatorname*{arg\,max}_{a \in \mathcal{A}} \underline{U}_h^t(s, a).
\end{align*}
The KL confidence sets $C_h^t(s, a)$ are introduced solely to motivate an optimistic model interpretation. We instead build computable optimistic and pessimistic expressions in the empirical MDP by choosing the radius $\alpha$ so that the true transition kernel belongs to $\mathcal{C}_h^t(s,a)$ in the same style as \citep{Menard2021Fast}. We then prove the corresponding optimism lemma, bound the certificate width, and derive the sample complexity. We treat the cases $\beta>0$ and $\beta<0$ separately. We first restate the theorem in more detail
\begingroup
\renewcommand\thetheorem{\ref{th:upper_bound}}
\begin{theorem}[sample complexity]
For any $\delta \in [0,1]$ and $\varepsilon \in ]0,\frac{2}{|\beta|HS}]$ and for any finite MDP $\mathcal{M}$, \textsc{Entropic-BPI} (\cref{alg:entropic-bpi-both-stopping}) outputs a policy that is $(\varepsilon,\delta)$-PAC for best policy identification problem for the entropic risk measure after $\tau$ episodes. Moreover, with probability $1 - \delta$:
\begin{align*}
    \tau \leq\frac{e^{2\max\{0,\beta\}\varepsilon}}{\big(e^{|\beta| \varepsilon} - 1\big)^2}\frac{(e^{|\beta| G_{\text{max}}(\mathcal{M})} - 1)^2}{e^{|\beta| G_{\text{max}}(\mathcal{M})}} SAH \log(\frac{3SAH}{\delta}) C_2^2
\end{align*}
Where $C_2 = 2765 e^{22}\log\bigg(4 e^{17} \frac{(S+1)(H+1) e^{|\beta H} S A H^2}{e^{\beta \varepsilon} - 1}\bigg)$. In particular, hiding constants and log terms:
$$\tau = \tilde{\mathcal{O}}\left( \frac{e^{2\max\{0,\beta\}\varepsilon}}{\big(e^{|\beta| \varepsilon} - 1\big)^2}\frac{(e^{|\beta| G_{\text{max}}(\mathcal{M})} - 1)^2}{e^{|\beta| G_{\text{max}}(\mathcal{M})}} SAH  \right)$$
\end{theorem}
\endgroup
\subsection{Case \texorpdfstring{$\beta > 0$}{beta > 0}}\label{beta_positive}.
We start by building optimistic and pessimistic functions for the state-value function
\subsubsection{Confidence bounds}
Let us start with a concentration inequality:
\begin{lemma}\label{concentration_positive}
On the good event $\mathcal{E}^+$ we have:
\begin{align*}
    | \big(p_h - \widehat{p}_h^t\big)Z_{h+1}^{\star}(s,a)| &\leq 2\sqrt{2}\sqrt{\operatorname{Var}_{\widehat{p}_h^t}(\widetilde  Z_{h+1}^t)\frac{\alpha(n_h^t(s,a))}{n_h^t(s,a)}} + 5 e^{\beta  (H - h)} \frac{\alpha(n_h^t(s,a))}{n_h^t(s,a)} + 4 He^{ \beta  (H - h)} \frac{\alpha(n_h^t(s,a))}{n_h^t(s,a)}\\ &\quad+  \frac{1}{H} \widehat{p}_h^t\Big(\widetilde{Z}^t_{h+1} - Z_{h+1}^{\star}\Big)(s,a)
\end{align*}
\end{lemma}
\begin{proof}
    On the good event $\mathcal{E}^+$ we have:
    \begin{align*}
          | \big(p_h - \widehat{p}_h^t\big)Z_{h+1}^\star(s,a)| \leq \sqrt{2 \operatorname{Var}_{p_h}(Z_{h+1}^\star) \frac{\alpha^\star(n_h^t(s,a))}{n_h^t(s,a)}} + 3e^{\beta (H - h)} \frac{\alpha^\star(n_h^t(s,a))}{n_h^t(s,a)}
    \end{align*}
    We apply lemma~\ref{KL transportation} and lemma~\ref{transportation} successively to transport the variance of $Z_{h+1}^\star$ under $p_h$ to the computable variance of $\widetilde Z_{h+1}^t$ under $\widehat{p}_h^t$ 
\begin{align*}
\operatorname{Var}_{p_h}(Z_{h+1}^{\star}) &\leq 2 \operatorname{Var}_{\widehat{p}^t_h}(Z_{h+1}^{\star})(s,a) + 4 e^{2 \beta  (H - h)} \frac{\alpha(n_h^t(s,a))}{n_h^t(s,a)} \\
&\leq 4 \operatorname{Var}_{\widehat{p}^t_h}(\widetilde{Z}^t_{h+1})(s,a) + 4 e^{\beta  (H - h )} \widehat{p}_h^t(\widetilde{Z}^t_{h+1} - Z_{h+1}^{\star})(s,a) + 4 e^{2 \beta  (H - h )} \frac{\alpha(n_h^t(s,a))}{n_h^t(s,a)}
    \end{align*}
    Hence, using the inequality $\sqrt{a + b} \leq \sqrt{a} + \sqrt{b}$ and then $\sqrt{ab} \leq a \eta + \frac{b}{\eta}$ for $\eta = H$ and using that $\alpha^\star(n,\delta) \leq \alpha(n,\delta)$: 
\begin{align*}
\sqrt{\operatorname{Var}_{p_h}(Z_{h+1}^{\star})\frac{\alpha^\star(n_h^t(s,a))}{n_h^t(s,a)}} &\leq 2\sqrt{\operatorname{Var}_{\widehat{p}_h^t}(\widetilde{Z}_h^t)\frac{\alpha^\star(n_h^t(s,a))}{n_h^t(s,a)}} + \sqrt{4 e^{\beta  (H - h)} \widehat{p}_h^t(\widetilde{Z}^t_{h+1} - Z_{h+1}^{\star})(s,a)\frac{\alpha^\star(n_h^t(s,a))}{n_h^t(s,a)}}\\& + 2 e^{\beta (H - h )} \sqrt{\frac{\alpha(n_h^t(s,a))}{n_h^t(s,a)} \frac{\alpha^\star(n_h^t(s,a))}{n_h^t(s,a)} }\\
& \leq 2\sqrt{\operatorname{Var}_{\widehat{p}_h^t}(\widetilde Z_h^t)\frac{\alpha^\star(n_h^t(s,a))}{n_h^t(s,a)}} + 4 He^{\beta  (H - h)} \frac{\alpha^\star(n_h^t(s,a))}{n_h^t(s,a)} \\ &\quad + \frac{1}{H} \widehat{p}_h^t(\widetilde{Z}^t_{h+1} - Z_{h+1}^{\star})(s,a) + 2 e^{\beta  (H - h )} \frac{\alpha(n_h^t(s,a))}{n_h^t(s,a)}
\end{align*}

Hence, by plugging this upper bound and using again that $\alpha^\star(n,\delta) \leq \alpha(n,\delta)$ we obtain:
\begin{align*}
    | \big(p_h - \widehat{p}_h^t\big)Z_{h+1}^{\star}| &\leq 2\sqrt{2}\sqrt{\operatorname{Var}_{\widehat{p}_h^t}(\widetilde Z_{h+1}^t)\frac{\alpha^\star(n_h^t(s,a))}{n_h^t(s,a)}} + 5 e^{\beta  (H - h )} \frac{\alpha(n_h^t(s,a))}{n_h^t(s,a)} + 4 He^{ \beta  (H - h)} \frac{\alpha^\star(n_h^t(s,a))}{n_h^t(s,a)}\\ &\quad+  \frac{1}{H} \widehat{p}_h^t\Big(\widetilde{Z}^t_{h+1} - Z_{h+1}^{\star}\Big)(s,a)
\end{align*}

\end{proof}
Denote the bonus term:
\begin{equation}\label{bonus positive}
b_h^t(s,a) = 2\sqrt{2}\sqrt{\operatorname{Var}_{\widehat{p}_h^t}(\widetilde Z_{h+1}^t)\frac{\alpha^\star(n_h^t(s,a))}{n_h^t(s,a)}} + 5 e^{\beta  (H - h)} \frac{\alpha(n_h^t(s,a))}{n_h^t(s,a)} +  4 He^{\beta  (H - h)}\frac{\alpha^\star(n_h^t(s,a))}{n_h^t(s,a)}\end{equation}
Now, following \citet{azar17a}, \citet{Zanette2019} and \citet{Menard2021Fast} we define define optimistic and pessimistic state-value function on the exponential transform of $Q_h^\star$ which denoted by $U_h^\star$: \\
\begin{align}\label{optimism_equation_positive}
    \widetilde{U}_h^t(s,a) &= \min \Bigg \{e^{\beta(H - h -  1)},e^{\beta r_h(s,a)} \Bigg[ \widehat{p}_h^t \widetilde Z_h^t (s,a) + b_h^t(s,a) + \frac{1}{H} \widehat{p}_h^t\Big( \widetilde{Z}_{h+1}^t - \underaccent{\widetilde}{Z}_{h+1}^t\Big)(s,a)\Bigg]\Bigg\} \nonumber \\
    \widetilde{Z}_h^t(s) &= \max_{a \in \mathcal{A}} \widetilde{U}_h^t(s,a), \qquad \widetilde{Z}_{H+1}^t(s) = 1 \nonumber \\
    \underaccent{\widetilde}{U}_h^t(s,a) &= \max \Bigg \{1,e^{\beta r_h(s,a)} \Bigg[ \widehat{p}_h^t \underaccent{\widetilde}{Z}_h^t (s,a) - b_h^t(s,a) - \frac{1}{H} \widehat{p}_h^t\Big( \widetilde{Z}_{h+1}^t - \underaccent{\widetilde}{Z}_{h+1}^t\Big)(s,a)\Bigg]\Bigg\} \nonumber \\
    \underaccent{\widetilde}{Z}_h^t(s) &= \max_{a \in \mathcal{A}} \underaccent{\widetilde}{U}_h^t(s,a), \qquad \underaccent{\widetilde}{Z}_{H+1}^t(s) = 1
\end{align}
And we consider the greedy policy:
$$    \pi_h^{t+1}(s) = \arg\max_{a\in \mathcal{A}} \widetilde{U}_h^t(s,a)$$
Now let us prove the optimism lemma:
\begin{lemma}\label{optimism positive}
    With high probability $1-\delta$ we have:
    $$\underaccent{\widetilde}{U}_h^t(s,a) \leq U^{\star}_h(s,a) \leq \widetilde{U}_h^t(s,a)$$
    and 
    $$\underaccent{\widetilde}{Z}_h^t(s) \leq Z^{\star}_h(s) \leq \widetilde{Z}_h^t(s)$$
\end{lemma}
\begin{proof}
    We proceed by induction over $h$. For $h = H+1$ the result is trivially upper bounding and (resp. lower bounding ) $U^{\star}_h$ by $e^{\beta(H-h)}$ and $1$.\\
    Assume the inequality holds for $h' > h$. Fix $(s,a)$ and assume $\widetilde{U}_h^t(s,a) < H$ since otherwise the inequality is trivial, we have that: 
\begin{align*}
    \widetilde{U}_h^t(s,a) - U_h^{\star}(s,a) &= e^{\beta r_h(s,a)}\Bigg[ \widehat{p}_h^t \widetilde Z_{h+1}^t (s,a) + b_h^t(s,a) + \frac{1}{H} \widehat{p}_h^t\Big( \widetilde{Z}_{h+1}^t - \underaccent{\widetilde}{Z}_{h+1}^t\Big)(s,a) - p_h Z_{h+1}^\star(s,a)\Bigg] \\
    &= e^{\beta r_h(s,a)}\Bigg[ \widehat{p}_h^t \Big( \widetilde Z_{h+1}^t (s,a) - Z^{\star}_{h+1}(s,a) \Big) + \Big( \widehat{p}_h^t - p_h\Big) Z_{h+1}^{\star}(s,a)\\
    & \quad + b_h^t(s,a) + \frac{1}{H} \widehat{p}_h^t\Big( \widetilde{Z}_{h+1}^t - \underaccent{\widetilde}{Z}_{h+1}^t\Big)(s,a)\Bigg]
\end{align*}
    But we know by Bernstein inequality that:
    \begin{align*}
        \Big( \widehat{p}_h^t - p_h\Big) Z_h^{\star}(s,a) \geq -b_h^t(s,a) - \frac{1}{H} \widehat{p}_h^t\Big( \widetilde{Z}_{h+1}^t - Z_{h+1}^\star\Big)(s,a)
    \end{align*}
    Hence:
    \begin{align*}
    \widetilde{U}_h^t(s,a) - U_h^{\star}(s,a) &\geq e^{\beta r_h(s,a)}\Bigg[ \Big( 1 - \frac{1}{H} \Big) \widehat{p}_h^t \Big( \widetilde Z_{h+1}^t (s,a) - Z^{\star}_{h+1}(s,a) \Big) + \frac{1}{H} \widehat{p}_h^t\Big( Z_{h+1}^\star - \underaccent{\widetilde}{Z}_{h+1}^t\Big)(s,a)\Bigg] \geq 0
\end{align*}
Where we used the induction hypothesis. We prove the pessimistic property in the same way
\end{proof}
\subsubsection{Stopping rule}
We define the width certificate for the algorithm for the case $\beta > 0$:
\begin{equation}\label{width_certificate_positive}
G_h^t(s,a) = \min \Bigg\{ e^{\beta (H - h)}, e^{\beta r_h^t(s,a)} \Big[ 3b_h^t(s,a)  + \Big(1 +  \frac{3}{H}\Big) \widehat{p}_h^t \pi^{t+1} G_{h+1}^t(s)\Big]\Bigg\}\end{equation}
Lemma~\ref{stopping_positive} establishes the validity of this stopping rule by showing that, with high probability, it bounds the certificate width:
\begin{lemma}\label{stopping_positive}
On the good event $\mathcal E^+$, for all $t$ and all $h$,
$$
Z_h^\star(s)-Z_h^{\pi^{t+1}}(s) \leq \pi_h^{t+1}G_h^t(s)\qquad\forall s\in\mathcal S
$$
In particular, at the initial state $s_1$, $Z_1^\star(s_1)-Z_1^{\pi^{t+1}}(s_1)\leq \pi_1^{t+1}G_1^t(s_1)$
\end{lemma}
We prove the lemma~\ref{stopping_positive} in this section :
We define the auxiliary variable $\mathring{Z}_h^t$. Setting $\mathring{Z}_{H+1}^t \equiv 1$, we recurse backward for $h=H \dots 1$:
\begin{align*}
    \mathring{U}_{h,\text{pes}}^{t} &= \max \left\{ 1, e^{\beta r_h} \left[ \widehat{p}_h^t \mathring{Z}_{h+1}^t - b_h^t - \tfrac{1}{H}\widehat{p}_h^t(\widetilde Z_{h+1}^t - \mathring{Z}_{h+1}^t) \right] \right\} \\
    \mathring{U}_h^{t} &= \min \left\{ e^{\beta r_h} (p_h \mathring{Z}_{h+1}^t), \mathring{U}_{h, \text{pes}}^{t} \right\} \\
    \mathring{Z}_h^{t}(s) &= \mathring{U}_h^{t}\bigl(s, \pi_h^{t+1}(s)\bigr)
\end{align*}
Because $\underaccent{\widetilde}{Z}^t$ is pessimistic against $Z^\star$, we cannot directly compare $\underaccent{\widetilde}{Z}^t$ to $Z^{\pi^{t+1}}$.
Intuitively, $\mathring Z^t$ satisfies the exponential Bellman recursion under the true kernel $p_h$, while being clipped by a pessimistic empirical backup; hence it serves as a worst-case lower bound for both $\underaccent{\widetilde}{Z}^t$ and $Z^{\pi^{t+1}}$ as shows the next lemma:
\begin{lemma}\label{ring_positive}
    For all $(h,s,a)$:
    $$\mathring{U}_h^t(s,a) \leq \min\Big(\underaccent{\widetilde}{U}_h^t(s,a),U_h^{\pi^{t+1}_h}(s,a)\Big)$$
    and 
    $$\mathring{Z}_h^t(s) \leq \min\Big(\underaccent{\widetilde}{Z}_h^t(s),Z_h^{\pi^{t+1}_h}(s,a)\Big)$$
\end{lemma}
\begin{proof}
    We proceed by backward induction. For $h = H+1$, all values are equal to $1$ so the inequalities hold. Assume that for some $h \leq H$ we have for all $(s,a)$:
    $$\mathring{U}_{h+1}^t(s,a) \leq \min\Big(\underaccent{\widetilde}{U}_{h+1}^t(s,a),U_h^{\pi^{t+1}_{h+1}}(s,a)\Big)$$
    and 
    $$\mathring{Z}_{h+1}^t(s) \leq \min\Big(\underaccent{\widetilde}{Z}_{h+1}^t(s),Z_h^{\pi^{t+1}_{h+1}}(s,a)\Big)$$
we have by construction:
$$\mathring{U}_h^t(s,a)  \leq \mathring{U}_{h, \text{true}}^{t}(s, a) =e^{\beta r_h(s,a)}  (p_h \mathring{Z}_{h+1}^t)(s,a) \leq e^{\beta r_h(s,a)}  (p_h Z^{\pi^{t+1}}_{h+1})(s,a) = U_h^{\pi^{t+1}_{h+1}}(s,a)$$
Where we used the induction hypothesis and the monotonicity of the exponential Bellman operator.
Again by construction; 
$$\mathring{U}_h^t(s,a)  \leq \mathring{U}_{h, \text{pes}}^{ t}(s, a)$$
But since we have: 
\begin{align*}\underaccent{\widetilde}{U}_h^t(s,a)-\mathring{U}_{h, \text{pes}}^{t}(s, a) &=  e^{\beta r_h(s,a)} \Big[ \Big(\widehat{p}_h^t \underaccent{\widetilde}{Z}_{h+1}^t (s,a) - b_h^t(s,a) - \frac{1}{H} \widehat{p}_h^t\Big( \widetilde{Z}_{h+1}^t - \underaccent{\widetilde}{Z}_{h+1}^t\Big)(s,a)\Big) \\ &- \Big(\widehat{p}_h^t \mathring{Z}_{h+1}^t (s,a) - b_h^t(s,a) - \frac{1}{H}\widehat{p}_h^t\Big(\widetilde Z_{h+1}^t - \mathring{Z}_{h+1}^t\Big)(s,a) \Big)\Big]\\
&= e^{\beta r_h(s,a)} \Big(1 - \frac{1}{H} \Big) \Big[ \widehat{p}_h^t \big(\underaccent{\widetilde}{Z}_{h+1}^t - \mathring{Z}_{h+1}^t \big) \Big]\\
& \geq 0
\end{align*}
Where we applied the induction hypothesis, we conclude then:
$$\mathring{U}_h^t(s,a)  \leq \underaccent{\widetilde}{U}_h^t(s,a)$$
The bound on $V$ follows immediately and we conclude the recursion
\end{proof}
On the good event $\mathcal{E}^+$ we have:
\begin{lemma}\label{certificate_positive}
On the good event $\mathcal{E}^+$
    $$\widetilde{U}^t_h(s,a) - \mathring{U}_h^t(s,a) \leq e^{\beta r_h^t(s,a)} \Big[ 3b_h^t(s,a)  + \Big(1 +  \frac{3}{H}\Big)\widehat{p}_h^t\Big(\widetilde{Z}^t_{h+1} - \mathring{Z}_{h+1}^t\Big)\Big]$$
\end{lemma}
\begin{proof}
Fix a state-action pair $(s,a)$ and $h \in \{1,...,H\}$, we consider two cases: 

\textbf{First case:} $\mathring{U}_h^t(s,a) = \mathring{U}_{h, \text{true}}^{t}(s, a)$ we then have:
$$\widetilde{U}^t_h(s,a) - \mathring{U}_h^t(s,a) \leq e^{\beta r_h^t(s,a)} \Big[ b_h^t(s,a)  + \frac{1}{H} \widehat{p}_{h}^t \Big(\widetilde{Z}^t_{h+1} - \underaccent{\widetilde}{Z}_{h+1}^t\Big)(s,a)  + \widehat{p}_h^t \widetilde{Z}_{h+1}^t(s,a) - p_h \mathring{Z}_{h+1}^t(s,a) \Big]$$
The last term can be written as: 
\begin{align*}
    \widehat{p}_h^t \widetilde{Z}_{h+1}^t(s,a) - p_h \mathring{Z}_{h+1}^t(s,a)= \widehat{p}_h^t \Big(\widetilde{Z}_{h+1}^t - \mathring{Z}_{h+1}^t\Big)(s,a) + \big(\widehat{p}_h^t - p_h\big) Z_{h+1}^\star + \big(p_h - \widehat{p}_h^t\big) \big(Z_{h+1}^\star - \mathring{Z}_{h+1}^t\big)
\end{align*}
For the second term, by lemma~\ref{concentration_positive}:
\begin{align*}
    | \big(p_h - \widehat{p}_h^t\big)Z_{h+1}^{\star}(s,a)| &\leq b_h^t(s,a)+  \frac{1}{H} \widehat{p}_h^t\Big(\widetilde{Z}^t_{h+1} - Z_{h+1}^{\star}\Big)(s,a) \leq b_h^t(s,a)+  \frac{1}{H} \widehat{p}_h^t\Big(\widetilde{Z}^t_{h+1} - \mathring{Z}_{h+1}^{t}\Big)(s,a)
\end{align*}
For the third term, by the KL-Bernstein inequality ~\ref{KLBernstein} and using the inequality $\sqrt{ab}\leq \frac{a}{H} + bH$:
\begin{align*}
    \big(p_h - \widehat{p}_h^t\big) \big(Z_{h+1}^\star - \mathring{Z}_{h+1}^t\big) &\leq \sqrt{2 \operatorname{Var}_{\widehat{p}_h^t}(Z_{h+1}^\star - \mathring{Z}_{h+1}^t)\frac{\alpha(n_h^t(s,a))}{n_h^t(s,a)}} + \frac{2}{3}e^{\beta (H - h)}\frac{\alpha(n_h^t(s,a))}{n_h^t(s,a)}\\
    &\leq \sqrt{2 e^{\beta (H-h)}\widehat{p}_h^t (Z_{h+1}^\star - \mathring{Z}_{h+1}^t)\frac{\alpha(n_h^t(s,a))}{n_h^t(s,a)}} + \frac{2}{3}e^{\beta (H - h)}\frac{\alpha(n_h^t(s,a))}{n_h^t(s,a)}\\
    &\leq \frac{1}{H} \widehat{p}_h^t (\widetilde{Z}_{h+1}^t - \mathring{Z}_{h+1}^t) + 2 H e^{\beta (H - h)}  \frac{\alpha(n_h^t(s,a))}{n_h^t(s,a)} + \frac{2}{3}e^{\beta (H - h)}\frac{\alpha(n_h^t(s,a))}{n_h^t(s,a)}
\end{align*}
Hence by combining the two bounds:
\begin{align*}
    \widehat{p}_h^t \widetilde{Z}_{h+1}^t(s,a) - p_h \mathring{Z}_{h+1}^t(s,a) \leq 2 b_h^t(s,a) + \Big( 1 + \frac{2}{H}\Big) \widehat{p}_h^t\Big(\widetilde{Z}^t_{h+1} - \mathring{Z}_{h+1}^t\Big)
\end{align*}
Hence by substituting and using lemma~\ref{ring_positive}:
$$\widetilde{U}^t_h(s,a) - \mathring{U}_h^t(s,a) \leq e^{\beta r_h^t(s,a)} \Big[ 3 b_h^t(s,a)  + \Big(1 +  \frac{3}{H}\Big)\widehat{p}_h^t\Big(\widetilde{Z}^t_{h+1} - \mathring{Z}_{h+1}^t\Big)\Big] $$

\textbf{Second case:} $\mathring{U}_h^t(s,a) = \mathring{U}_{h, \text{pes}}^{t}(s, a)$. In this case:
\begin{align*}
\widetilde{U}^t_h(s,a) - \mathring{U}_h^t(s,a) &\leq e^{\beta r_h^t(s,a)} \Big[ b_h^t(s,a) + \frac{1}{H} \widehat{p}_{h}^t (\widetilde{Z}^t_{h+1} - \underaccent{\widetilde}{Z}_{h+1}^t)(s,a) + \widehat{p}_h^t \widetilde{Z}_{h+1}^t(s,a) \\&- \Big(\widehat{p}_h^t \mathring{Z}_{h+1}^t(s,a) - b_h^t(s,a) - \frac{1}{H}\widehat{p}_h^t\big(\widetilde Z_{h+1}^t - \mathring{Z}_{h+1}^t\big)(s,a) \Big) \\
    &=  e^{\beta r_h^t(s,a)} \Big[ 2b_h^t(s,a) +  \left(1 + \frac{1}{H}\right) \widehat{p}_h^t (\widetilde{Z}_{h+1}^t - \mathring{Z}_{h+1}^t)  + \frac{1}{H}(\widetilde{Z}_{h+1}^t - \underaccent{\widetilde}{Z}_{h+1}^t) \Big)(s,a) \Big]
\end{align*}
Hence by lemma~\ref{ring_positive} we find: 
\begin{align*}
    \widetilde{U}^t_h(s,a) - \mathring{U}_h^t(s,a) \leq e^{\beta r_h^t(s,a)} \Big[ 2b_h^t(s,a)  + \Big(1 +  \frac{2}{H}\Big)\widehat{p}_h^t\Big(\widetilde{Z}^t_{h+1} - \mathring{Z}_{h+1}^t\Big)\Big]
\end{align*}
Which conclude the recursion
\end{proof}
We now prove lemma~\ref{stopping_positive}:

\begin{proof}
We first prove by backward induction that, for all $h$ and $s$,
\begin{equation}\label{eq:Vgap_cert}
\widetilde Z_h^t(s)-\mathring Z_h^t(s)\ \le\ (\pi_h^{t+1}G_h^t)(s).
\end{equation}
For $h=H+1$ it holds since both sides are $0$.
Assume it holds at step $h+1$. For $a=\pi_h^{t+1}(s)$
$$
\widetilde Z_h^t(s)-\mathring Z_h^t(s)
=\widetilde U_h^t(s,a)-\mathring U_h^t(s,a).
$$
By lemma~\ref{certificate_positive}, we have
: 
$$\widetilde{U}^t_h(s,a) - \mathring{U}_h^t(s,a) \leq e^{\beta r_h^t(s,a)} \Big[ 3b_h^t(s,a)  + \Big(1 +  \frac{3}{H}\Big)\widehat{p}_h^t\Big(\widetilde{Z}^t_{h+1} - \underaccent{\widetilde}{Z}_{h+1}^t\Big)\Big]$$
By the induction hypothesis inside the tilted expectation:
$$
\widehat{p}_{h}^t(\widetilde Z_{h+1}^t-\mathring Z_{h+1}^t)(s,a)
\leq\widehat{p}_{h}^t(\pi^{t+1}G_{h+1}^t)(s,a).
$$
Thus,
$$
\widetilde Z_h^t(s)-\mathring Z_h^t(s)
\leq 3b_h^t(s,a)+\Big(1+\frac{3}{H}\Big)\widehat{p}_{h}^t(\pi^{t+1}G_{h+1}^t)(s,a)
\leq G_h^t(s,a)=(\pi_h^{t+1}G_h^t)(s)
$$
Finally, use optimism and the ring bridge:
on $\mathcal E$, $Z_h^\star\leq \widetilde Z_h^t$ (lemma~\ref{optimism positive}) and $\mathring Z_h^t\leq Z_h^{\pi^{t+1}}$ (lemma~\ref{ring_positive}), hence
$$
Z_h^\star(s)-Z_h^{\pi^{t+1}}(s)
\leq \widetilde Z_h^t(s)-\mathring Z_h^t(s)
\leq (\pi_h^{t+1}G_h^t)(s)$$
\end{proof}
\subsubsection{Sample complexity}
Now we prove theorem ~\ref{th:upper_bound} for $\beta > 0$
\begin{proof}\label{proof:upper_bound}
    the width certificate is:
$$G_h^t(s,a) = \min \Bigg\{ e^{\beta (H - h)}, e^{\beta r_h^t(s,a)} \Big[ 3b_h^t(s,a)  + \Big(1 +  \frac{3}{H}\Big) \widehat{p}_h^t \pi^{t+1} G_{h+1}^t(s)\Big]\Bigg\}$$
Let us transition to the true MDP. Using Bernstein inequality:
\begin{align*}
 \big|(\widehat{p}_h^t - p_h) \pi^{t+1} G_{h+1}^t(s) \big| \leq \sqrt{2 \operatorname{Var}_{p_h}\big( \pi^{t+1} G_{h+1}^t(s)\big) \alpha(n_h^t(s,a)} + \frac{2}{3} e^{\beta (H - h )} \alpha(n_h^t(s,a)
\end{align*}
Now, we use the inequality $\operatorname{Var}(\pi^{t+1}_{h+1} G_{h+1}^t(s)) \leq e^{\beta (H - h)} \pi^{t+1}_{h+1} G_{h+1}^t(s)$. Hence, using the inequality $\sqrt{xy} \leq x + y$:
\begin{align*}
 \big|(\widehat{p}_h^t - p_h) \pi^{t+1} G_{h+1}^t(s) \big| \leq \frac{1}{H} p_h \pi^{t+1} G_{h+1}^t(s) + 3 He^{\beta (H - h)} \alpha(n_h^t(s,a))
\end{align*}
And using the variance transportation lemmas~\ref{KL transportation},\ref{transportation} and that $ \alpha^\star(n_h^t(s,a)) \leq \alpha(n_h^t(s,a))$:
\begin{align*}
\sqrt{\operatorname{Var}_{\widehat{p}^t_h}(\widetilde Z_{h+1}^t)\frac{\alpha^\star(n_h^t(s,a))}{n_h^t(s,a)}} &\leq 2\sqrt{\operatorname{Var}_{p_h}(Z_{h+1}^{\pi^{t+1}})\frac{\alpha^\star(n_h^t(s,a))}{n_h^t(s,a)}} + \sqrt{4 e^{\beta  (H - h)} p_h(\widetilde{Z}^t_{h+1} - Z_{h+1}^{\pi^{t+1}})(s,a)\frac{\alpha^\star(n_h^t(s,a))}{n_h^t(s,a)}} \\&+ 2 e^{\beta (H - h)} \frac{\alpha(n_h^t(s,a))}{n_h^t(s,a)} \\
& \leq 2\sqrt{\operatorname{Var}_{p_h}(Z_{h+1}^{\pi^{t+1}})\frac{\alpha^\star(n_h^t(s,a))}{n_h^t(s,a)}} + 4 H e^{\beta  (H - h)} \frac{\alpha^\star(n_h^t(s,a))}{n_h^t(s,a)} \\ &\quad + \frac{1}{H} p_h(\widetilde{Z}^t_{h+1} - Z_{h+1}^{\star})(s,a) + 2 e^{\beta  (H - h)} \frac{\alpha(n_h^t(s,a))}{n_h^t(s,a)} \\
& \leq  2\sqrt{\operatorname{Var}_{p_h}(Z_{h+1}^{\pi^{t+1}}r)\frac{\alpha^\star(n_h^t(s,a))}{n_h^t(s,a)}} + 4 H e^{\beta  (H - h)} \frac{\alpha^\star(n_h^t(s,a))}{n_h^t(s,a)} \\ &\quad + \frac{1}{H}p_h \pi^{t+1} G_{h+1}^t(s)+ 2 e^{\beta  (H - h)} \frac{\alpha(n_h^t(s,a))}{n_h^t(s,a)} 
\end{align*}
Hence, by coarsening the constants for the sake of simplicity:
\begin{align*}
b_h^t(s,a) &\leq 4\sqrt{2}\sqrt{\operatorname{Var}_{p_h}(Z_{h+1}^{\pi^{t+1}})\frac{\alpha^\star(n_h^t(s,a))}{n_h^t(s,a)}} +  4(2\sqrt{2} + 1) He^{\beta  (H - h )} \frac{\alpha^\star(n_h^t(s,a))}{n_h^t(s,a)} \\ &\quad +(5 + 4\sqrt{2}) e^{\beta  (H - h )} \frac{\alpha(n_h^t(s,a))}{n_h^t(s,a)} + \frac{2\sqrt{2}}{H} p_h \pi^{t+1} G_{h+1}^t(s) \\
&\leq 6\sqrt{\operatorname{Var}_{p_h}(Z_{h+1}^{\pi^{t+1}})\frac{\alpha^\star(n_h^t(s,a))}{n_h^t(s,a)}} + \frac{2\sqrt{2}}{H} p_h \pi^{t+1} G_{h+1}^t(s) + 27 H e^{\beta (H - h)} \frac{\alpha(n_h^t(s,a))}{n_h^t(s,a)}
\end{align*}
We combine the two terms:
\begin{align*}
    G_h^t(s,a) &\leq e^{\beta r_h(s,a)} \Bigg[36\sqrt{\operatorname{Var}_{p_h}(Z_h^{\pi^{t+1}})\frac{\alpha^\star(n_h^t(s,a))}{n_h^t(s,a)}}+ \frac{6\sqrt{2}}{H} p_h \pi^{t+1} G_{h+1}^t(s) \\&+ 81 H e^{\beta (H - h)} \frac{\alpha(n_h^t(s,a))}{n_h^t(s,a)} + \Big(1 + \frac{3}{H}\Big) p_h \pi^{t+1} G_{h+1}^t(s) \\&+ \Big(1 + \frac{3}{H} \Big) \frac{1}{H} p_h \pi^{t+1} G_{h+1}^t(s) + \Big(1 + \frac{3}{H}\Big) 3 H e^{\beta (H - h)} \frac{\alpha(n_h^t(s,a))}{n_h^t(s,a)} \Bigg]\end{align*}
    Hence, simplifying it gives:
    $$G_h^t(s,a) \leq e^{\beta r_h(s,a)}\Bigg[36\sqrt{\operatorname{Var}_{p_h}(Z_{h+1}^{\pi^{t+1}})\frac{\alpha^\star(n_h^t(s,a))}{n_h^t(s,a)}} + \left(1 + \frac{13}{H} \right) p_h \pi^{t+1} G_{h+1}^t(s) + 84 H e^{\beta (H - h )} \frac{\alpha(n_h^t(s,a))}{n_h^t(s,a)}\Bigg]$$
Unrolling this inequality and using the terminal condition $G_{H+1}^t = 0$ we get:
\begin{align*}
    (\pi_1 G_1^t)(s_1)
    &\leq
    \mathbb{E}^{\pi}\Bigg[
    \sum_{h=1}^{H}
    \kappa^{h-1}
    \exp\Big(\beta\sum_{i=1}^{h} r_i(s_i,a_i)\Big)
    \bigg(
    36\sqrt{\operatorname{Var}_{p_h}\big(Z_{h+1}^{\pi^{t+1}}\big)\alpha^\star\big(n_h^t(s_h,a_h) \wedge 1\big)}
    \\
    &\quad +
    84He^{\beta(H-h)}\alpha\big(n_h^t(s_h,a_h) \wedge 1\big)
    \bigg)
    \Bigg| s_1
    \Bigg]
\end{align*}
where $\kappa = 1 + \frac{13}{11}$. Since we have:
$$\kappa^{h-1} = \Big(1 + \frac{13}{H}\Big)^{h-1} \leq \lim_{H \to +\infty}\Big(1 + \frac{13}{H}\Big)^H = e^{13} $$ we get:
\begin{align*}
(\pi_1 G_1^t)(s_1)
&\leq e^{13}
\mathbb{E}^{\pi^{t+1}}\Bigg[
\sum_{h=1}^{H}
\exp\Big(\beta\sum_{i=1}^{h} r_i(s_i,a_i)\Big)
\Bigg(
36\sqrt{\operatorname{Var}_{p_h}\big(Z_{h+1}^{\pi}\big)\alpha^\star\big(n_h^t(s_h,a_h) \wedge 1\big)}
\\ 
&\quad+
84 H e^{\beta(H-h)}\alpha\big(n_h^t(s_h,a_h) \wedge 1\big)
\Bigg)
\Bigg| s_1
\Bigg]
\end{align*}
The algorithm stops when:
$$\pi_1^\tau G_1(s_1) \leq \frac{e^{|\beta| \varepsilon} - 1}{e^{|\beta| \varepsilon}} \widetilde{Z}^\tau_1(s_1) $$
We upper bound $\frac{\pi_1^t G_1(s_1)}{\widetilde{Z}^t_1(s_1)}$ for $t=1,...,\tau-1$, using optimism: 
\begin{align*}
\frac{\pi_1 G_1^t(s_1)}{\widetilde{Z}_1^t(s_1)}
&\leq \frac{(\pi_1 G_1^t(s_1)}{ Z_1^{\pi^{t+1}}(s_1)} \\ &\leq e^{13}
\mathbb{E}^{\pi^{t+1}}\Bigg[
\sum_{h=1}^{H} 
\exp\Big(\beta\sum_{i=1}^{h} r_i(s_i,a_i)\Big)
\Bigg(
36\sqrt{\frac{\operatorname{Var}_{p_h}\big(Z_{h+1}^{\pi^{t+1}}\big)(s_h,a_h)}{(Z_1^{\pi^{t+1}})^2(s_1)}\alpha\big(n_h^t(s_h,a_h) \wedge 1\big)}
\\&+
84He^{\beta(H-h)}\alpha\big(n_h^t(s_h,a_h) \wedge 1\big)
\Bigg)
\Bigg| s_1
\Bigg]
\end{align*}
Let us bound the first term, using Cauchy-Schwartz inequality:
\begin{align*}
\mathbb{E}^{\pi^{t+1}} & \left[\sum_{h=1}^{H}  
\exp\Big(\beta \sum_{i=1}^{h} r_i(s_i,a_i)\Big)
\Bigg(
\sqrt{\frac{\operatorname{Var}_{p_h}\big(Z_{h+1}^{\pi^{t+1}}\big)}{(Z_1^{\pi^{t+1}})^2}\big(\frac{\alpha^\star\big(n_h^t(s,a)}{n_h^t(s,a)} \wedge 1 \big)}\Big) \Bigg| s_1\right] \\ &=  \sum_{h=1}^H \sum_{s,a} p_h^{t+1}(s,a) \exp\Big(\beta \sum_{i=1}^{h} r_i(s_i,a_i)\Big)
\Bigg(
\sqrt{\frac{\operatorname{Var}_{p_h}\big(Z_{h+1}^{\pi^{t+1}}\big)}{(Z_1^{\pi^{t+1}})^2}\big(\frac{\alpha^\star\big(n_h^t(s,a)}{n_h^t(s,a)} \wedge 1 \big)}\Big) \\
&\leq \sqrt{\sum_{h=1}^H \sum_{s,a}p_h^{t+1}(s,a) \exp\Big(2\beta \sum_{i=1}^{h} r_i(s_i,a_i)\Big) \frac{\operatorname{Var}_{p_h}\big(Z_{h+1}^{\pi^{t+1}}\big)}{(Z_1^{\pi^{t+1}})^2}} \sqrt{\sum_{h=1}^H \sum_{s,a}p_h^{t+1}(s,a) \frac{\alpha^\star(n_h^t(s,a))}{n_h^t(s,a)}}
\end{align*}
For a policy $\pi$. By lemma~\ref{entropicvariance} we have:
\begin{align*}
    \sigma V_h^{\pi}(s) = e^{2\beta r_h(s, \pi(s))} \operatorname{Var}_{p_h}\left(Z_{h+1}^{\pi}\right)(s, \pi(s)) + e^{2\beta r_h(s, \pi(s))} \left(p_h \sigma V_{h+1}^{\pi}\right)(s, \pi(s))
\end{align*}
Multiplying the equation by $\sum_{i=1}^{h-1} r_i(s_i,a_i)$ and take the expectation under $\pi$:
\begin{align*}
    \mathbb{E}^\pi\left[ e^{2\beta \sum_{i=1}^{h-1} r_i(s_i,a_i)} \sigma V_h^\pi(s_h) \right] = \mathbb{E}^\pi\left[ e^{2\beta \sum_{i=1}^h r_i(s_i,a_i)} \operatorname{Var}_{p_h}\left(Z_{h+1}^\pi \right) \right] + \mathbb{E}^\pi\left[ e^{2\beta \sum_{i=1}^h r_i(s_i,a_i)} \sigma V_{h+1}^\pi(s_{h+1}) \right]
\end{align*} 
By summing over $h = 1,...,H$ and since $\sigma V_{H+1}^\pi = 0$ we get:
\begin{align*}
    \sum_{h=1}^H \mathbb{E}^{\pi} \left[ e^{2\beta \sum_{i=1}^h r_i(s_i, a_i)} \frac{\operatorname{Var}_{p_h}\big(Z_{h+1}^{\pi^{t+1}}\big)}{(Z_1^{\pi^{t+1}})^2} \right] = \mathbb{E}^{\pi} \left[ \frac{\sigma V_1^{\pi}(s_1)}{(Z_1^{\pi})^2} \right]
\end{align*}
But notice that for a deterministic policy $\pi$:
$$\frac{\sigma V_1^{\pi}(s_1)}{(Z_1^{\pi})^2} = \frac{\operatorname{Var}(e^{\beta R_1^\pi}|S_1 =s_1)}{\mathbb{E}(e^{\beta R_1^\pi}|S_1 = s_1)^2}$$
Using lemma~\ref{bound} we get that:
$$\frac{\sigma V_1^{\pi}(s_1)}{(Z_1^{\pi})^2} \leq \frac{(e^{\beta G_{\text{max}}(\mathcal{M})} - 1)^2}{e^{\beta G_{\text{max}}(\mathcal{M})}}$$
Applying this for $\pi^{t+1}$ we get:
\begin{align*}
    \sum_{h=1}^H \sum_{s,a}p_h^{t+1}(s,a) \exp\Big(2\beta \sum_{i=1}^{h} r_i(s_i,a_i)\Big) \frac{\operatorname{Var}_{p_h}\big(Z_{h+1}^{\pi^{t+1}}\big)}{(Z_1^{\pi^{t+1}})^2} \leq \frac{(e^{\beta G_{\text{max}}(\mathcal{M})} - 1)^2}{e^{\beta G_{\text{max}}(\mathcal{M})}}
\end{align*}
For the second term:
\begin{align*}
   \mathbb{E}^{\pi^{t+1}} & \left[
\sum_{h=1}^{H}
\exp\Big(\beta\sum_{i=1}^{h} r_i(s_i,a_i)\Big)
He^{\beta(H-h)}\big(\frac{\alpha\big(n_h^t(s,a)}{n_h^t(s,a)} \wedge 1 \big)
\Bigg)
\middle| s_1
\right] \\&=  \sum_{h=1}^H \sum_{s,a} p_h^{t+1}(s,a)\exp\Big(\beta\sum_{i=1}^{h} r_i(s_i,a_i)\Big)
He^{\beta(H-h)}\big(\frac{\alpha\big(n_h^t(s,a)}{n_h^t(s,a)} \wedge 1 \big) \\&\leq  \sum_{h=1}^H \sum_{s,a} p_h^{t+1}(s,a)e^{\beta h} He^{\beta(H-h)}\big(\frac{\alpha\big(n_h^t(s,a)}{n_h^t(s,a)} \wedge 1 \big) \\
&\leq  H e^{\beta H} \sum_{h=1}^H \sum_{s,a} p_h^{t+1}(s,a)\big(\frac{\alpha\big(n_h^t(s,a)}{n_h^t(s,a)} \wedge 1 \big)
\end{align*}
Hence:
\begin{align*}
(\pi_1 G_1^t)(s_1)
&\leq 36 e^{13} \sqrt{\frac{(e^{\beta G_{\text{max}}(\mathcal{M})} - 1)^2}{e^{\beta G_{\text{max}}(\mathcal{M})}}\sum_{h=1}^H \sum_{s,a} p_h^{t+1}(s,a)\big(\frac{\alpha^\star\big(n_h^t(s,a)}{n_h^t(s,a)} \wedge 1 \big)} \\&+ 84 e^{13} e^{\beta H}\sum_{h=1}^H \sum_{s,a} p_h^{t+1}(s,a)\big(\frac{\alpha\big(n_h^t(s,a)}{n_h^t(s,a)} \wedge 1 \big) \\
&\leq 36 e^{13} \sqrt{\frac{(e^{\beta G_{\text{max}}(\mathcal{M})} - 1)^2}{e^{\beta G_{\text{max}}(\mathcal{M})}}}\sqrt{\sum_{h=1}^H \sum_{s,a} p_h^{t+1}(s,a)\big(\frac{\alpha^\star\big(n_h^t(s,a)}{n_h^t(s,a)} \wedge 1 \big)}\\& + 84 e^{13} e^{\beta H}\sum_{h=1}^H \sum_{s,a} p_h^{t+1}(s,a)\big(\frac{\alpha\big(n_h^t(s,a)}{n_h^t(s,a)} \wedge 1 \big)
\end{align*}
Let us sum over $t \leq \tau$. By sub-optimality for each episode $t = 0,...,\tau-1$ we have:
$$\pi_1^{t+1} G_1(s_1) > \frac{e^{|\beta| \varepsilon} - 1}{e^{|\beta| \varepsilon}} \widetilde{Z}^t_1(s_1)$$
Hence by summing over all the episodes and using Cauchy-Schwartz:
\begin{align*}
    \tau \frac{e^{\beta \varepsilon} - 1}{e^{\beta \varepsilon}}  &\leq 36 e^{13} \sum_{t=1}^{\tau - 1} \sqrt{\frac{(e^{\beta G_{\text{max}}(\mathcal{M})} - 1)^2}{e^{\beta G_{\text{max}}(\mathcal{M})}}\sum_{h=1}^H \sum_{s,a} p_h^{t+1}(s,a)\big(\frac{\alpha^\star\big(n_h^t(s,a)}{n_h^t(s,a)} \wedge 1 \big)} \\&+ 84 e^{13} e^{\beta H}\sum_{t=1}^{\tau -1}\sum_{h=1}^H \sum_{s,a} p_h^{t+1}(s,a)\big(\frac{\alpha\big(n_h^t(s,a)}{n_h^t(s,a)} \wedge 1 \big)\\
    &\leq 36 e^{13} \sqrt{\frac{(e^{\beta G_{\text{max}}(\mathcal{M})} - 1)^2}{e^{\beta G_{\text{max}}(\mathcal{M})}}}\sqrt{T}\sqrt{\sum_{t=1}^{\tau - 1}\sum_{h=1}^H \sum_{s,a} p_h^{t+1}(s,a)\big(\frac{\alpha^\star\big(n_h^t(s,a)}{n_h^t(s,a)} \wedge 1 \big)} \\&+ 84 e^{13} e^{\beta H}\sum_{t=1}^{\tau -1}\sum_{h=1}^H \sum_{s,a} p_h^{t+1}(s,a)\big(\frac{\alpha\big(n_h^t(s,a)}{n_h^t(s,a)} \wedge 1 \big)
\end{align*}
Using lemma~\ref{pseudo-counts} to relate the true counts to pseudo-counts we get:
\begin{align*}
    \sum_{t=1}^{\tau -1}\sum_{h=1}^H \sum_{s,a} p_h^{t+1}(s,a)\big(\frac{\alpha\big(n_h^t(s,a)}{n_h^t(s,a)} \wedge 1 \big) &\leq \sum_{t=1}^{\tau -1}\sum_{h=1}^H \sum_{s,a} p_h^{t+1}(s,a)\alpha\big(\widetilde{n}_h^t(s,a) \vee 1\big) \\
    &\leq \alpha(\tau-1,\delta)\sum_{t=1}^{\tau -1}\sum_{h=1}^H \sum_{s,a} p_h^{t+1}(s,a)\frac{1}{\widetilde{n}_h^t(s,a) \vee 1} \\
    &\leq \alpha(\tau-1,\delta) \sum_{t=1}^{\tau -1}\sum_{h=1}^H \sum_{s,a}\frac{\widetilde{n}_{h+1}^t(s,a) - \widetilde{n}_h^t(s,a)}{\widetilde{n}_h^t(s,a) \vee 1} \\
    &\leq 3 SAH\alpha(\tau-1,\delta) \log(\tau + 1)
\end{align*}
Where in the final inequality we used lemma~\ref{inequality}. Similarly, we find:
\begin{align*}
    \sum_{t=1}^{\tau -1}\sum_{h=1}^H \sum_{s,a} p_h^{t+1}(s,a)\big(\frac{\alpha^\star\big(n_h^t(s,a)}{n_h^t(s,a)} \wedge 1 \big) \leq 3 SAH\alpha^\star(\tau-1,\delta) \log(\tau + 1)
\end{align*}
Hence:
\begin{align*}
    \tau \frac{e^{\beta \varepsilon} - 1}{e^{\beta \varepsilon}} \leq  36 e^{13}  \sqrt{\frac{(e^{\beta G_{\text{max}}(\mathcal{M})} - 1)^2}{e^{\beta G_{\text{max}}(\mathcal{M})}}\tau SAH\alpha^\star(t-1,\delta) \log(t + 1)} + 84 e^{13} e^{\beta H}SAH\alpha(t-1,\delta) \log(t + 1)
\end{align*}
Therefore, by replacing $\alpha^*$ and $\alpha$ by their expression and using that $\log(\tau + 1) \leq \log(8 e \tau)$ since $\tau \geq 1$:
\begin{align*}
    \tau \frac{e^{\beta \varepsilon} - 1}{e^{\beta \varepsilon}}  &\leq 36 e^{13}  \sqrt{\frac{(e^{\beta G_{\text{max}}(\mathcal{M})} - 1)^2}{e^{\beta G_{\text{max}}(\mathcal{M})}}\tau SAH\Big(\log\big(\frac{3 SAH}{\delta}\big)\log\big(8 e \tau \big) + \log\big( 8 e \tau\big)^2\Big)}\\& + 84 e^{13} e^{\beta H}SAH\Big(\log\big(\frac{3 SAH}{\delta}\big)\log\big(8 e \tau \big) + S\log\big( 8 e \tau\big)^2\Big)
\end{align*}
Finally, we use lemma~\ref{inequality} with :
\begin{align*}
    &C= 36 e^{13}  \frac{e^{\beta \varepsilon}}{e^{\beta \varepsilon} - 1} \sqrt{(\frac{(e^{\beta G_{\text{max}}(\mathcal{M})} - 1)^2}{e^{\beta G_{\text{max}}(\mathcal{M})}} SAH} \quad, A = \log(\frac{3SAH}{\delta}) \quad, B = 1\\&\quad \quad \quad \quad  D = \frac{e^{\beta \varepsilon}}{e^{\beta \varepsilon} - 1}84e^{13} e^{\beta H} H^2 SA \quad \text{and} \quad E = S\end{align*}
    Which yield:
\begin{align*}
    \tau &\leq \frac{e^{2\beta \varepsilon}}{\big(e^{\beta \varepsilon} - 1\big)^2}\frac{(e^{\beta G_{\text{max}}(\mathcal{M})} - 1)^2}{e^{\beta G_{\text{max}}(\mathcal{M})}} SAH \bigg(\log(\frac{3SAH}{\delta}) + 1 \bigg) C_1^2 + 3 e^{\beta \varepsilon}\frac{e^{\beta H} H^2 SA}{e^{\beta \varepsilon} - 1}\bigg(\log(\frac{3SAH}{\delta}) + S\bigg)C_1^2 + 1
\end{align*}
Where $C_1 = \frac{8}{5} \log\bigg(4 e^{17} \frac{(S+1)(H+1) e^{|\beta H} S A H^2}{e^{\beta \varepsilon} - 1}\bigg) $

In particular, if $\varepsilon$ is small enough so that the first term dominates the second then:
\begin{align*}
    \tau \leq\frac{e^{2\beta \varepsilon}}{\big(e^{\beta \varepsilon} - 1\big)^2}\frac{(e^{\beta G_{\text{max}}(\mathcal{M})} - 1)^2}{e^{\beta G_{\text{max}}(\mathcal{M})}} SAH \log(\frac{3SAH}{\delta}) C_2^2
\end{align*}
Where $C_2 = 3e C_1$. We can finally hide the constants and the log terms to get:
\[\tau = \tilde{\mathcal{O}}\left( \frac{1}{\big(e^{\beta \varepsilon} - 1\big)^2}\frac{(e^{\beta G_{\text{max}}(\mathcal{M})} - 1)^2}{e^{\beta G_{\text{max}}(\mathcal{M})}} SAH \right)\]
Finally to see that the stopping rule implies $(\varepsilon,\delta)$PAC, remark that at time $\tau$:
$$\pi_1^\tau G_1(s_1) \leq \frac{e^{\beta \varepsilon} - 1}{e^{\beta \varepsilon}} \widetilde{Z}^t_1(s_1) $$
This is equivalent to :
$$\pi_1^\tau G_1(s_1) \leq (e^{|\beta| \varepsilon} - 1) \Bigg(\widetilde{Z}_1^t(s_1) - \pi_{1}^\tau G_1(s_1)\Bigg)$$
Since $
\widetilde{Z}_1^t(s_1)-Z_1^{\star}(s_1)\ \leq \widetilde{Z}_1^t(s_1)-\mathring{Z}_1^{t}(s_1)\leq \pi_{t+1,1}G_1^t(s_1)
$, this stopping condition is stronger than the condition:
$$\pi_1^\tau G_1(s_1) \leq (e^{\beta \varepsilon} - 1) Z_1^{\pi^{\tau}}$$
But we can write:
\begin{align*}
    (V^\star_1 - V^{\pi^{\tau+1}}_1)(s_1) = \frac{1}{\beta} \log\left( \frac{Z^\star_1}{Z_1^{\pi^{\tau}}}\right)(s_1) 
    = \frac{1}{\beta} \log \left(1 + \frac{Z_1^\star - Z_1^{\pi^{\tau}}}{Z_1^{\pi^{\tau}}} \right)(s_1) \leq \frac{1}{\beta} \log\left(1 + \frac{\pi^\tau_1 G_1(s_1)}{Z_1^{\pi^{\tau}}} \right)(s_1) \leq \varepsilon
\end{align*}
\end{proof}
\subsection{Case \texorpdfstring{$\beta <0$}{beta < 0}}\label{beta_negative}
\subsubsection{Stopping rule}
We first discuss the stopping rule for $\beta < 0$. The difference in value functions can be expressed in the exponential space as:
\begin{align*}
    (V^\star_1 - V^{\pi^{t+1}}_1)(s_1) &= \frac{1}{\beta} \log\left( \frac{Z^\star_1}{Z_1^{\pi^{t+1}}}\right)(s_1) \\
    &= \frac{1}{\beta} \log \left(1 + \frac{Z_1^\star - Z_1^{\pi^{t+1}}}{Z_1^{\pi^{t+1}}} \right)(s_1).
\end{align*}
To ensure the policy is $\varepsilon$-optimal, it suffices to satisfy at stopping time $\tau$:
\begin{equation}\label{true_optimality_condition_negative_true}
    \pi_1^t G_1(s_1) \leq \left(e^{\beta \varepsilon} - 1\right) Z_1^{\pi^{t}}(s_1)
\end{equation}
However, $Z_1^{\pi^t}$ is unknown as it relies on the true dynamics. We therefore substitute it with a computable lower bound:
\begin{equation}\label{true_optimality_condition_negative}
    \pi_1^t G_1(s_1) \leq \left(e^{\beta \varepsilon} - 1\right) \underaccent{\widetilde}{Z}_1^{\pi^{t}}(s_1)
\end{equation}
where both $\pi_1^t G_1(s_1)$ and $\underaccent{\widetilde}{Z}_1^t(s_1)$ can be computed efficiently by dynamic programming. 
\subsubsection{Confidence Bounds}
We first start with a lemma:
\begin{lemma}\label{concentration_negative}
    On the good event $\mathcal{E}^-$:
    \begin{align*}
    | \big(p_h - \widehat{p}_h^t\big)Z_{h+1}^\star| &\leq 2\sqrt{2}\sqrt{\operatorname{Var}_{\widehat{p}_h^t}(\underaccent{\widetilde}{Z}_h^t)(s,a)\frac{\alpha^\star(n_h^t(s,a))}{n_h^t(s,a)}} + 5 (1 - e^{\beta  (H - h)}) \alpha(n_h^t(s,a))\\& + 4 H(1 - e^{ \beta  (H - h H)}) \frac{\alpha^\star(n_h^t(s,a))}{n_h^t(s,a)}+  \frac{1}{H} \widehat{p}_h^t\Big( Z_{h+1}^\star - \underaccent{\widetilde}{Z}_{h+1}^t\Big)(s,a)
\end{align*}
\end{lemma}
\begin{proof}
On the good event $\mathcal{E}^-$ we have:
\begin{align*}
    \big| (p_h - \widehat{p}_h^t) Z_{h+1}^\star (s,a)| \leq \sqrt{2 \operatorname{Var}_{p_h}(Z_{h+1}^*) \frac{\alpha^\star(n_h^t(s,a))}{n_h^t(s,a)}} + 3(1 - e^{\beta (H - h)}) \frac{\alpha^\star(n_h^t(s,a))}{n_h^t(s,a)}
\end{align*}
Since $Z_h^{\star}$ and $p_h$ are unknown, we use a variance transportation inequality to replace with its value for $\widetilde Z_h^t$ the optimistic bound for $V^{\star}$. By applying lemma~\ref{KL transportation} and lemma~\ref{transportation} successively:
\begin{align*}
\operatorname{Var}_{p_h}(Z_{h+1}^\star) &\leq 2 \operatorname{Var}_{\widehat{p}^t_h}(Z_{h+1}^\star)(s,a) + 4 (1 - e^{\beta  (H - h)})^2 \frac{\alpha(n_h^t(s,a))}{n_h^t(s,a)} \\
&\leq 4 \operatorname{Var}_{\widehat{p}^t_h}(\underaccent{\widetilde}{Z}_{h+1}^t)(s,a) + 4 (1 - e^{\beta  (H - h)}) \widehat{p}_h^t(Z_{h+1}^\star - \underaccent{\widetilde}{Z}_{h+1}^t)(s,a) + 4 (1 - e^{\beta  (H - h)}) \frac{\alpha(n_h^t(s,a))}{n_h^t(s,a)}
\end{align*}
Hence, using the inequality $\sqrt{a + b} \leq \sqrt{a} + \sqrt{b}$ and then $\sqrt{ab} \leq a \eta + \frac{b}{\eta}$ for $\eta = H $ and using that $\alpha^\star(n_h^t(s,a)) \leq \alpha(n_h^t(s,a))$: 
\begin{align*}
\sqrt{\operatorname{Var}_{p_h}(Z_h^\star)(s,a) \frac{\alpha^\star(n_h^t(s,a))}{n_h^t(s,a)}} &\leq 2\sqrt{\operatorname{Var}_{\widehat{p}_h^t}(\underaccent{\widetilde}{Z}_h^t)(s,a)\frac{\alpha^\star(n_h^t(s,a))}{n_h^t(s,a)}} \\&+ \sqrt{4(1- e^{\beta  (H - h)}) \widehat{p}_h^t(Z_{h+1}^\star - \underaccent{\widetilde}{Z}_{h+1}^t)(s,a)\frac{\alpha^\star(n_h^t(s,a))}{n_h^t(s,a)}}\\& + 2 (1 - e^{\beta (H - h)}) \sqrt{\frac{\alpha^\star(n_h^t(s,a))}{n_h^t(s,a)} \frac{\alpha(n_h^t(s,a))}{n_h^t(s,a)}}\\
& \leq 2\sqrt{\operatorname{Var}_{\widehat{p}_h^t}(\underaccent{\widetilde}{Z}_h^t)(s,a)\frac{\alpha^\star(n_h^t(s,a))}{n_h^t(s,a)}} + 4 H(1 - e^{\beta  (H - h)}) \frac{\alpha^\star(n_h^t(s,a))}{n_h^t(s,a)} \\ &\quad + \frac{1}{H} \widehat{p}_h^t(Z_{h+1}^\star - \underaccent{\widetilde}{Z}_{h+1}^t)(s,a) + 2(1 - e^{\beta  (H - h)}) \frac{\alpha(n_h^t(s,a))}{n_h^t(s,a)}
\end{align*}

Hence:
\begin{align*}
    | p_h Z_{h+1}^\star - \widehat{p}_h^t(s,a)Z_{h+1}^\star| &\leq 2\sqrt{2}\sqrt{\operatorname{Var}_{\widehat{p}_h^t}(\underaccent{\widetilde}{Z}_{h+1}^t)(s,a)\frac{\alpha^\star(n_h^t(s,a))}{n_h^t(s,a)}} + 5 (1 - e^{\beta  (H - h)}) \alpha(n_h^t(s,a))\\& + 4 H(1 - e^{ \beta  (H - h)}) \frac{\alpha^\star(n_h^t(s,a))}{n_h^t(s,a)}+  \frac{1}{H} \widehat{p}_h^t\Big(Z_{h+1}^\star - \underaccent{\widetilde}{Z}_{h+1}^t\Big)(s,a)
\end{align*}
\end{proof}
Denote the bonus term as:
\begin{equation}\label{bonus_negative}
b_h^t(s,a) =2\sqrt{2}\sqrt{\operatorname{Var}_{\widehat{p}_h^t}(\underaccent{\widetilde}{Z}_h^t)(s,a)\frac{\alpha^\star(n_h^t(s,a))}{n_h^t(s,a)}} + 5 (1 - e^{\beta  (H - h)}) \frac{\alpha(n_h^t(s,a))}{n_h^t(s,a)} + 4 H(1 - e^{ \beta  (H - h)}) \frac{\alpha^\star(n_h^t(s,a))}{n_h^t(s,a)}\end{equation}
Now, like the case $\beta > 0$ we define define optimistic and pessimistic state-value function on the exponential transform of $Q_h^\star$ which denoted by $U_h^\star$. \\
\begin{align}\label{optimism_equation_negative}
    \widetilde{U}_h^t(s,a) &= \min \Bigg \{1,e^{\beta r_h(s,a)} \Bigg[ \widehat{p}_h^t \widetilde Z_{h+1}^t (s,a) + b_h^t(s,a) + \frac{1}{H} \widehat{p}_h^t\Big(  \widetilde{Z}_{h+1}^t - \underaccent{\widetilde}{Z}_{h+1}^t\Big)(s,a)\Bigg]\Bigg\} \nonumber\\
    \widetilde{Z}_h^t(s) &= \min_{a \in \mathcal{A}} \widetilde{U}_h^t(s,a), \qquad \widetilde{Z}_{H+1}^t(s) = 1  \nonumber\\
    \underaccent{\widetilde}{U}_h^t(s,a) &= \max \Bigg \{e^{\beta(H - h -  1)},e^{\beta r_h(s,a)} \Bigg[ \widehat{p}_h^t \underaccent{\widetilde}{Z}_{h+1}^t (s,a) - b_h^t(s,a) - \frac{1}{H} \widehat{p}_h^t\Big( \widetilde{Z}_{h+1}^t - \underaccent{\widetilde}{Z}_{h+1}^t\Big)(s,a)\Bigg]\Bigg\} \nonumber \\
    \underaccent{\widetilde}{Z}_h^t(s) &= \min_{a \in \mathcal{A}} \underaccent{\widetilde}{U}_h^t(s,a), \qquad \underaccent{\widetilde}{Z}_{H+1}^t(s) = 1 
\end{align}
And we consider the greedy policy:
$$  \pi_h^{t+1}(s) = \arg\min_{a\in \mathcal{A}} \underaccent{\widetilde}{U}_h^t(s,a)$$
Let us prove an optimism lemma:
\begin{lemma}\label{optimism_negative}
    On the good event $\mathcal{E}^-$ we have:
    $$\underaccent{\widetilde}{U}_h^t(s,a) \leq U^{\star}_h(s,a) \leq \widetilde{U}_h^t(s,a)$$
    and 
    $$\underaccent{\widetilde}{Z}_h^t(s) \leq Z^{\star}_h(s) \leq \widetilde{Z}_h^t(s)$$
\end{lemma}
\begin{proof}
    We proceed by induction over $h$. For $h = H+1$ the result is trivially upper bounding and (resp. lower bounding ) $Q^{\star}$ by H and 1.\\
    Assume the inequality holds for $h' > h$. Fix $(s,a)$ and assume $\widetilde{Q}_h^t(s,a) < H$ .we have that: 
\begin{align*}
    \widetilde{U}_h^t(s,a) - U_h^{\star}(s,a) &= e^{\beta r_h(s,a)}\Bigg[ \widehat{p}_h^t \widetilde Z_{h+1}^t (s,a) + b_h^t(s,a) + \frac{1}{H} \widehat{p}_h^t\Big( \widetilde{Z}_{h+1}^t - \underaccent{\widetilde}{Z}_{h+1}^t\Big)(s,a) - p_h Z_{h+1}^\star(s,a)\Bigg] \\
    &= e^{\beta r_h(s,a)}\Bigg[ \widehat{p}_h^t \Big( \widetilde Z_{h+1}^t (s,a) - Z^{\star}_{h+1}(s,a) \Big) + \Big( \widehat{p}_h^t - p_h\Big) Z_{h+1}^{\star}(s,a)\\
    & \quad + b_h^t(s,a) + \frac{1}{H} \widehat{p}_h^t\Big( \widetilde{Z}_{h+1}^t - \underaccent{\widetilde}{Z}_{h+1}^t\Big)(s,a)\Bigg]
\end{align*}
    But we know by Bernstein inequality that:
    \begin{align*}
        \Big( \widehat{p}_h^t - p_h\Big) Z_h^{\star}(s,a) \geq -b_h^t(s,a) - \frac{1}{H} \widehat{p}_h^t\Big(Z_{h+1}^\star - \underaccent{\widetilde}{Z}_{h+1}^t\Big)(s,a) \geq -b_h^t(s,a) - \frac{1}{H} \widehat{p}_h^t\Big( \widetilde{Z}_{h+1}^t - Z_{h+1}^\star\Big)(s,a)
    \end{align*}
    Hence:
    \begin{align*}
    \widetilde{U}_h^t(s,a) - U_h^{\star}(s,a) &\geq e^{\beta r_h(s,a)}\Bigg[ \Big( 1 + \frac{1}{H} \Big) \widehat{p}_h^t \Big( \widetilde Z_{h+1}^t (s,a) - Z^{\star}_{h+1}(s,a) \Big) \Bigg] \geq 0
\end{align*}
Where we used the induction hypothesis. We prove the pessimistic property in the same way
\end{proof}
\subsubsection{Stopping rule}
We define the stopping rule for the algorithm for $\beta <0$ as:
    \begin{align}\label{width_certificate_negative}
        G_h^t(s,a) = \min\Bigg\{1 , e^{\beta r_h(s,a)} \Big[ 3 b_h^t(s,a) + \Big( 1 + \frac{3}{H} \Big) \widehat{p}_h^t \pi^{t+1}G_{h+1}^t(s,a)\Big] \Bigg\}
    \end{align}
Lemma~\ref{stopping_negative} establishes the validity of this stopping rule by showing that, with high probability, it bounds the certificate width:
\begin{lemma}\label{stopping_negative}
On the good event $\mathcal E^+$, for all $t$ and all $h$,
$$
Z_h^{\pi^{t+1}}(s)- Z_h^\star(s) \leq \pi_h^{t+1}G_h^t(s)\qquad\forall s\in\mathcal S
$$
In particular, at the initial state $s_1$, $V_1^\star(s_1)-V_1^{\pi^{t+1}}(s_1)\leq \pi_1^{t+1}G_1^t(s_1)$
\end{lemma}
We prove the lemma~\ref{stopping_negative} in this section :
Like the case $\beta < 0$. We define the auxiliary (analysis-only) variable $\mathring{Z}_h^t$. Setting $\mathring{Z}_{H+1}^t \equiv 1$, we recurse backward for $h=H,\dots,1$:
\begin{align*}
    \mathring{U}_{h,\text{opt}}^{t}(s,a)
    &= \min \Bigl\{ 1,
    e^{\beta r_h(s,a)} \Bigl[
        (\widehat{p}_h^t \mathring{Z}_{h+1}^t)(s,a) + b_h^t(s,a)
        + \tfrac{1}{H}\bigl(\widehat{p}_h^t(\mathring{Z}_{h+1}^t-\underaccent{\widetilde}{Z}{}_{h+1}^t)\bigr)(s,a)
    \Bigr]\Bigr\}, \\
    \mathring{U}_h^{t}(s,a)
    &= \max \Bigl\{
        e^{\beta r_h(s,a)} (p_h \mathring{Z}_{h+1}^t)(s,a),
        \mathring{U}_{h,\text{opt}}^{t}(s,a)
    \Bigr\}, \\
    \mathring{Z}_h^{t}(s)
    &= \mathring{U}_h^{t}\bigl(s,\pi_h^{t+1}(s)\bigr).
\end{align*}

Because $\widetilde{Z}^t$ is pessimistic with respect to $Z^\star$ (and $\beta<0$ reverses the relevant order), we cannot directly compare it to $Z^{\pi^{t+1}}$. We introduce $\mathring Z^t$ as a bridge quantity. Intuitively, $\mathring Z^t$ satisfies the exponential Bellman recursion under the true kernel $p_h$ while being clipped by an optimistic empirical backup; hence it serves as a worst-case upper bound for both $\widetilde{Z}^t$ and $Z^{\pi^{t+1}}$.
\begin{lemma}\label{ring_negative}
    For all $t$, For all $(h,s,a)$:
    $$\mathring{U}_h^t(s,a) \geq \max\Big(\widetilde{U}_h^t(s,a),U_h^{\pi^{t+1}_h}(s,a)\Big)$$
    and 
    $$\mathring{Z}_h^t(s) \geq \max\Big(\widetilde{Z}_h^t(s),Z_h^{\pi^{t+1}_h}(s,a)\Big)$$
\end{lemma}
\begin{proof}
    We proceed by backward induction. For $h = H+1$, all values are equal to 0 so the inequalities hold. Assume that for some $h \leq H$ we have for all $(s,a)$:
    $$\mathring{U}_{h+1}^t(s,a) \geq \max\Big(\widetilde{U}_{h+1}^t(s,a),U_{h+1}^{\pi^{t+1}h}(s,a)\Big)$$
    and 
    $$\mathring{Z}_{h+1}^t(s) \geq \max\Big(\widetilde{Z}_{h+1}^t(s),Z_{h+1}^{\pi^{t+1}}(s,a)\Big)$$
we have by construction:
$$\mathring{U}_h^t(s,a)  \geq \mathring{U}_{h, \text{true}}^{t}(s, a) =e^{\beta r_h(s,a)}  (p_h \mathring{Z}_{h+1}^t)(s,a) \geq e^{\beta r_h(s,a)}  (p_h Z^{\pi^{t+1}}_{h+1})(s,a) = U_h^{\pi^{t+1}}(s,a)$$
Where we used the induction hypothesis and the monotonicity of the exponential Bellman operator.
\begin{align*}
    \mathring{U}_h^t(s,a) - \widetilde{U}_h^t(s,a) &\geq \mathring{U}_{h, \text{opt}}^{ t}(s, a) - \widetilde{U}_h^t(s,a) \\
    &\geq e^{\beta r_h(s,a)} \Bigg[ \widehat{p}_h^t \mathring{Z}_{h+1}^t(s,a) + b_h^t(s,a) + \frac{1}{H}\widehat{p}_h^t\Big(\mathring{Z}_{h+1}^t - \underaccent{\widetilde} Z_{h+1}^t \Big)(s,a)\\& - \Big( \widehat{p}_h^t \widetilde Z_h^t (s,a) + b_h^t(s,a) + \frac{1}{H} \widehat{p}_h^t\Big(   \widetilde{Z}_{h+1}^t - \underaccent{\widetilde}{Z}_{h+1}^t\Big)(s,a)\Big) \Bigg]\\
    &\geq  e^{\beta r_h(s,a)} \Bigg[\left(1 + \frac{1}{H}\right) \widehat{p}_h^t \left( \mathring{Z}_{h+1}^t - \widetilde Z_{h+1}^t \right)(s,a)\Bigg]
\end{align*} 
Where we applied the induction hypothesis, we conclude then:
$$\mathring{U}_h^t(s,a)  \geq \widetilde{U}_h^t(s,a)$$
For $Z$:
\begin{align*}
   \mathring{Z}_h^t(s) - Z_h^{\pi^{t+1}}(s,a) =  \mathring{U}_h^t(s,\pi_h^{t}(s)) - U_h^{\pi^{t+1}}(s,\pi_h^{t+1}(s)) \geq 0
\end{align*}
And:
\begin{align*}
     \mathring{Z}_h^t(s)  = \mathring{U}_h^t(s,\pi_h^{t}(s)) \geq \widetilde{U}_h^t(s,\pi_h^{t}(s)) \geq \min_{a \in \mathcal{A}} \widetilde{U}_h^t(s,a) = \widetilde{Z}_h^t(s,a)
\end{align*}
Which conclude the recurrence
\end{proof}
\begin{lemma}\label{certificate_negative}
    For any$t$, any $h \in \{1,...,H\}$ and any state-action pair:
    \begin{align*}
    \mathring{U}_h^t(s,a) - \underaccent{\widetilde}{U}_h^t(s,a)  \leq e^{\beta r_h(s,a)} \Big[ 3 b_h^t(s,a) + \Big( 1 + \frac{3}{H} \Big) \widehat{p}_h^t \Big(\mathring{Z}_h^t(s,a) - \underaccent{\widetilde}{Z}_{h+1}^t\Big)(s,a) \Big]
\end{align*}
\end{lemma}
\begin{proof}
    Fix $h \in\{1,...,H\}$ and fix a state-action pair $(s,a)$. We have two cases:
    
\textbf{First case:} $\mathring{U}_h^t(s,a) = \mathring{U}_{h, \text{true}}^{t}(s, a)$, we have:
\begin{align*}
    \mathring{U}_h^t(s,a) - \underaccent{\widetilde}{U}_h^t(s,a) \leq e^{\beta r_h(s,a)} \Big[  b_h^t(s,a) + \frac{1}{H} \widehat{p}_h^t\Big( \widetilde{Z}_{h+1}^t -\underaccent{\widetilde}{Z}_{h+1}^t \Big)(s,a) + p_h \mathring{Z}_{h+1}^t(s,a) - \widehat{p}_h^t \underaccent{\widetilde}{Z}_{h+1}^t (s,a) \Big]
    \end{align*}
The last term can be written as:
\begin{align*}
     p_h \mathring{Z}_{h+1}^t(s,a) - \widehat{p}_h^t \underaccent{\widetilde}{Z}_{h+1}^t (s,a) = \widehat{p}_h^t \Big(\mathring{Z}_{h+1}^t - \underaccent{\widetilde}{Z}_{h+1}^t\Big)(s,a) +  \big(\widehat{p}_h^t - p_h\big) Z_{h+1}^\star(s,a) +  \big(p_h - \widehat{p}_h^t\big) \big(\mathring{Z}_{h+1}^t - Z_{h+1}^\star \big)
\end{align*}
For the second term, by lemma~\ref{concentration_negative}:
\begin{align*}
    | \big(p_h - \widehat{p}_h^t\big)Z_{h+1}^{\star}(s,a)| &\leq b_h^t(s,a)+  \frac{1}{H} \widehat{p}_h^t\Big(\widetilde{Z}^t_{h+1} - Z_{h+1}^{\star}\Big)(s,a) \leq b_h^t(s,a)+  \frac{1}{H} \widehat{p}_h^t\Big(\widetilde{Z}^t_{h+1} - \mathring{Z}_{h+1}^{t}\Big)(s,a)
\end{align*}
For the third term, by the KL-Bernstein inequality ~\ref{KLBernstein} and using the inequality $\sqrt{ab}\leq \frac{a}{H} + bH$:
\begin{align*}
    \big(p_h - \widehat{p}_h^t\big) \big(Z_{h+1}^\star - \mathring{Z}_{h+1}^t\big) &\leq \sqrt{2 \operatorname{Var}_{\widehat{p}_h^t}(Z_{h+1}^\star - \mathring{Z}_{h+1}^t)\frac{\alpha(n_h^t(s,a))}{n_h^t(s,a)}} + \frac{2}{3}\left(1 - e^{\beta (H - h)}\right)\frac{\alpha(n_h^t(s,a))}{n_h^t(s,a)}\\
    &\leq \sqrt{2 e^{\beta (H-h)}\widehat{p}_h^t (Z_{h+1}^\star - \mathring{Z}_{h+1}^t)\frac{\alpha(n_h^t(s,a))}{n_h^t(s,a)}} + \frac{2}{3}\left(1 - e^{\beta (H - h)}\right)\frac{\alpha(n_h^t(s,a))}{n_h^t(s,a)}\\
    &\leq \frac{1}{H} \widehat{p}_h^t (\widetilde{Z}_{h+1}^t - \mathring{Z}_{h+1}^t) + 2 H e^{\beta (H - h)}  \frac{\alpha(n_h^t(s,a))}{n_h^t(s,a)} + \frac{2}{3}\left(1 - e^{\beta (H - h)}\right)\frac{\alpha(n_h^t(s,a))}{n_h^t(s,a)}
\end{align*}
Hence by combining the two bounds:
\begin{align*}
    p_h \mathring{Z}_{h+1}^t(s,a) - \widehat{p}_h^t \underaccent{\widetilde}{Z}_{h+1}^t(s,a) \leq 2 b_h^t(s,a) + \Big( 1 + \frac{2}{H}\Big) \widehat{p}_h^t\Big(\mathring{Z}_{h+1}^t - \underaccent{\widetilde}{Z}^t_{h+1} \Big)
\end{align*}
Hence by substituting and using lemma~\ref{ring_negative}:
$$\mathring{U}_h^t(s,a) - \underaccent{\widetilde}{U}_h^t(s,a) \leq e^{\beta r_h^t(s,a)} \Big[ 3 b_h^t(s,a)  + \Big(1 +  \frac{3}{H}\Big)\widehat{p}_h^t\Big(\widetilde{Z}^t_{h+1} - \mathring{Z}_{h+1}^t\Big)\Big] $$

\textbf{Second case:} $\mathring{U}_h^t(s,a) = \mathring{U}_{h, \text{opt}}^{t}(s, a)$, we have:
\begin{align*}
    \mathring{U}_h^t(s,a) - \underaccent{\widetilde}{U}_h^t(s,a)  &\leq e^{\beta r_h(s,a)} \Bigg[ \widehat{p}_h^t \mathring{Z}_{h+1}^t(s,a) + b_h^t(s,a) + \frac{1}{H}\widehat{p}_h^t\Big(\mathring{Z}_{h+1}^t - \underaccent{\widetilde} Z_{h+1}^t \Big)(s,a) \\&- \Bigg(\widehat{p}_h^t \underaccent{\widetilde}{Z}_{h+1}^t (s,a) - b_h^t(s,a) - \frac{1}{H} \widehat{p}_h^t\Big(\widetilde{Z}_{h+1}^t  - \underaccent{\widetilde}{Z}_{h+1}^t \Big)(s,a) \Bigg) \Bigg]\\
    &= e^{\beta r_h(s,a)} \Bigg[ 2 b_h^t(s,a) +\left(1 + \frac{1}{H}\right)\widehat{p}_h^t \Big(\mathring{Z}_{h+1}^t - \underaccent{\widetilde} Z_{h+1}^t \Big)(s,a) + \frac{1}{H}\widehat{p}_h^t \Big(\widetilde{Z}_{h+1}^t - \underaccent{\widetilde}{Z}_{h+1}^t \Big)(s,a)\Bigg]
\end{align*}
Using lemma~\ref{ring_negative} we get:
\begin{align*}
    \mathring{U}_h^t(s,a) - \underaccent{\widetilde}{U}_h^t(s,a)  \leq e^{\beta r_h(s,a)} \Big[ 2 b_h^t(s,a) + \Big( 1 + \frac{2}{H} \Big) \widehat{p}_h^t \Big(\mathring{Z}_h^t(s,a) - \underaccent{\widetilde}{Z}_{h+1}^t\Big)(s,a) \Big]
\end{align*}
\end{proof}
We now prove lemma~\ref{stopping_negative}:
\begin{proof}
We first prove by backward induction that, for all $h$ and $s$,
\begin{equation*}
\mathring{Z}_h^t(s)-\underaccent{\widetilde}{Z}_h^t(s)\ \le\ (\pi_h^{t+1}G_h^t)(s).
\end{equation*}
For $h=H+1$ it holds since both sides are $0$.
Assume it holds at step $h+1$. For $a=\pi_h^{t+1}(s)$
$$
\mathring{Z}_h^t(s)-\underaccent{\widetilde}{Z}_h^t(s)
=\mathring{U}_h^t(s,a)-\underaccent{\widetilde}{U}_h^t(s,a)
$$
Apply Lemma~\ref{ring_negative}: 
$$\mathring{U}_h^t(s,a)-\underaccent{\widetilde}{U}_h^t(s,a) \leq e^{\beta r_h^t(s,a)} \Big[ 3b_h^t(s,a)  + \Big(1 +  \frac{3}{H}\Big)\widehat{p}_h^t\Big(\mathring{Z}_h^t(s,a) - \underaccent{\widetilde}{Z}_{h+1}^t\Big)(s,a) \Big]$$
By the induction hypothesis:
$$
\widehat{p}_{h}^t\Big(\mathring{Z}_h^t(s,a) - \underaccent{\widetilde}{Z}_{h+1}^t\Big)(s,a) 
\leq\widehat{p}_{h}^t(\pi^{t+1}G_{h+1}^t)(s,a).
$$
Thus,
$$
\mathring{Z}_h^t(s)-\underaccent{\widetilde}{Z}_h^t(s)
\leq 3b_h^t(s,a)+\Big(1+\frac{3}{H}\Big)\widehat{p}_{h}^t(\pi^{t+1}G_{h+1}^t)(s,a)
\leq G_h^t(s,a)=(\pi_h^{t+1}G_h^t)(s)
$$
Finally, use optimism and the ring bridge:
on $\mathcal E$, $Z_h^\star\geq \underaccent{\widetilde}{Z}_h^t$ (optimism lemma~\ref{optimism_negative}) and $Z_h^{\pi^{t+1}} \leq\mathring Z_h^t$ (Lemma~\ref{ring_negative}):
$$
Z_h^{\pi^{t+1}}(s) - Z_h^\star(s)
\leq \mathring{Z}_h^t(s)-\underaccent{\widetilde}{Z}_h^t(s)
\leq (\pi_h^{t+1}G_h^t)(s)
$$
\end{proof}
\subsubsection{sample complexity}
\begin{proof}
the width certificate is:
$$G_h^t(s,a) = \min \Bigg\{ 1, e^{\beta r_h^t(s,a)} \Big[ 3b_h^t(s,a)  + \Big(1 +  \frac{3}{H}\Big) \widehat{p}_h^t \pi^{t+1} G_{h+1}^t(s)\Big]\Bigg\}$$
Let us transition to the true MDP. Using Bernstein inequality:
\begin{align*}
 \big|(\widehat{p}_h^t - p_h) \pi^{t+1} G_{h+1}^t(s) \big| \leq \sqrt{2 \operatorname{Var}_{p_h}\big( \pi^{t+1} G_{h+1}^t(s)\big) \frac{\alpha^(n_h^t(s,a)}{n_h^t(s,a)}} + \frac{2}{3} (1 - e^{\beta (H - h )}) \frac{\alpha(n_h^t(s,a))}{n_h^t(s,a)}
\end{align*}
Now, we use the inequality $\operatorname{Var}(\pi^{t+1}_{h+1} G_{h+1}^t(s)) \leq (1 - e^{\beta  (H - h)})  \pi^{t+1}_{h+1} G_{h+1}^t(s)$. Hence, using the inequality $\sqrt{xy} \leq x + y$:
\begin{align*}
 \big|(\widehat{p}_h^t - p_h) \pi^{t+1} G_{h+1}^t(s) \big| \leq \frac{1}{H} p_h \pi^{t+1} G_{h+1}^t(s) + 3 H(1 - e^{\beta (H - h +1)}) \frac{\alpha(n_h^t(s,a)}{n_h^t(s,a)}
\end{align*}
\vspace{-1mm}
And using the variance transportation lemmas ~\ref{KL transportation} and \ref{transportation} and that $\alpha^\star(n_h^t(s,a)) \leq \alpha(n_h^t(s,a))$:
\begin{align*}
\sqrt{\operatorname{Var}_{\widehat{p}^t_h}(\underaccent{\widetilde}{Z}_{h+1}^t)(s,a)\frac{\alpha^\star(n_h^t(s,a))}{n_h^t(s,a)}} &\leq 2\sqrt{\operatorname{Var}_{p_h}(Z_{h+1}^{\pi^{t+1}})\frac{\alpha^\star(n_h^t(s,a))}{n_h^t(s,a)}} \\&+ \sqrt{4 (1 - e^{\beta  (H - h)}) p_h(Z_{h+1}^{\pi^{t+1}} - \underaccent{\widetilde}{Z}^t_{h+1})(s,a)\frac{\alpha^\star(n_h^t(s,a))}{n_h^t(s,a)}} \\&+ 2(1 - e^{\beta (H - h)}) \frac{\alpha(n_h^t(s,a))}{n_h^t(s,a)} \\
& \leq 2\sqrt{\operatorname{Var}_{p_h}(Z_{h+1}^{\pi^{t+1}})\frac{\alpha^\star(n_h^t(s,a))}{n_h^t(s,a)}} + 4 H(1 - e^{\beta  (H - h)}) \frac{\alpha^\star(n_h^t(s,a))}{n_h^t(s,a)} \\ &\quad + \frac{1}{H} p_h ( Z_{h+1}^{\pi^{t+1}} - \underaccent{\widetilde}{Z}^t_{h+1})(s,a) + 2 (1 - e^{\beta  (H - h)}) \frac{\alpha(n_h^t(s,a))}{n_h^t(s,a)} \\
& \leq  2\sqrt{\operatorname{Var}_{p_h}(Z_{h+1}^{\pi^{t+1}})\frac{\alpha^\star(n_h^t(s,a))}{n_h^t(s,a)}} + 4 H(1 - e^{\beta  (H - h)}) \frac{\alpha^\star(n_h^t(s,a))}{n_h^t(s,a)} \\ &\quad + \frac{1}{H}p_h \pi^{t+1} G_{h+1}^t(s)+ 2 (1 - e^{\beta  (H - h)}) \frac{\alpha(n_h^t(s,a))}{n_h^t(s,a)} 
\end{align*}
Since we have :
\begin{align*}
    Z_{h+1}^{\pi^{t+1}} - \underaccent{\widetilde}{Z}^t_{h+1} \leq \mathring{Z}_{h+1}^{t} - \underaccent{\widetilde}{Z}^t_{h+1} \leq  G_{h+1}^t(s)
\end{align*}
Hence, using that $\alpha^\star(n_h^t(s,a)) \leq \alpha(n_h^t(s,a))$ and simplifying constants and using that $H \geq 1$:
\begin{align*}
b_h^t(s,a) \leq 6\sqrt{\operatorname{Var}_{p_h}(Z_h^{\pi^{t+1}})\frac{\alpha^\star(n_h^t(s,a))}{n_h^t(s,a)}} + \frac{3}{H} p_h \pi^{t+1} G_{h+1}^t(s) + 27 H (1 - e^{\beta  (H - h)})  \frac{\alpha(n_h^t(s,a))}{n_h^t(s,a)}
\end{align*}
We combine the two terms:
\begin{align*}
    G_h^t(s,a) &\leq e^{\beta r_h(s,a)} \Bigg[36\sqrt{\operatorname{Var}_{p_h}(Z_h^{\pi^{t+1}})\frac{\alpha^\star(n_h^t(s,a))}{n_h^t(s,a)}} + \frac{6\sqrt{2}}{H} p_h \pi^{t+1} G_{h+1}^t(s) \\&+ 81 H (1 - e^{\beta (H - h)}) \frac{\alpha(n_h^t(s,a))}{n_h^t(s,a)} + \Big(1 + \frac{3}{H}\Big) p_h \pi^{t+1} G_{h+1}^t(s) \\&+ \Big(1 + \frac{3}{H} \Big) \frac{1}{H} p_h \pi^{t+1} G_{h+1}^t(s) + \Big(1 + \frac{3}{H}\Big) 3 H (1 - e^{\beta (H - h)}) \frac{\alpha(n_h^t(s,a))}{n_h^t(s,a)} \Bigg]\end{align*}
    Hence, simplifying it gives:
    $$G_h^t(s,a) \leq e^{\beta r_h(s,a)}\Bigg[36\sqrt{\operatorname{Var}_{p_h}(Z_h^{\pi^{t+1}})\frac{\alpha(n_h^t(s,a))}{n_h^t(s,a)}} + \left(1 + \frac{13}{H} \right) p_h \pi^{t+1} G_{h+1}^t(s) + 81 H (1 - e^{\beta (H - h)}) \frac{\alpha(n_h^t(s,a))}{n_h^t(s,a)}\Bigg]$$
Unrolling this inequality like the case $\beta > 0$:
\begin{align*}
    (\pi_1 G_1^t)(s_1)
    &\leq e^{13}
    \mathbb{E}^{\pi}\Bigg[
    \sum_{h=1}^{H}
    \exp\Big(\beta\sum_{i=1}^{h} r_i(s_i,a_i)\Big)
    \Bigg(
    36\sqrt{\operatorname{Var}_{p_h}\big(Z_{h+1}^{\pi}\big)\alpha\big(n_h^t(s_h,a_h) \wedge 1\big)}
    \\
    &\quad +
    81H(1 - e^{\beta(H - h)})\alpha\big(n_h^t(s_h,a_h) \wedge 1\big)
    \Bigg)\Bigg| s_1
    \Bigg]
\end{align*}
The algorithm stops when:
\begin{align*}
    \pi_1^\tau G_1(s_1) \leq (1 - e^{\beta \varepsilon}) \underaccent{\widetilde}Z_1^{\pi^{\tau}}
\end{align*}
This is equivalent to:
$$\pi_1^\tau G_1(s_1) \leq \frac{e^{|\beta|\varepsilon} - 1}{2 e^{|\beta| \varepsilon} - 1}\left(\underaccent{\widetilde}Z_1^{\pi^{\tau}} + \pi_1 G_1^t(s_1)\right)$$
We then need to upper bound the quantity $
\frac{\pi_1^t G_1(s_1)}{\underaccent{\widetilde}Z_1^{\pi^{t}}}$ for $t =1,...,\tau - 1$:
\begin{align*}
\frac{\pi_1 G_1^t(s_1)}{\underaccent{\widetilde}{Z}_1^{\pi^t} + \pi_1 G_1^t(s_1)} &\leq \frac{\pi_1 G_1^t(s_1)}{Z_1^{\pi^{t+1}}(s_1)}\\
&\leq\frac{e^{13}}{\beta}
\mathbb{E}^{\pi^{t+1}}\Bigg[
\sum_{h=1}^{H} 
\exp\Big(\beta\sum_{i=1}^{h} r_i(s_i,a_i)\Big)
\Bigg(
36\sqrt{\frac{\operatorname{Var}_{p_h}\big(Z_{h+1}^{\pi^{t+1}}\big)}{(Z_1^{\pi^{t+1}})^2}\Big(\frac{\alpha\big(n_h^t(s_h,a_h)}{n_h^t(s,a)} \wedge 1\Big)}
\\ &\quad +
81H(1 - e^{\beta(H - h)})\Big(\frac{\alpha\big(n_h^t(s_h,a_h)}{n_h^t(s,a)} \wedge 1\Big)
\Bigg)
\Bigg| s_1
\Bigg]
\end{align*}
Like the case $\beta >0$ we write directly:
\begin{align*}
\sum_{h=1}^H \sum_{s,a}p_h^{t+1}(s,a) \exp\Big(2\beta \sum_{i=1}^{h} r_i(s_i,a_i)\Big) \frac{\operatorname{Var}_{p_h}\big(Z_{h+1}^{\pi}\big)}{(Z_1^{\pi^{t+1}})^2} = \mathbb{E}^\pi\big[ \frac{\sigma V_1^{\pi^{t+1}}}{(Z_1^{\pi^{t+1}})^2} \big]
\end{align*}
But since the greedy policy is deterministic we have: 
$$\frac{\sigma V_1^{\pi}(s_1)}{(Z_1^{\pi})^2} = \frac{\operatorname{Var}(e^{\beta R_1^\pi}|S_1 =s_1)}{\mathbb{E}(e^{\beta R_1^\pi}|S_1 = s_1)^2}$$
Where $L = e^{\beta \sum_{i=1}^H r_i(s_i,a_i)}$, using lemma~\ref{bound} we get that:
$$\frac{\sigma V_1^{\pi}(s_1)}{(Z_1^{\pi})^2} \leq \frac{(e^{|\beta| G_{\text{max}}(\mathcal{M})} - 1)^2}{e^{|\beta| G_{\text{max}}(\mathcal{M})}} $$
Hence:
\begin{align*}
\sum_{h=1}^H \sum_{s,a} &p_h^{t+1}(s,a) \exp\Big(\beta \sum_{i=1}^{h} r_i(s_i,a_i)\Big)
\Bigg(
36\sqrt{\frac{\operatorname{Var}_{p_h}\big(Z_{h+1}^{\pi^{t+1}}\big)}{(Z_1^{\pi^{t+1}})^2}\big(\frac{\alpha\big(n_h^t(s,a)}{n_h^t(s,a)} \wedge 1 \big)}\Big)\\ 
&\leq  36\sqrt{\sum_{h=1}^H \sum_{s,a}p_h^{t+1}(s,a) \exp\Big(2\beta \sum_{i=1}^{h} r_i(s_i,a_i)\Big) \frac{\operatorname{Var}_{p_h}\big(Z_{h+1}^{\pi^{t+1}}\big)}{(Z_1^{\pi^{t+1}})^2}} \sqrt{\sum_{h=1}^H \sum_{s,a}p_h^{t+1}(s,a) \frac{\alpha(n_h^t(s,a))}{n_h^t(s,a)}}\\
&\leq 36\sqrt{ \frac{(e^{|\beta| G_{\text{max}}(\mathcal{M})} - 1)^2}{e^{|\beta| G_{\text{max}}(\mathcal{M})}}}  \sqrt{\sum_{h=1}^H \sum_{s,a}p_h^{t+1}(s,a) \frac{\alpha^\star(n_h^t(s,a))}{n_h^t(s,a)}}
\end{align*}
For the second term:
\begin{align*}
\frac{1}{Z_1^{\pi^{t+1}}(s_1)} \mathbb{E}^{\pi^{t+1}} & \left[
\sum_{h=1}^{H}
\exp\left(\beta\sum_{i=1}^{h} r_i(s_i,a_i)\right)
H(1 - e^{\beta(H - h)})\left(\frac{\alpha\big(n_h^t(s,a)\big)}{n_h^t(s,a)} \wedge 1 \right)
\middle| s_1
\right] \\
&= \frac{1}{Z_1^{\pi^{t+1}}} \sum_{h=1}^H \sum_{s,a} p_h^{t+1}(s,a)\exp\left(\beta\sum_{i=1}^{h} r_i(s_i,a_i)\right)
H(1 - e^{\beta(H - h)})\left(\frac{\alpha\big(n_h^t(s,a)\big)}{n_h^t(s,a)} \wedge 1 \right) \\
&\leq \sum_{h=1}^H \sum_{s,a} p_h^{t+1}(s,a)e^{\beta h}
H(1 - e^{\beta(H - h)})\left(\frac{\alpha\big(n_h^t(s,a)\big)}{n_h^t(s,a)} \wedge 1 \right) \\
&\leq He^{|\beta| H} \sum_{h=1}^H \sum_{s,a} p_h^{t+1}(s,a)\left(\frac{\alpha\big(n_h^t(s,a)\big)}{n_h^t(s,a)} \wedge 1 \right)
\end{align*}
We then sum on $t < \tau$ and use that by sub-optimality we have for $t = 1,...,\tau-1$:
\begin{align*}
    \frac{\pi_1 G_1^t(s_1)}{Z_1^{\pi^{t+1}}} \geq \frac{\pi_1^\tau G_1(s_1)}{\left(\underaccent{\widetilde}Z_1^{\pi^{\tau}} + \pi_1 G_1^t(s_1)\right)} \geq e^{|\beta| \varepsilon} - 1 
\end{align*}
Hence:
    \begin{align*}
    \tau ( e^{|\beta| \varepsilon} - 1) &\leq 36 e^{13} e^{2 |\beta| \varepsilon} \sum_{t=1}^{\tau - 1} \sqrt{(\frac{(e^{|\beta| G_{\text{max}}(\mathcal{M})} - 1)^2}{e^{|\beta| G_{\text{max}}(\mathcal{M})}} )\sum_{h=1}^H \sum_{s,a} p_h^{t+1}(s,a)\big(\frac{\alpha^\star\big(n_h^t(s,a)}{n_h^t(s,a)} \wedge 1 \big)}\\& + 81e^{13} e^{2 |\beta| \varepsilon} H  e^{|\beta| H}\sum_{t=1}^{\tau -1}\sum_{h=1}^H \sum_{s,a} p_h^{t+1}(s,a)\big(\frac{\alpha\big(n_h^t(s,a)}{n_h^t(s,a)} \wedge 1 \big)\\
    &\leq 36 e^{13}e^{2 |\beta| \varepsilon} \sqrt{(\frac{(e^{|\beta| G_{\text{max}}(\mathcal{M})} - 1)^2}{e^{|\beta| G_{\text{max}}(\mathcal{M})}} )}\sqrt{T}\sqrt{\sum_{t=1}^{\tau - 1}\sum_{h=1}^H \sum_{s,a} p_h^{t+1}(s,a)\big(\frac{\alpha^\star\big(n_h^t(s,a)}{n_h^t(s,a)} \wedge 1 \big)} \\&+ 81e^{13}e^{2 |\beta| \varepsilon} He^{|\beta| H}\sum_{t=1}^{\tau -1}\sum_{h=1}^H \sum_{s,a} p_h^{t+1}(s,a)\big(\frac{\alpha\big(n_h^t(s,a)}{n_h^t(s,a)} \wedge 1 \big)
\end{align*}
Similarly to the case $\beta >0$ we bound the other terms using the counting argument which yield:
\begin{align*}
    \tau ( e^{|\beta| \varepsilon} - 1)  \leq 36 e^{13} e^{2 |\beta| \varepsilon} \sqrt{(\frac{(e^{|\beta| G_{\text{max}}(\mathcal{M})} - 1)^2}{e^{|\beta| G_{\text{max}}(\mathcal{M})}} )\tau SAH\alpha^\star(\tau-1,\delta) \log(\tau + 1)} + 81e^{13} e^{2 |\beta| \varepsilon} e^{|\beta| H}H^2 S A\alpha(\tau-1,\delta) \log(\tau + 1)
\end{align*}
We replace $\alpha^\star$ and $\alpha$ by their expressions and using that $\log(\tau + 1) \leq \log(8 e \tau)$ since $\tau \geq 1$:
\begin{align*}
    \tau ( e^{|\beta| \varepsilon} - 1)  &\leq 36 e^{13} e^{2 |\beta| \varepsilon}  \sqrt{(\frac{(e^{|\beta| G_{\text{max}}(\mathcal{M})} - 1)^2}{e^{|\beta| G_{\text{max}}(\mathcal{M})}} )\tau SAH\Big( \log\big(\frac{3 SAH}{\delta}\big) \log\big( 8 e \tau \big) + \log \big( 8 e \tau \big)^2\Big) } \\&+ 81e^{13} e^{2 |\beta| \varepsilon} e^{|\beta| H}H^2 S A\Big(\log\big(\frac{3 SAH}{\delta}\big) \log\big( 8 e \tau \big) + S\log \big( 8 e \tau \big)^2\Big)
\end{align*}
Finally, we use lemma~\ref{inequality} with :
\begin{align*}
    &C= 36 e^{13}  \frac{\sqrt{(\frac{(e^{|\beta| G_{\text{max}}(\mathcal{M})} - 1)^2}{e^{|\beta| G_{\text{max}}(\mathcal{M})}} ) SAH}}{e^{|\beta| \varepsilon} - 1} \quad, A = \log(\frac{3SAH}{\delta}) \quad, B = 1\\&\quad \quad \quad \quad  D = \frac{81e^{13} e^{\beta H} H^2 SA}{ e^{|\beta| \varepsilon} - 1} \quad \text{and} \quad E = S\end{align*}
    Which yield:
\begin{align*}
    \tau &\leq \frac{(e^{|\beta| G_{\text{max}}(\mathcal{M})} - 1)^2}{e^{|\beta| G_{\text{max}}(\mathcal{M})}} \frac{e^{2 |\beta| \varepsilon}}{\big( e^{|\beta| \varepsilon} - 1\big)^2} SAH \bigg(\log(\frac{3SAH}{\delta}) + 1 \bigg) C_1^2 + 3 \frac{e^{2 |\beta| \varepsilon}}{ e^{|\beta| \varepsilon} - 1}e^{\beta H} H^2 SA\bigg(\log(\frac{3SAH}{\delta}) + S\bigg)C_1^2 + 1
\end{align*}
Where $C_1 = \frac{8}{5} \log\bigg(4 e^{17} \frac{(S+1)(H+1) e^{|\beta H} S A H^2}{( e^{|\beta| \varepsilon} - 1)}\bigg) $

In particular, assuming that $\varepsilon$ is small enough that the first term dominates the second term then:
\begin{align*}
    \tau \leq \frac{(e^{|\beta| G_{\text{max}}(\mathcal{M})} - 1)^2}{e^{|\beta| G_{\text{max}}(\mathcal{M})}} \frac{e^{2 |\beta| \varepsilon}}{\big( e^{|\beta| \varepsilon} - 1\big)^2}SAH\log(\frac{3SAH}{\delta}) C_2^2
\end{align*}
Where $C_2 = 3 C_1$. We can finally hide the constants and the log terms to get:
$$\tau = \tilde{\mathcal{O}}\left( \frac{(e^{|\beta| G_{\text{max}}(\mathcal{M})} - 1)^2}{e^{|\beta| G_{\text{max}}(\mathcal{M})}} \frac{e^{2 |\beta| \varepsilon}}{(e^{|\beta| \varepsilon} -1)^2}SAH \right)$$
Finally to see that the algorithm is $(\varepsilon,\delta)$ PAC, At time $\tau$:
$$\pi_1^\tau G_1(s_1) \leq (1 - e^{\beta \varepsilon}) \underaccent{\widetilde}Z_1^{\pi^{\tau}}$$
Since $Z_1^{\pi^{\tau}} \geq \underaccent{\widetilde}Z_1^{\pi^{\tau}}$, this is a stronger stopping condition than:
$$\pi_1^\tau G_1(s_1) \leq (1 - e^{\beta \varepsilon}) Z_1^{\pi^{\tau}}$$
Now, we write:
\begin{align*}
    \big(V^\star_1 - V^{\pi^{\tau}}_1\big)(s_1) = \frac{1}{\beta} \log\big( \frac{Z^\star_1}{Z_1^{\pi^{\tau}}}\Big)(s_1) = \frac{1}{\beta} \log \Big(1 + \frac{Z_1^\star - Z_1^{\pi^{\tau}}}{Z_1^{\pi^{\tau}}} \Big)(s_1) \leq \frac{1}{\beta} \log \Big(1 + \frac{\pi_1^\tau G_1(s)}{Z_1^{\pi^{\tau}}} \Big)(s_1)\leq \varepsilon
\end{align*}
\end{proof}
\section{Lower bound}\label{lower_bound}
We first state a change of measure for bandit models result from \citep{kauffman2016}:
\begin{lemma}\label{measure_change}
Let $N_a(t) = \sum_{s=1}^t \mathbbm{1}_{\{A_s=a\}}$ be the number of draws of arm $a$ between the instants $1$ and $t$ and $N_a = N_a(\tau)$ be the total number of draws of arm $a$ by some algorithm $\mathcal{A} = ((A_t), \tau, \hat{S}_m)$.
 
Let $\nu$ and $\nu'$ be two bandit models with $K$ arms such that for all $a$, the distributions $\nu_a$ and $\nu'_a$ are mutually absolutely continuous. For any almost-surely finite stopping time $\sigma$ with respect to $(\mathcal{F}_t)$,
$$
\sum_{a=1}^K \mathbb{E}_\nu[N_a(\sigma)]\mathrm{KL}(\nu_a, \nu'_a) \ge \sup_{\mathcal{E} \in \mathcal{F}_\sigma} d(\mathbb{P}_\nu(\mathcal{E}), \mathbb{P}_{\nu'}(\mathcal{E}))
$$
where $d(x, y) = x \log(x/y) + (1 - x) \log((1 - x)/(1 - y))$ is the binary relative entropy, with the convention that $d(0, 0) = d(1, 1) = 0$.
\end{lemma}
We make the following assumption:

\noindent \textbf{Assumption 1} Let $d = \lceil \log_{A}((S - 3)(A - 1) + 1) \rceil$. Assume that $H \geq 3d$

This assumption means that the horizon is long enough with respect to the size of MDP so that the agent can reach the reward state. We state the proof under the stronger assumption that $d = \log_{A}((S - 3)(A - 1) + 1)$ for simplicity. But as discussed in \citep{domingues2021episodic}, we can extend the construction to the general case by not having a full $A$-ary tree.

\noindent \textbf{Condition A} The bound is stated in the small $\varepsilon$ regime. Let $c = e^{|\beta| H} - 1$, Assume that we have :
\begin{align}\label{condition A}
e^{|\beta|\varepsilon}< \begin{cases}
\min\left\{\frac{4(c+1)^2}{2c^2+7c+4},\frac{16}{13}\right\} & \beta>0\\[10pt]
\min\left\{\frac{11c+8}{10c+8},\frac{11}{10}\right\} & \beta<0~
\end{cases}
\end{align}
When $|\beta|H \geq \ln(2)$, the condition is reduced to the constants and gets harder to satisfy as $\beta$ goes to $0$. Condition ~\ref{condition A} is used to have a valid construction in the proof below i.e to ensure the transition probabilities are in $[0,1]$
\begingroup
\renewcommand\thetheorem{\ref{th:lower_bound_theorem}}
\begin{theorem}
Fix $S \geq 6, A \geq 2, H \in \mathbb{N}, \beta \in \mathbb{R}^\star$, and $\varepsilon, \delta \in (0, 1)$ such that $\delta \leq 1/16$ and $\varepsilon$ verify the condition ~\eqref{condition A}. Then there exists an MDP $\mathcal{M}_0$ with $S$ states, $A$ actions, horizon $H$, and rewards in $[0, 1]$ such that for every algorithm $\mathcal{A}$ output a policy $\hat{\pi}$ that is $(\varepsilon, \delta)$-PAC for the entropic risk measure after sampling $\tau$ trajectories we have: 
$$\mathbb{E}_{\mathcal{M}_0}[\tau] \geq \frac{1}{1650} \frac{(e^{|\beta| G_{\max}(\mathcal{M}_0)} - 1)^2}{e^{|\beta| G_{\max}(\mathcal{M}_0)}} \frac{e^{2\min \{\beta,0\}\varepsilon}SAH}{(e^{|\beta| \varepsilon} - 1)^2} \log \left( \frac{1}{\delta} \right)$$
\end{theorem}
\endgroup 
\begin{proof}
Consider the following MDP defined in \citep{domingues2021episodic}(with different transitions). We have three special states $s_w$(waiting state), $s_g$(the good absorbing state) and $s_b$(the bad absorbing state). The rest $S-3$ are arranged in the form of a full A-tree of depth $d - 1$ denoted $\mathcal{L}$ whose root is $s_{\text{root}}$, this means that the number of states is 
$$3 +\sum_{i=1}^{d-1} A^i = 3 + \frac{A^d - 1}{A  -1} = S$$
The action set is $\mathcal{A} = \{1,...,A\}$ and let $a_w \in A$ be a fixed action, denote $L = A^{d-1}$. Let $\overline{H} \leq H-d$ be an integer to be chosen later.

The episode starts at $s_w$, for steps $h = 1,...,\overline{H}$ we have the transition kernel:
$$p_h(s_w|a,s_w) =  \mathbbm{1}_{\{a = a_w, h \leq \overline{H}\}} \quad \quad \text{and} \quad \quad p_h(s_{\text{root}}|s_w,a) = 1 - p_h(s_w|s_w,a)$$
This means that for the first $\overline{H}$ steps, you can either chose the action $a_w$ to stay in the waiting state $s_w$, or pick any other action and enter the tree at the next step, and at step $h = \overline{H}$ we exit regardless of the chosen action.

Once you enter the tree, the transition is deterministic. From any internal node $x$, the action $a$ deterministically goes to the $a$-th child of $x$. Thus, after exactly $\overline{H} + d$ the policy reaches a leaf of the tree $\mathcal{L}$. 

Define the index set of ``triples''
$$
\mathcal{U} = [\overline{H}] \times L \times [A], \quad |\mathcal{U}| = \overline{H}  L  A
$$

A triple $u = (h, l, a)$ corresponds to:
\begin{itemize}
    \item exiting at step $h$ (so leaf step is $h + d$),
    \item reaching leaf $l$,
    \item choosing leaf action $a$
\end{itemize}

Fix two numbers $0 < p_- < p_+ < 1$.The baseline instance $\mathcal{M}_0$: for every triple $u = (t, l, a)$,
$$
\mathbb{P}(s_g \mid t, l, a) = p_-
$$
For $\mathcal{M}_0$, all leafs are the same and takes to the good state with probability $p_-$ regardless of the chosen action.
The Special instance $\mathcal{M}_u$ for each $u =(h^\star,l^\star,a^\star) \in \mathcal{U}$: identical to $\mathcal{M}_0$ except that at the single triple $u = $,
$$
\mathbb{P}(s_g \mid u) = p_+
$$
and for all other triples $u' \neq u, \mathbb{P}(s_g \mid u') = p_-$. This means that there exists one unique optimal leaf $l^\star$ where the agent can chose one unique optimal action $a^\star$ when exiting precisely at step $h^\star$ 

Thus, the family is $\{\mathcal{M}_0\} \cup \{\mathcal{M}_u : u \in \mathcal{U}\}$, and any two instances differ in exactly one Bernoulli parameter at one triple.

Let
$$
\tilde{H} = \overline{H} + d + 1, \quad H' = H - \overline{H} - d
$$

Define rewards by
$$
r_h(s, a) = \mathbf{1}\{s = s_g\}  \mathbf{1}\{h \geq \tilde{H}\}
$$

So rewards are in $[0, 1]$, and rewards only accumulate from $\tilde{H}$ onward. We give an illustration in figure ~\eqref{fig:hard-mdp} taken from \citep{domingues2021episodic} and adapted to have our transisitions that are a bit different from the original construction:
\begin{figure}[t]
\centering
\begin{tikzpicture}[
  x=1cm,y=1cm,
  >={Stealth[length=2.2mm]},
  every node/.style={font=\small},
  st/.style={circle,draw,inner sep=0pt,minimum size=10mm},
  leaf/.style={circle,draw,fill=yellow!80,inner sep=0pt,minimum size=10mm},
  blk/.style={circle,draw=black,fill=black,inner sep=0pt,minimum size=12mm},
  dashedarrow/.style={->,dash pattern=on 2.2pt off 2.2pt},
  bluearrow/.style={->,blue,thick}
]

\node[st]  (sw)    at (0,4.20) {$s_w$};
\node[st]  (sroot) at (0,2.90) {$s_{\mathrm{root}}$};

\node[blk] (uL) at (-3.90,1.55) {};
\node[blk] (uR) at ( 3.90,1.55) {};

\node[leaf] (s1) at (-5.60,0.25) {$s_1$};
\node[leaf] (s2) at (-2.10,0.25) {$s_2$};
\node[leaf] (s3) at ( 2.10,0.25) {$s_3$};
\node[leaf] (s4) at ( 5.60,0.25) {$s_4$};

\node[st] (sb) at (-1.10,-1.80) {$s_b$};
\node[st] (sg) at ( 1.10,-1.80) {$s_g$};

\draw[->] (sw) edge[loop above] node[above] {action $=a_w$} (sw);
\draw[->] (sw) -- node[right] {action $\neq a_w$} (sroot);

\draw[->] (sroot) -- (uL);
\draw[->] (sroot) -- (uR);

\draw[->] (uL) -- (s1);
\draw[->] (uL) -- (s2);
\draw[->] (uR) -- (s3);
\draw[->] (uR) -- (s4);

\draw[dashedarrow] (s1) -- node[pos=0.30,below left] {$1 -p_-$} (sb);
\draw[dashedarrow] (s1) -- node[pos=0.30,above left] {$p_-$} (sg);

\draw[dashedarrow] (s2) -- (sb);
\draw[dashedarrow] (s2) -- (sg);

\draw[dashedarrow] (s3) -- (sb);
\draw[dashedarrow] (s3) -- (sg);
\draw[dashedarrow] (s4) -- (sb);
\draw[dashedarrow] (s4) -- (sg);

\draw[bluearrow]
  (s2) to[bend right=25] node[pos=0.25,blue,below left] {$1-p_+$} (sb);

\draw[bluearrow]
  (s2) to[bend left=18]  node[pos=0.30,blue,above] {$p_+$} (sg);

\draw[->] (sb) edge[loop below] node[below] {$1$} (sb);
\draw[->] (sg) edge[loop below] node[below] {$1$} (sg);

\node[anchor=east] at ($(sb)+(-1.65,0)$) {$r_h(s_b,a)=0$};
\node[anchor=west] at ($(sg)+( 1.65,0)$) {$r_h(s_g,a)=\mathbbm{1}\{h\ge \bar{H}+d+1\}$};

\end{tikzpicture}
\caption{An example of the MDP construction. Here the state $s_2$ is optimal and have a better probability $p_+$ of landing in the good state $s_g$, $p_+$ and $p_-$ are defined bellow in the proof. Figure reproduced from \citep{domingues2021episodic}}
\label{fig:hard-mdp}
\end{figure}
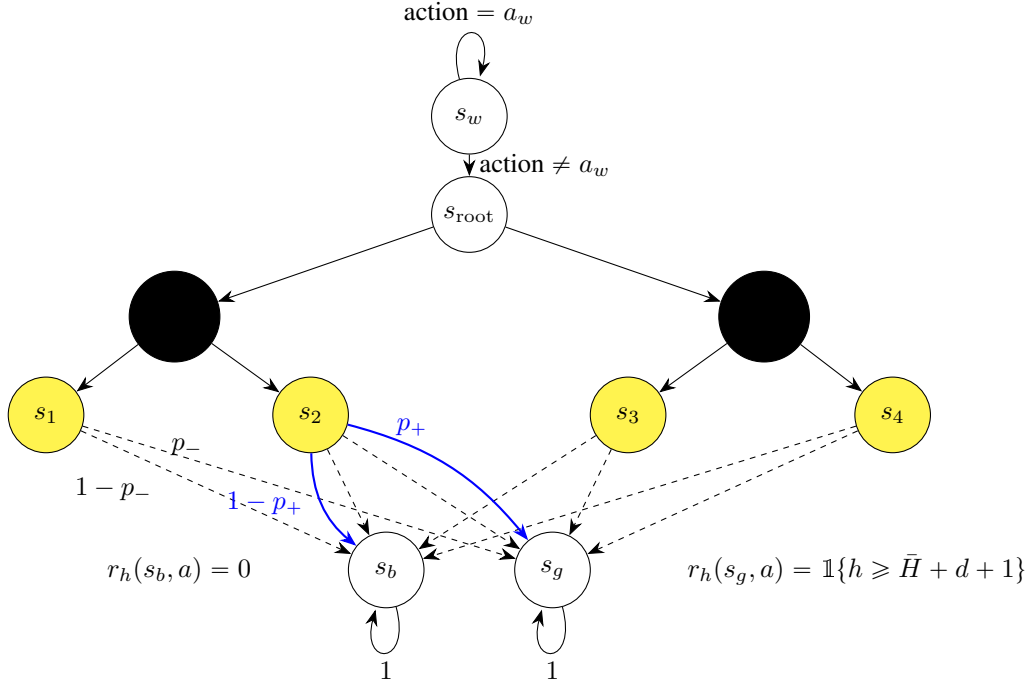

By construction, by time $\tilde{H}$ the chain has already entered $s_g$ or $s_b$ and is absorbing. Therefore the return is
\begin{equation*}
G = \sum_{h=\tilde{H}}^{H} \mathbf{1}\{S_h = s_g\} = H'  \mathbf{1}\{S_{\tilde{H}} = s_g\} \in \{0, H'\}
\end{equation*}

For any policy $\pi$ and instance $\mathcal{M}$, let
$$
p^\mathcal{M}(\pi) = \mathbb{P}_{\mathcal{M}, \pi}(S_{\tilde{H}} = s_g)
$$
The probability of being at the good state by the time $\tilde{H}$. By construction, by the time $\tilde H$ we are either in $s_g$ or $s_h$, Then:
$$
\mathbb{E}[e^{\beta G}] = (1 - p^M(\pi))  1 + p^\mathcal{M}(\pi)  e^{\beta H'} = 1 + c p^\mathcal{M}(\pi), \quad c = e^{\beta H'} - 1
$$
Hence
\begin{equation}
V^\mathcal{M}(\pi) = \frac{1}{\beta} \log \left( 1 + c p^\mathcal{M}(\pi) \right)
\end{equation}
For a policy $\pi$, define the probability of sucess: 
$$\nu_{\pi}(u) = \mathbb{P}_{\pi}(\text{the episode's exit/leaf/action triple equals } u), \quad u \in \mathcal{U}$$
Because each episode produces exactly one triple, we have
$$\sum_{u \in \mathcal{U}} \nu_{\pi}(u) = 1$$
In instance $\mathcal{M}_u$, since only when the realized triple equals the special one do we get probability $p_+$; otherwise $p_-$. The success probability is
$$p^{\mathcal{M}_u}(\pi) = p_- + (p_+ - p_-)\nu_{\pi}(u)$$
Let $\Delta = p_+ - p_- > 0$. The optimal policy in $\mathcal{M}_u$ can choose $t$, route to $l$, and pick action $a$ deterministically so that $\nu_{\pi}(u) = 1$. Hence the optimal success probability is $p_+$, and
$$V^{\mathcal{M}_u, \star} = \frac{1}{\beta} \log(1 + cp_+)$$
Let us now chose $p_+$ and $p_-$ so any $\varepsilon$-optimal policy must satisfy $\nu_\pi(u) > \frac{1}{2}$ in $\mathcal{M}_u$. 
The entropic criterion overweights rare high-return trajectories, so we make success(reaching the good state) rare by choosing $p_- \sim  e^{-\beta H}$, more precisely:
$$p_- = \frac{1}{2(c + 1)} \sim e^{-\beta H'}$$
We have:
$$V_{\beta}^{\mathcal{M}_u}(\pi) = \frac{1}{\beta} \log (1 + c(p_- + \Delta \nu_\pi(u)))$$
We chose $\Delta$ so that whenever we have a probability of success $\nu_\pi(u)$ smaller than $\frac{1}{2}$ the the policy $\pi$ is $\varepsilon$-suboptimal. When $\nu_{\pi}(u) \leq 1/2$, then $p^{\mathcal{M}_u}(\pi) \leq p_- + \Delta/2$. So it suffices to enforce
$$\frac{1}{\beta} \log \frac{1 + c(p_- + \Delta)}{1 + c(p_- + \Delta/2)} = \varepsilon \iff \frac{1 + c(p_- + \Delta)}{1 + c(p_- + \Delta/2)} = e^{\beta \varepsilon}$$
Hence, we must have:
$$\Delta = \frac{(3c + 2)(e^{\beta \varepsilon} - 1)}{c(c + 1)(2 - e^{\beta \varepsilon})} $$
Since we have that $e^{\beta \varepsilon} < \frac{4(c+1)^2}{2c^2 + 7c + 4} <2 $ by condition ~\ref{condition A}, the construction is admissible in the sense that $0 <p_- <p_+ <1$.

By this construction, we have for every $u \in \mathcal{U}$ and every policy $\pi$,
\begin{equation}\label{hardness}
V^{\mathcal{M}_u}(\pi) \geq V^{\mathcal{M}_u, \star} - \varepsilon \implies \nu_{\pi}(u) > \frac{1}{2}\end{equation}
Indeed, if $\nu_{\pi}(u) \leq 1/2$, then $p^{\mathcal{M}_u}(\pi) \leq p_- + \Delta/2$, while the optimal policy achieves $p_+ = p_- + \Delta$. Hence
$$V^{\mathcal{M}_u,\star} - V^{\mathcal{M}_u}(\pi) \geq \frac{1}{\beta} \log \frac{1 + c(p_- + \Delta)}{1 + c(p_- + \Delta/2)} = \varepsilon$$
And the contraposition yields the claim~\eqref{hardness}.
Let the algorithm output $\hat{\pi}$ at stopping time $\tau$. Define the event
$$\mathcal{E}_u = \{\nu_{\hat{\pi}}(u) > 1/2\}$$
By the previous remark and $(\varepsilon, \delta)$-PAC correctness,
$$\mathbb{P}_{u}(\mathcal{E}_u) \geq 1 - \delta \quad \forall u \in \mathcal{U}$$
Also, since $\sum_u \nu_{\hat{\pi}}(u) = 1$, at most one $u$ can satisfy $\nu_{\hat{\pi}}(u) > 1/2$. Hence the events $\{\mathcal{E}_u\}_{u \in \mathcal{U}}$ are mutually exclusive and in particular under $\mathcal{M}_0$,
$$\sum_{u \in \mathcal{U}} \mathbb{P}_{0}(\mathcal{E}_u) \leq 1$$
Let $N_u(\tau)$ be the number of episodes $k \leq \tau$ in which the algorithm's realized triple equals $u$. Then
$$\sum_{u \in \mathcal{U}} N_u(\tau) = \tau \quad \text{a.s.}$$
Fix $u \in \mathcal{U}$. We apply lemma~\ref{measure_change} for the event $\mathcal{E}_u$, the two instances $\mathcal{M}_u$ and $\mathcal{M}_0$ differ only in the Bernoulli transition at triplet $u$ hence:
\begin{align*}
\mathbb{E}_0[N_u(\tau)] d(p_-,p_+) \geq d(\mathbb{P}_0(\mathcal{E}_u),\mathbb{P}_u(\mathcal{E}_u)) &\geq (1 - \mathbb{P}_0(\mathcal{E}_u) )\log\left(\frac{1}{1 - \mathbb{P}_u(\mathcal{E}_u))}\right) - \log(2) \\&\geq (1 - \mathbb{P}_0(\mathcal{E}_u) )\log\left(\frac{1}{\delta}\right) - \log(2)
\end{align*}
Where we used the $(\varepsilon,\delta)$-PAC in the last inequality. We sum over $u \in \mathcal{U}$:
$$\mathbb{E}_0[\tau] = \sum_{u \in \mathcal{U}} \mathbb{E}_0[N_u(\tau)] \geq \frac{1}{d(p_-,p_+)}\left(\sum_{u \in \mathcal{U}}  (1 - \mathbb{P}_0(\mathcal{E}_u) )\log\left(\frac{1}{\delta}\right) - \log(2) \right)\geq \frac{1}{2d(p_-,p_+)}|\mathcal{U|} \log\left(\frac{1}{\delta}\right)$$
Where we used in the final inequality that $\sum_{u \in \mathcal{U}} \mathbb{P}_{0}(\mathcal{E}_u) \leq 1$ and the condition $\delta \leq \frac{1}{16}$ and since:
$$d(p_-, p_+) \leq \frac{(p_+ - p_-)^2}{p_-(1 - p_-)} \leq (c+1)  \frac{(e^{\beta \varepsilon} - 1)^2 \left(\frac{3c+2}{2(c+1)}\right)^2}{c^2(1 - \frac{1}{2}e^{\beta \varepsilon})^2} = \frac{(c+1)}{c^2}  \frac{(e^{\beta \varepsilon} - 1)^2}{(1 - \frac{1}{2}e^{\beta \varepsilon})^2}  \left(\frac{3c+2}{2(c+1)}\right)^2$$
Using the condition $e^{\beta \varepsilon} \leq \frac{16}{13}$ we get:
$$d(p_-,p_+) \leq \frac{1521}{100} \frac{(c+1)}{c^2} (e^{\beta \varepsilon} - 1)^2$$
Hence:
$$\mathbb{E}_0[\tau] \geq \frac{50}{1521} \frac{c^2}{c+1} \frac{\big( H - \overline{H}-d\big) L A}{(e^{\beta \varepsilon} - 1)^2} \log\left(\frac{1}{\delta}\right)$$
Since the number of leaves is given by $L = (1 - 1/A)(S - 3) + 1/A \geq S/4$ (for $A \geq 2, S \geq 6$), and taking $\overline{H} = H/3$ with $d \leq H/3$, we obtain the sample complexity lower bound:
$$\mathbb{E}_0[\tau] \geq \frac{1}{1650} \frac{\left(e^{\beta (H - \overline{H} -1)}- 1\right)^2}{e^{\beta (H - \overline{H} -1)}} \frac{SAH}{(e^{\beta \varepsilon} - 1)^2} \log\left(\frac{1}{\delta}\right)$$
Since in the constructed MDP, we only accumulate rewards after $\tilde H$ and the optimal policy accumulate rewards across all subsequent steps, we have: $G_{\max}(\mathcal{M}_0) = H - \tilde{H} + 1 = H - \overline{H} -d $ Hence:
$$\mathbb{E}_0[\tau] \geq \frac{1}{1650} \frac{\left(e^{\beta G_{\max}(\mathcal{M}_0)}- 1\right)^2}{e^{\beta G_{\max}(\mathcal{M}_0)}} \frac{SAH}{(e^{\beta \varepsilon} - 1)^2} \log\left(\frac{1}{\delta}\right)$$
For $\beta>0$ (risk-seeking entropic criterion), we choose $p_-$ so that transitioning to the good absorbing state is a rare event. This makes upside tail events hard
to detect and estimate, and since the entropic objective overweights favorable rare outcomes, this
again leads to larger sample complexity.\begin{align*}
\mathbb{E}[e^{\beta G}] &= (1 - p^\mathcal{M}(\pi))  1 + p^\mathcal{M}(\pi)  e^{\beta H'} = e^{\beta H'} \left( (1 - p^\mathcal{M}(\pi)) e^{|\beta | H'} + p^\mathcal{M}(\pi) \right) \\&=  e^{\beta H'} \left( (1 - p^\mathcal{M}(\pi)) (e^{|\beta | H'} -1)  + 1 \right) = e^{\beta H'} \left( (1 - p^\mathcal{M}(\pi))c  + 1 \right)
\end{align*}
And we have the entropic risk measure:
$$V^{\mathcal{M}}(\pi) = -\frac{1}{|\beta|} \log\left(e^{\beta H'} \left( (1 - p^\mathcal{M}(\pi))c  + 1 \right)\right) = H' - \frac{1}{|\beta|} \log\left((1 - p^\mathcal{M}(\pi))c  + 1 \right)$$
Where $c = e^{|\beta | H'} -1$.

For $\beta<0$, the entropic criterion is especially sensitive to adverse tail events, so we instead make failure rare by choosing $p_- \sim 1 - e^{-|\beta|H}$, more precisely:
$$p_- = 1 - \frac{1}{2(c+1)} \sim 1 - e^{-\beta H'}$$
We derive $p_+$ similarly to $\beta > 0$ and we find if $p_+ = p_- + \Delta$ then:
$$\Delta = \frac{(3c + 2)(e^{|\beta| \varepsilon} - 1)}{c(c + 1)(e^{|\beta| \varepsilon} -\frac{1}{2})}$$
Again, since we have $e^{|\beta|\varepsilon} < \frac{11c + 8}{10c + 8}$, this construction is admissible in the sense that $0 <p_- <p_+ <1$.

And we prove in the same way that:
 for every $u \in \mathcal{U}$ and every policy $\pi$,
$$V^{\mathcal{M}_u}(\pi) \geq V^{M_u, \star} - \varepsilon \implies \nu_{\pi}(u) > \frac{1}{2}$$
The rest of the argument (change of measure and summing over $\mathcal{U}$ goes the same way )and we finally upper bound the kl divergence:
$$d(p_-, p_+) \leq (c+1)  \frac{(e^{|\beta| \varepsilon} - 1)^2 \left(\frac{3c+2}{2(c+1)}\right)^2}{c^2(e^{|\beta| \varepsilon} -\frac{1}{2})^2} \leq \frac{9(c+1)}{c^2}  \frac{(e^{|\beta| \varepsilon} - 1)^2}{e^{2\beta \varepsilon}}  \left(\frac{3c+2}{2(c+1)}\right)^2$$
Hence by having a looser constant to match the $\beta>0$ lower bound:
$$\mathbb{E}_0[\tau] \geq \frac{1}{1650} \frac{\left(e^{|\beta| G_{\max}(\mathcal{M}_0)}- 1\right)^2}{e^{|\beta| G_{\max}(\mathcal{M}_0)}} \frac{e^{2\beta \varepsilon}SAH}{(e^{|\beta| \varepsilon} - 1)^2} \log\left(\frac{1}{\delta}\right)$$
\end{proof}
\section{Regret bounds}\label{app:regret}
First we prove theorem~\ref{th:regretupper}. We only give the proof for $\beta > 0$ but the proof for $\beta < 0$ is similar to the sample complexity analysis
\begin{proof}
Consider algorithm 1 run indefinitely with $\pi^{t+1}$ the algorithm computed from $n^t$ and $\widehat p^t$.Define the accumulated regret as:
\begin{equation*}
    R(K) = \sum_{t=1}^{K-1} V_1^\star(s_1) - V_1^{\pi^{t+1}}(s_1) 
\end{equation*}
We will work on the good event $\mathcal{E}$. We have:
\begin{align*}
    R(K) &= \sum_{t=1}^{K-1} V_1^\star(s_1) - V_1^{\pi_{t+1}}(s_1)  \\
    &= \sum_{t=1}^{K-1} \frac{1}{\beta} \log\left( \frac{Z^\star_1}{Z_1^{\pi^{t+1}}}\right)(s_1) \\
    &= \sum_{t=1}^{K-1} \frac{1}{\beta} \log \left(1 + \frac{Z_1^\star - Z_1^{\pi^{t+1}}}{Z_1^{\pi^{t+1}}} \right)(s_1) \\
    &\leq \sum_{t=1}^{K-1} \frac{1}{\beta} \log\left(1 + \frac{\pi^{t+1}_1 G_1^t(s_1)}{Z_1^{\pi^{t+1}}} \right)(s_1) \\
    &\leq \sum_{t=1}^{K-1} \frac{1}{\beta} \frac{\pi^{t+1}_1 G_1^t(s_1)}{Z_1^{\pi^{t+1}}}(s_1) 
\end{align*}
The inequality follows since it's at each step we have 
$$Z_1^\star - Z_1^{\pi^{t+1}}(s_1) \leq \pi_1^{t+1} G_1^t(s_1)$$
By following the same analysis as the paper we find that for any episode $t$ we have:
\begin{align*}
(\pi_1 G_1^t)(s_1)
&\leq 36 e^{13} \sqrt{\frac{(e^{\beta G_{\text{max}}(\mathcal{M})} - 1)^2}{e^{\beta G_{\text{max}}(\mathcal{M})}}}\sqrt{\sum_{h=1}^H \sum_{s,a} p_h^{t+1}(s,a)\big(\frac{\alpha^\star\big(n_h^t(s,a)}{n_h^t(s,a)} \wedge 1 \big)}\\& + 84 e^{13} e^{\beta H}\sum_{h=1}^H \sum_{s,a} p_h^{t+1}(s,a)\big(\frac{\alpha\big(n_h^t(s,a)}{n_h^t(s,a)} \wedge 1 \big)
\end{align*}
Now to compute the regret we sum up to the $K$ episodes and using Cauchy-Schwartz inequality:
\begin{align*}
    R(K)  &\leq 36 e^{13} \sum_{t=1}^{K - 1} \frac{1}{\beta}\sqrt{\frac{(e^{\beta G_{\text{max}}(\mathcal{M})} - 1)^2}{e^{\beta G_{\text{max}}(\mathcal{M})}}\sum_{h=1}^H \sum_{s,a} p_h^{t+1}(s,a)\big(\frac{\alpha^\star\big(n_h^t(s,a)}{n_h^t(s,a)} \wedge 1 \big)} \\&+ 84 e^{13} \frac{1}{\beta} e^{\beta H}\sum_{t=1}^{K -1}\sum_{h=1}^H \sum_{s,a} p_h^{t+1}(s,a)\big(\frac{\alpha\big(n_h^t(s,a)}{n_h^t(s,a)} \wedge 1 \big)\\
    &\leq 36 e^{13}\frac{1}{\beta} \sqrt{\frac{(e^{\beta G_{\text{max}}(\mathcal{M})} - 1)^2}{e^{\beta G_{\text{max}}(\mathcal{M})}}}\sqrt{K}\sqrt{\sum_{t=1}^{K - 1}\sum_{h=1}^H \sum_{s,a} p_h^{t+1}(s,a)\big(\frac{\alpha^\star\big(n_h^t(s,a)}{n_h^t(s,a)} \wedge 1 \big)} \\&+ 84 e^{13} \frac{1}{\beta} e^{\beta H}\sum_{t=1}^{K -1}\sum_{h=1}^H \sum_{s,a} p_h^{t+1}(s,a)\big(\frac{\alpha\big(n_h^t(s,a)}{n_h^t(s,a)} \wedge 1 \big)
\end{align*}
Then using the pseudo-counts lemma we have:
\begin{align*}
 R(K)  &\leq 36 e^{13}  \frac{1}{\beta }\sqrt{\frac{(e^{\beta G_{\text{max}}(\mathcal{M})} - 1)^2}{e^{\beta G_{\text{max}}(\mathcal{M})}}K SAH\Big(\log\big(\frac{3 SAH}{\delta}\big)\log\big(8 e K \big) + \log\big( 8 e K\big)^2\Big)}\\& + 84 e^{13} \frac{e^{\beta H}}{\beta}SAH\Big(\log\big(\frac{3 SAH}{\delta}\big)\log\big(8 e K \big) + S\log\big( 8 e K\big)^2\Big)
\end{align*}
Hence omitting constants we have:
$$R(K) = \mathcal{O}\left( \frac{1}{|\beta|} \sqrt{ \frac{(e^{|\beta |G_{\max}(\mathcal{M})} - 1)^2}{e^{|\beta| G_{\max}(\mathcal{M})}} KSAH \log\left( \frac{SAHK}{\delta} \right) } + \frac{e^{|\beta| H}}{\beta} S^2 AH \log^2\left( \frac{SAHK}{\delta} \right) \right).$$
\end{proof}
Now we prove the lower bound, that we state first as a theorem:
\begin{theorem}
Fix $\beta \neq 0$. Let $\mathcal{M}_\varepsilon$ be the hard family used in the BPI lower bound and denote by $G_{\max}$ the maximal achievable return of the instances in this family (which does not depend on $\varepsilon$). Then, for every regret algorithm $\mathcal{A}$ and every $K \geq 2e^{|\beta|G_{\max}}SAH$ we have:
there exists an instance $M \in \mathfrak{M}_{\varepsilon_K}$(for a well chosen $\varepsilon_K)$ and a constant $\eta > 0$ such that
$$
\mathbb{E}_M[R_M(K)] \ge \frac{\eta}{6400|\beta|}\sqrt{\frac{(e^{|\beta|G_{\max}} - 1)^2}{e^{|\beta|G_{\max}}}SAHK}$$
Consequently
$$
\inf_{\mathcal{A}} \sup_M \mathbb{E}_M[R_M(K)] \ge \frac{\eta}{32|\beta|}\sqrt{\frac{(e^{|\beta|G_{\max}} - 1)^2}{e^{|\beta|G_{\max}}}SAHK}.
$$
\end{theorem}
\begin{proof}
Fix an arbitrary regret algorithm $\mathcal{A}$. When $\mathcal{A}$ is run for $K$ episodes on an instance $M$, let $\pi^1, \pi^2, \ldots, \pi^K$ be the policies played by the algorithm.
The cumulative regret is
$$R_M(K) = \sum_{t=1}^K \left( V_1^{\star, M}(s_1) - V_1^{\pi^t, M}(s_1) \right).$$
Choose $$\varepsilon_K = \frac{\eta}{|\beta|}\sqrt{\frac{(e^{|\beta|G_{\max}} - 1)^2}{e^{|\beta|G_{\max}}K}SAH}$$
Where $\eta$ is a constant such that:
$\eta \le \sqrt{2}$ and $\eta^2 \le \frac{c_{\text{BPI}}e^{-2}\log(32)}{8}$ (with $c_{\text{BPI}} = \frac{1}{1650}$). The condition on $K$ allow us to apply the lower bound on the BPI with $\varepsilon = \varepsilon_K$. Now suppose for contradiction that: 
Now suppose, toward a contradiction, that
$$\sup_{M \in \mathfrak{M}_{\varepsilon_K}} \mathbb{E}_M[R_M(K)] < \delta_0 K \varepsilon_K$$
We construct a fixed-budget BPI algorithm from $\mathcal{A}$: Run $\mathcal{A}$ for exactly $K$ episodes. After these $K$ episodes, sample
$J \sim \text{Unif}\{1, \dots, K\}$ independently of the interaction history, and output $\widehat{\pi} = \pi^J.$

For any $M \in \mathfrak{M}_{\varepsilon_K}$, by independence of $J$ and linearity of expectation,
$$\begin{aligned}
\mathbb{E}_M \left[ V_1^{\star, M}(s_1) - V_1^{\widehat{\pi}, M}(s_1) \right] &= \frac{1}{K} \sum_{t=1}^K \mathbb{E}_M \left[ V_1^{\star, M}(s_1) - V_1^{\pi^t, M}(s_1) \right] \\
&= \frac{\mathbb{E}_M[R_M(K)]}{K} \\
&< \delta_0 \varepsilon_K
\end{aligned}$$
Since $V_1^{\star, M}(s_1) - V_1^{\widehat{\pi}, M}(s_1) \ge 0$ Markov inequality gives
$$\mathbb{P}_M \left( V_1^{\star, M}(s_1) - V_1^{\widehat{\pi}, M}(s_1) > \varepsilon_K \right) \le \frac{\mathbb{E}_M[V_1^{\star, M}(s_1) - V_1^{\widehat{\pi}, M}(s_1)]}{\varepsilon_K} < \delta_0$$
Therefore,
$$\mathbb{P}_M \left( V_1^{\star, M}(s_1) - V_1^{\widehat{\pi}, M}(s_1) \le \varepsilon_K \right) \ge 1 - \delta_0$$
We will show now this contradicts the BPI lower bound. The running time is deterministic and satisfies $\mathbb{E}_{\mathcal{M}_0}[\tau] = K$
We now show that this contradicts the BPI lower bound.
Since
$|\beta|\varepsilon_K \le 1$ by the condition on $K$ and $\eta$,
we have
$$e^{2\min\{\beta, 0\}\varepsilon_K} \ge e^{-2|\beta|\varepsilon_K} \ge e^{-2}$$
Also, using the inequality $e^x - 1 \le 2x$ for $x \in [0, 1]$, we obtain
$$e^{|\beta|\varepsilon_K} - 1 \le 2|\beta|\varepsilon_K$$
Applying the BPI lower bound with $\delta_0 = 1/32$, we get$$\begin{aligned}
T_{\text{BPI}}(\varepsilon_K, \delta_0) &\ge c_{\text{BPI}} \frac{(e^{|\beta|G_{\max}} - 1)^2}{e^{|\beta|G_{\max}}} \frac{e^{2\min\{\beta, 0\}\varepsilon_K}SAH}{(e^{|\beta|\varepsilon_K} - 1)^2} \log(32) \\
&\ge c_{\text{BPI}} \frac{(e^{|\beta|G_{\max}} - 1)^2}{e^{|\beta|G_{\max}}} \frac{e^{-2}SAH}{4\beta^2\varepsilon_K^2} \log(32)
\end{aligned}$$By definition of $\varepsilon_K$,$$\beta^2\varepsilon_K^2 = \eta^2 \frac{(e^{|\beta|G_{\max}} - 1)^2}{e^{|\beta|G_{\max}}} \frac{SAH}{K}$$
Substituting this into the previous display yields$$T_{\text{BPI}}(\varepsilon_K, \delta_0) \ge \frac{c_{\text{BPI}}e^{-2}\log(32)}{4\eta^2}K$$
Since by the definition of $\eta$ we have:
$$\frac{c_{\text{BPI}}e^{-2}\log(32)}{4\eta^2} \ge 2$$
Thus
$$K = \mathbb{E}_{\mathcal{M}_0}[\tau] \geq T_{\text{BPI}}(\varepsilon_K, \delta_0) \ge 2K > K$$
This contradicts the existence of the fixed-budget $(\varepsilon_K, \delta_0)$-PAC BPI algorithm using exactly $K$ episodes.
Therefore, the contradiction assumption must be false. Hence
$$\sup_{M \in \mathfrak{M}_{\varepsilon_K}} \mathbb{E}_M[R_M(K)] \ge \delta_0 K \varepsilon_K$$
Substituting the definition of $\varepsilon_K$, we get
$$\begin{aligned}
\sup_{M \in \mathfrak{M}_{\varepsilon_K}} \mathbb{E}_M[R_M(K)] &\ge \delta_0 K \frac{\eta}{|\beta|} \sqrt{\frac{(e^{|\beta|G_{\max}} - 1)^2}{e^{|\beta|G_{\max}}}\frac{SAH}{K}} \\
&= \frac{\eta}{32|\beta|} \sqrt{\frac{(e^{|\beta|G_{\max}} - 1)^2}{e^{|\beta|G_{\max}}}SAHK}
\end{aligned}$$
Finally, since $\mathfrak{M}_{\varepsilon_K}$ is a subset of the full MDP class under consideration, and since $\mathcal{A}$ was arbitrary,
$$\inf_{\mathcal{A}} \sup_M \mathbb{E}_M[R_M(K)] \ge \frac{\eta}{32|\beta|}\sqrt{\frac{(e^{|\beta|G_{\max}} - 1)^2}{e^{|\beta|G_{\max}}}SAHK}$$
\end{proof}
\section{Experiments}\label{app:experiments}
For our experiments, We consider a toy MDP  that consists of 8-states and two actions, \textbf{safe} and \textbf{risky}. Starting from $s_0$, \textbf{safe} deterministically walks along the bridge $s_0 \to s_1 \to s_2 \to s_3 \to s_4 \to s_5 \to s_g$, where $s_g$ is an absorbing goal state. The \textbf{risky} action attempts shortcuts: from early states it jumps directly to $s_4$ but can fall into an absorbing bad state $s_b$ with state-dependent probability (start/mid/final risk); from $s_4$ it makes a final dash to $s_g$ that can also fail and transition to $s_b$. Both $s_g$ and $s_b$ are absorbing and we receive reward $1$ in $s_g$ for the rest of the horizon. 
\begin{figure}[h]
\centering
\begin{tikzpicture}[
  >={Stealth[length=2.2mm]},
  font=\small,
  st/.style={circle,draw,minimum size=10mm,inner sep=0pt},
  safe/.style={st,fill=yellow!70},
  good/.style={st,fill=green!18},
  bad/.style={st,fill=red!15},
  safearrow/.style={->,draw=black,thick},
  riskyarrow/.style={->,draw=red,very thick,dashed},
  lab/.style={midway,sloped,above,fill=white,inner sep=1pt},
  labd/.style={midway,sloped,below,fill=white,inner sep=1pt}
]

\node[safe] (s0) at (-6,1.6) {$s_0$};
\node[safe] (s1) at (-4,1.6) {$s_1$};
\node[safe] (s2) at (-2,1.6) {$s_2$};
\node[safe] (s3) at ( 0,1.6) {$s_3$};
\node[safe] (s4) at ( 2,1.6) {$s_4$};
\node[safe] (s5) at ( 4,1.6) {$s_5$};
\node[good] (sg) at ( 6,1.6) {$s_g$};

\node[bad]  (sb) at ( 0,-1.6) {$s_b$};

\draw[safearrow] (s0) -- node[lab] {1} (s1);
\draw[safearrow] (s1) -- node[lab] {1} (s2);
\draw[safearrow] (s2) -- node[lab] {1} (s3);

\draw[safearrow] (s3) -- node[lab] {1} (s4);
\draw[safearrow] (s4) -- node[lab] {1} (s5);
\draw[safearrow] (s5) -- node[lab] {1} (sg);

\draw[riskyarrow]
  (s0) to[bend left=50]
  node[lab] {$1-p_{\text{risk,start}}$} (s4);

\draw[riskyarrow]
  (s0) to[bend right=18]
  node[labd] {$p_{\text{risk,start}}$} (sb);

\draw[riskyarrow]
  (s1) to[bend left=40]
  node[lab] {$1-p_{\text{risk,mid}}$} (s4);

\draw[riskyarrow]
  (s1) to[bend right=10]
  node[labd] {$p_{\text{risk,mid}}$} (sb);

\draw[riskyarrow]
  (s2) to[bend left=30]
  node[lab] {$1-p_{\text{risk,final}}$} (s4);

\draw[riskyarrow]
  (s2) to[bend right=6]
  node[labd] {$p_{\text{risk,final}}$} (sb);

\draw[riskyarrow]
  (s4) to[bend left=40]
  node[lab] {$1-p_{\text{fall,dash}}$} (sg);

\draw[riskyarrow]
  (s4) to[bend right=22]
  node[labd] {$p_{\text{fall,dash}}$} (sb);

\draw[->] (sg) edge[loop above] node[above] {$1$} (sg);
\draw[->] (sb) edge[loop below] node[below] {$1$} (sb);

\end{tikzpicture}
\label{toy MDP}
\caption{toy MDP. Safe action follows the straight chain to $s_g$; risky action can shortcut but may fall to $s_b$. \textbf{safe} action is presented by a black edge and \textbf{risky} action by a dashed red edge.}
\label{fig:risky-bridge-mdp}
\end{figure}
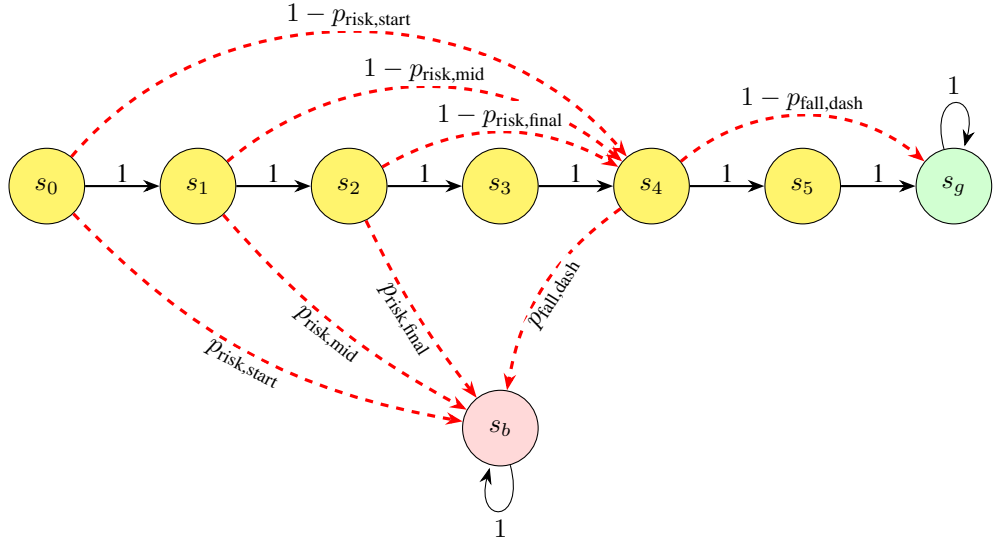
For analogy, think of it as a child standing at the top of a long staircase. Taking safe means walking down (or along) the stairs one step at a time, steadily progressing. Taking risky means you try to jump over several steps at once to land much farther ahead saving time, but with some chance of missing the landing and falling into a pit (failure). After the staircase you reach a narrow bridge: safe is crossing it normally, while risky is a final dash across the bridge, which is faster but again risks falling into the pit. If you make it, you end in the goal; if you fall, you’re stuck in failure. The MDP is visualized in the figure ~\ref{fig:risky-bridge-mdp}. In our experiments we take $\Big\{p_{\text{risk,start}},p_{\text{risk,mid}},p_{\text{risk,final}},p_{\text{fall,dash}}\Big\} = \Big\{ 0.95,0.75,0.25,0.85\Big\} $, we also take $\varepsilon = 0.2$ and $\delta = 0.1$
\subsection{Sample complexity}
We study the sample complexity of our method as a function of the risk sensitivity $\beta$ and the horizon $H$. In the first sweep, we fix $H = 7$ and run the algorithm for $\beta \in \{2.5, 3.0, 3.5, 4.0, 4.5\}$ until the stopping condition is met. In the second sweep, we fix $\beta = 0.5$ and vary the horizon $H \in \{5, 7, 9, 11, 13, 15, 17\}$. For each configuration, we run 15 independent trials (different random seeds) and plot the mean stopping time $\tau$; the $y$-axis is shown on a log scale.
\begin{figure}[h]
    \centering
    \begin{minipage}{0.48\textwidth}
        \centering
        \includegraphics[width=\textwidth]{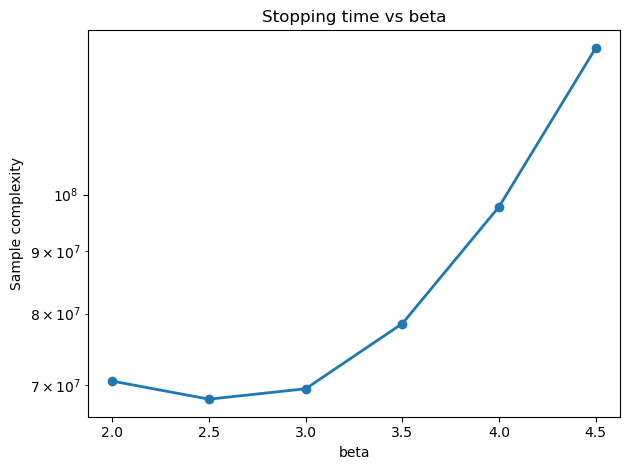}
        \caption{Sample complexity in function of $\beta$ (log-scale y-axis)}
        \label{fig:first}
    \end{minipage}\hfill
    \begin{minipage}{0.48\textwidth}
        \centering
        \includegraphics[width=\textwidth]{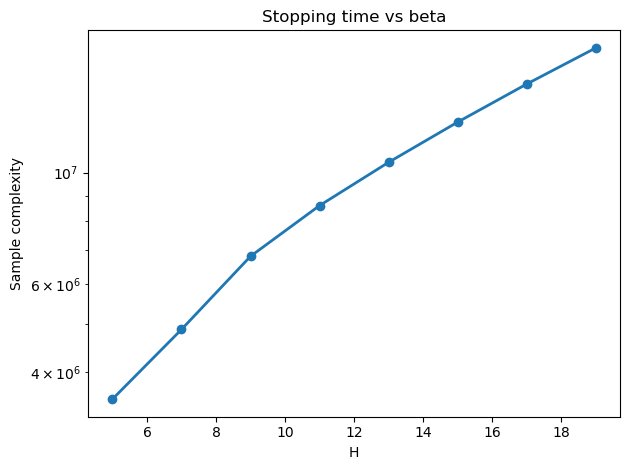}
        \caption{Sample complexity in function of $\beta$ (log-scale y-axis)}
        \label{fig:second}
    \end{minipage}
\end{figure}
For Figure~\ref{fig:second}, the stopping time $\tau$ increases sharply as $\beta$ grows. Moreover, the approximately linear trend on the log-scale $y$-axis suggests that $\log \tau$ increases roughly linearly with $\beta$, i.e., the sample complexity is consistent with an exponential scaling $\tau \approx \exp(c\,\beta)$ over the tested range. This is in line with the intuition behind the entropic objective, where larger $\beta$ amplifies tail events and makes BPI more demanding.

Interestingly, the effective slope $c$ is not constant across $\beta$. We hypothesize that this variation is driven by the $\chi^2$-divergence term $\chi^2(\mathbb{P}_\beta^\pi,\mathbb{P}^\pi)$ discussed in Paragraph~\ref{insight}, which depends on the policy induced at each $\beta$. In our sweep, the learned policy changes substantially as $\beta$ increases, reflecting policy changes as the agent shows increasingly risk-seeking behavior, and then stabilizes around $\beta \approx 3.5$. Beyond this point, further increases in $\beta$ do not change the policy, which may explain the change in scaling behavior. For Figure~4, we observe a similarly sharp increase in the stopping time $\tau$ as the horizon $H$ grows and is roughly linear on the log scale
\subsection{Regret bounds}
We now compare our approach to other regret algorithms. At each episode, we evaluate the current policy returned by each algorithm from the initial state $s_0$. We then compute the cumulative regret for each algorithm defined by:
$$R(K) = \sum_{k=1}^K V^\star(s_0) - V^{\pi^{k+1}}(s_0)$$
Where $\pi^{k+1}$ is the algorithm returned by the algorithm at the end of episode $k$. We cap the stopping time at $10^7$ episodes and draw the regret for $\beta \in [1.0,2.0,3.0,4.0]$. We compare our method to multiple regret algorithms:
\begin{itemize}
    \item \textsc{RSVI} and \textsc{RSQ} from \citep{Fei2020RS}
    \item \textsc{RSVI2} and \textsc{RSQ2} from \citep{FeiYCW21}
    \item \textsc{UCB Advantage} from \citep{hu23b}
    \item  \textsc{RODI} (OTP and PTO variants) and \textsc{ROVI} from \citep{liang2024bridging}
\end{itemize}
We report the findings of our experiment:
\begin{figure}[h]
    \centering
    \begin{minipage}{0.48\textwidth}
        \centering
        \includegraphics[width=\textwidth]{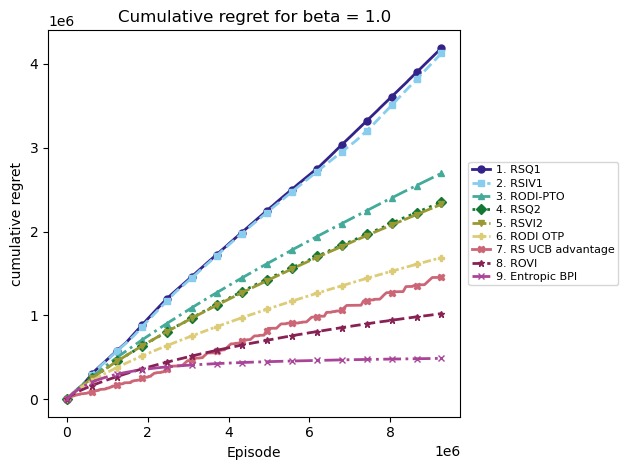}
        \caption{Regret for different algorithms for $\beta = 1.0$}
        \label{fig:beta1}
    \end{minipage}\hfill
    \begin{minipage}{0.48\textwidth}
        \centering
        \includegraphics[width=\textwidth]{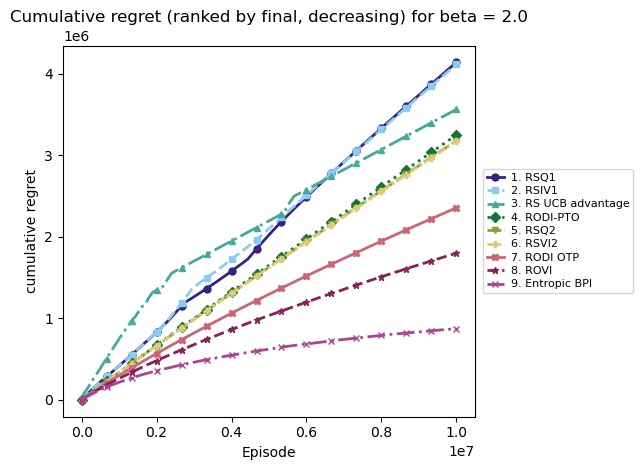}
        \caption{Regret for different algorithms for $\beta = 2.0$}
        \label{fig:beta2}
    \end{minipage}
\end{figure}

\begin{figure}[h]
    \centering
    \begin{minipage}{0.48\textwidth}
        \centering
        \includegraphics[width=\textwidth]{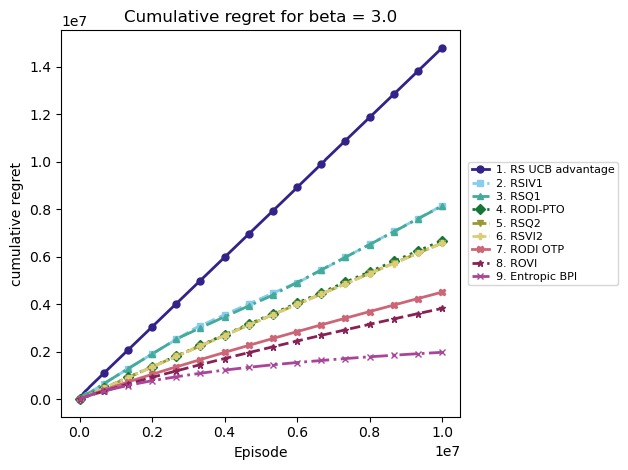}
        \caption{Regret for different algorithms for $\beta = 3.0$}
        \label{fig:beta3}
    \end{minipage}\hfill
    \begin{minipage}{0.48\textwidth}
        \centering
        \includegraphics[width=\textwidth]{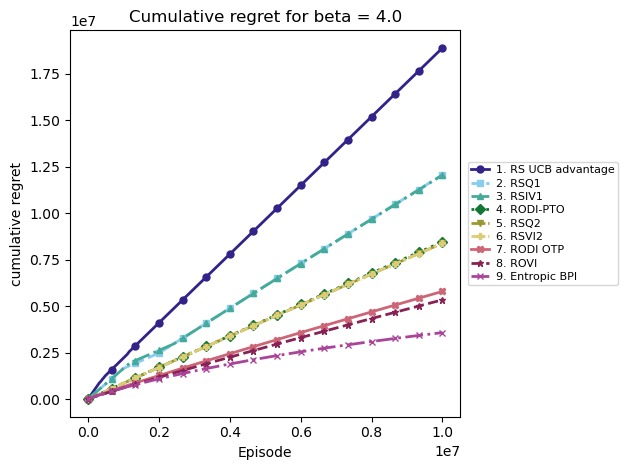}
        \caption{Regret for different algorithms for $\beta = 4.0$}
        \label{fig:beta4}
    \end{minipage}
\end{figure}
Figures ~\ref{fig:beta1},~\ref{fig:beta2},~\ref{fig:beta3},~\ref{fig:beta4} shows that our algorithm consistently achieves the lowest cumulative-regret trajectory over all values of $\beta$. Since cumulative regret is computed by evaluating the current policy from the initial state each episode and summing the resulting value gap to optimality, the consistently flatter Entropic BPI curve indicates stronger learning efficiency and performance. 
\section{Concentration inequalities}
\subsection{Sanov's theorem}\label{Sanov}
First we introduce the K divergence concentration inequality derived via Sanov's theorem
\begin{lemma}[High-probability KL bound via Sanov]
Let $\Sigma_m = \{q\in\mathbb{R}^m_{\ge 0}:\sum_{i=1}^m q_i=1\}$ and let $p\in\Sigma_m$.
Let $X_1,\dots,X_n$ be i.i.d.\ with law $p$, and let the empirical distribution $\widehat{p}_n$ be
\[
\widehat{p}_n(i)=\frac{1}{n}\sum_{k=1}^n \mathbf{1}\{X_k=i\},\qquad i\in[m].
\]
Then for any $\delta\in(0,1)$, with probability at least $1-\delta$,
\[
D(\widehat{p}_n\|p)\leq\frac{(m-1)\log(n+1)+\log(1/\delta)}{n}.
\]
\end{lemma}
\begin{proof}
Let $\Sigma_m = \Big\{ q \in \mathbb{R}^m_{\ge 0} : \sum_{i=1}^m q_i = 1 \Big\}$ be the $(m-1)$-simplex and let $p \in \Sigma_m$. 
Let $X_1,\ldots,X_n$ be i.i.d.\ with law $p$, and let $\widehat{p}_n$ denote the empirical distribution by
$$
\widehat{p}_n(i) := \frac{1}{n}\sum_{k=1}^n \mathbf{1}\{X_k=i\}, \qquad i\in[m]
$$
Sanov's theorem (Theorem 11.4.1 \cite{CoverThomas2006}) states that for any set $E \subseteq \Sigma_m$ 
$$
\Pr \big(\widehat{p}_n \in E\big)\leq(n+1)^{(m-1)}\exp\Big(-n\inf_{q\in E} D(q\|p)\Big)
$$
Fix $\varepsilon > 0 $ and take the set $E = \{ q \in \Sigma_m| D(p||q) \geq \varepsilon\}$ then $\inf_{q\in E} D(q\|p) = \varepsilon$ and by Sanov's theorem: 
$$
\Pr \big(D(\widehat{p}_n||p) \geq \varepsilon\big) \leq (n+1)^m\exp\Big(-n\varepsilon\Big)
$$
We turn it to high probability bound by setting the r.h.s to $\delta$ which yield:
$$\varepsilon = \frac{m \log(n+1) + \log(1/\delta)}{n}$$
Doing a union bound on all state action pairs and all time steps we get that with probability $1 - \delta$ we have:
$$
D(\widehat{p}_n||p) \leq \frac{(m-1) \log(n+1) + \log(SAH/\delta)}{n} 
$$
\end{proof}
\subsection{Concentration inequality for Bernoulli random variables}
    We state a deviation inequality for Bernoulli random variables from \cite{DannLB17}(Lemma F.4):
\begin{lemma}\label{Bernoulli}
Let $\mathcal{F}_i$ for $i = 1 \dots$ be a filtration and $X_1, \dots X_n$ be a sequence of Bernoulli random variables with $\mathbb{P}(X_i = 1|\mathcal{F}_{i-1}) = P_i$ with $P_i$ being $\mathcal{F}_{i-1}$-measurable and $X_i$ being $\mathcal{F}_i$ measurable. It holds that
$$
\mathbb{P} \left( \exists n : \sum_{t=1}^{n} X_t < \sum_{t=1}^{n} P_t/2 - \log(\frac{1}{\delta}) \right) \leq \delta
$$
\end{lemma}
\subsection{Self-normalized Bernstein inequality}
We state  the self-normalized Bernstein-type inequality by \citep{Menard2021Fast}.
\begin{lemma}\label{Bernstein}
Let $(Y_t)_{t \in \mathbb{N}^\star}, (w_t)_{t \in \mathbb{N}^\star}$ be two sequences of random variables adapted to a filtration $(\mathcal{F}_t)_{t \in \mathbb{N}}$. We assume that the weights are in the unit interval $w_t \in [0, 1]$ and predictable, i.e. $\mathcal{F}_{t-1}$ measurable. We also assume that the random variables $Y_t$ are bounded $|Y_t| \leq b$ and centered $\mathbb{E}[Y_t | \mathcal{F}_{t-1}] = 0$. Consider the following quantities

$$
S_t \triangleq \sum_{s=1}^t w_s Y_s, \quad V_t \triangleq \sum_{s=1}^t w_s^2 \cdot \mathbb{E}[Y_s^2 | \mathcal{F}_{s-1}], \quad \text{and} \quad W_t \triangleq \sum_{s=1}^t w_s
$$

and let $h(x) \triangleq (x+1)\log(x+1) - x$ be the Cramér transform of a Poisson distribution of parameter 1. Then For all $\delta > 0$:

$$
\mathbb{P}\left( \exists t \geq 1, (V_t/b^2 + 1) h \left( \frac{b|S_t|}{V_t + b^2} \right) \geq \log(1/\delta) + \log(4e(2t+1)) \right) \leq \delta.
$$

\textit{The previous inequality can be weakened to obtain a more explicit bound: with probability at least $1 - \delta$, for all $t \geq 1$,}

$$
|S_t| \leq \sqrt{2 V_t \log(4e(2t+1)/\delta)} + 3b \log(4e(2t+1)/\delta).
$$
\end{lemma}
\subsection{KL-Bernstein inequality}
We state a KL-Bernstein inequality from \citep{talebi18a}:
\begin{lemma}\label{KLBernstein}
 Let $p, q \in \Sigma_S$, where $\Sigma_S$ denotes the probability simplex of dimension $S-1$. For all $\alpha > 0$, for all functions $f$ defined on $\mathcal{S}$ with $0 \le f(s) \le b$, for all $s \in \mathcal{S}$, if $\mathrm{KL}(p, q) \le \alpha$ then
$$
|pf - qf| \le \sqrt{2\mathrm{Var}_q(f)\alpha} + \frac{2}{3}b\alpha
$$
\end{lemma}
where we use the expectation operator defined as $pf \triangleq \mathbb{E}_{s \sim p}f(s)$ and the variance operator defined as $\mathrm{Var}_p(f) \triangleq \mathbb{E}_{s \sim p}(f(s) - \mathbb{E}_{s' \sim p}f(s'))^2 = p(f - pf)^2$.
\section{Technical results}\label{Technical_results}
The next lemma introduces the entropic variance under a fixed policy $\pi$, defined as the conditional variance of the exponentiated return $e^{\beta G_h}$ around its entropic value $e^{\beta Q_h^\pi}$. It shows that $\sigma Q_h^\pi$ satisfies a Bellman-style recursion: for each step $h$ and state--action pair $(s,a)$, $\sigma Q_h^\pi(s,a)$ decomposes into (i) the variance under the next-state transition $p_h(\cdot\mid s,a)$ of $e^{\beta(r_h(s,a)+V_{h+1}^\pi(s'))}$ and (ii) the expected next-step entropic variance scaled by $e^{2\beta r_h(s,a)}$ with terminal condition $\sigma_{H+1}^\pi V_h^\pi=0$.
\begin{lemma}[Entropic variance recursion]\label{entropicvariance}
Fix a (deterministic) policy $\pi$.$$\sigma Q_h^\pi(s,a) = \mathbb{E}^\pi \Bigg[ \Bigg(e^{\beta R_h^\pi(s,a)} - e^{\beta Q_h^\pi(s,a)}\Bigg)^2\quad \Bigg| S_h = s,A_h=a\Bigg] $$
    with terminal condition $\sigma_{H+1}^\pi(s,a) = 0$ and:
    $$\sigma V_h^\pi(s) = \sigma Q_h^\pi(s,\pi(s))$$

Then, for every $h\in\{1,...,H\}$ and $(s,a)\in\mathcal S\times\mathcal A$,
\begin{equation}\label{entropicvarianceequation}
\sigma Q_h^\pi(s,a)
=
e^{2\beta r_h(s,a)} \,\operatorname{Var}_{S'\sim p_h(\cdot\mid s,a)}\Big(Z_{h+1}^\pi(S')\Big)
+
e^{2\beta r_h(s,a)} \,(p_h\,\sigma V_{h+1}^\pi)(s,a)
\end{equation}
\end{lemma}
\begin{proof}
Fix $h\in\{1,...,H\}$ and $(s,a)$.
By definition of $U_h^\pi=e^{\beta Q_h^\pi}$ (exponential entropic $Q$), we have
\[
U_h^\pi(s,a)
=
\mathbb{E}^\pi\left[e^{\beta R_h^\pi}\mid S_h=s,\ A_h=a\right],
\]
hence $\sigma Q_h^\pi(s,a)=\operatorname{Var}\left(e^{\beta R_h^\pi}\mid S_h=s, A_h=a\right)$.

Since $r_h(s,a)$ is deterministic given $(s_h,a_h)=(s,a)$,
\[
e^{\beta R_h^\pi}=e^{\beta r_h(s,a)}\,e^{\beta R_{h+1}^\pi}
\quad\Rightarrow\quad
\sigma Q_h^\pi(s,a)
=
e^{2\beta r_h(s,a)}\operatorname{Var}\left(e^{\beta R_{h+1}^\pi}\mid s_h=s,\ a_h=a\right)
\]

Apply the law of total variance w.r.t. $S_{h+1}$:
\begin{align*}
\operatorname{Var}\left(e^{\beta R^\pi_{h+1}}\mid s_h=s,\ a_h=a\right)
&=
\mathbb{E}\left[\operatorname{Var}\left(e^{\beta R^\pi_{h+1}}\mid S_{h+1} = s_{h+1}\right)\mid S_h=s,\ A_h=a\right]\\
&\quad+
\operatorname{Var}\left(\mathbb{E}\left[e^{\beta R^\pi_{h+1}}\mid S_{h+1} = s_{h+1}\right]\mid S_h=s,\ A_h=a\right)
\end{align*}

For the first term, conditioning on $S_{h+1}=s'$ and following $\pi$ thereafter gives
$$
\operatorname{Var}\left(e^{\beta R^\pi_{h+1}}\mid S_{h+1}=s'\right)=\sigma V_{h+1}^\pi(s')
$$
so
$$
\mathbb{E}\left[\operatorname{Var}\left(e^{\beta R^\pi_{h+1}}\mid S_{h+1}\right)\mid S_h=s,\ A_h=a\right]
=
\mathbb{E}_{S'\sim p_h(\cdot\mid s,a)}\left[\sigma V_{h+1}^\pi(S')\right]
=
(p_h\,\sigma V_{h+1}^\pi)(s,a)
$$

For the second term, by definition of $Z_{h+1}^\pi(s')=e^{\beta V_{h+1}^\pi(s')}
=\mathbb{E}^\pi[e^{\beta G_{h+1}}\mid s_{h+1}=s']$,
$$
\operatorname{Var}\left(\mathbb{E}\left[e^{\beta G_{h+1}}\mid S_{h+1}\right]\mid S_h=s,\ A_h=a\right)
=
\operatorname{Var}_{S'\sim p_h(\cdot\mid s,a)}\Big(Z_{h+1}^\pi(S')\Big)
$$

Multiplying by $e^{2\beta r_h(s,a)}$ yields
$$
\sigma Q_h^\pi(s,a)
=
e^{2\beta r_h(s,a)} \operatorname{Var}_{S'\sim p_h(\cdot\mid s,a)}\Big(Z_{h+1}^\pi(S')\Big)
+
e^{2\beta r_h(s,a)} (p_h\,\sigma V_{h+1}^\pi)(s,a)
$$
as claimed.
\end{proof}
\begin{lemma}\label{bound}[A normalized variance bound for an exponential transform]
Let $f:\mathcal X\to[0,R]$ and let $\beta\in\mathbb R$. Define
$$ Y = e^{\beta f(X)}$$
Set
$$
m = e^{\min\{0,\beta R\}}=\min\{1,e^{\beta R}\},
\qquad
M = e^{\max\{0,\beta R\}}=\max\{1,e^{\beta R}\}.
$$
Then $Y\in[m,M]$ a.s., $\mu\in[m,M]$, and
\begin{align*}
\frac{\operatorname{Var}(Y)}{\mu^2}
 \leq \frac{(e^{|\beta|R}-1)^2}{4e^{|\beta|R}} 
\end{align*}
\end{lemma}

\begin{proof}
If $\beta=0$ then $Y\equiv 1$ and $\operatorname{Var}=0$, so assume $\beta\neq 0$.
Since $f(X)\in[0,R]$, we have $\beta f(X)\in[\min\{0,\beta R\},\max\{0,\beta R\}]$, hence
$$
m \le Y=e^{\beta f(X)} \le M \quad\text{a.s.}
$$
and therefore $\mu=\mathbb E[Y]\in[m,M]$.

By the Bhatia--Davis inequality: 
$$
\operatorname{Var}(Y) \leq (M-\mu)(\mu-m)
$$
Dividing by $\mu^2>0$ yields 
$$
\frac{\operatorname{Var}(Y)}{\mu^2} \leq \frac{(M-\mu)(\mu-m)}{\mu^2}
$$
Now consider
$$
g(\mu)=\frac{(M-\mu)(\mu-m)}{\mu^2}, \qquad \mu\in[m,M]
$$
Rewrite
$$
g(\mu)=\frac{M+m}{\mu}-\frac{Mm}{\mu^2}-1,
\qquad
g'(\mu)= -\frac{M+m}{\mu^2}+\frac{2Mm}{\mu^3}
=\frac{2Mm-(M+m)\mu}{\mu^3}
$$
Thus $g'(\mu)=0$ iff $\mu^\star=\frac{2Mm}{M+m}\in[m,M]$, and $g$ increases on
$[m,\mu^\star]$ and decreases on $[\mu^\star,M]$. Hence the maximum is attained
at $\mu^\star$. Substituting gives
$$
g(\mu^\star)
=\frac{\bigl(M-\frac{2Mm}{M+m}\bigr)\bigl(\frac{2Mm}{M+m}-m\bigr)}
{\bigl(\frac{2Mm}{M+m}\bigr)^2}
=\frac{(M-m)^2}{4Mm}
$$
so for all $\mu\in[m,M]$
$$
\frac{(M-\mu)(\mu-m)}{\mu^2}\le \frac{(M-m)^2}{4Mm}
$$
Finally, since $M/m=e^{|\beta|R}\ge 1$, we have
$$
\frac{(M-m)^2}{4Mm}=\frac{\left(\frac{M}{m}-1\right)^2}{4\frac{M}{m}}
=\frac{(e^{|\beta|R}-1)^2}{4e^{|\beta|R}}
$$
\end{proof}
We state lemma 11 and 12 from \citep{Menard2021Fast} that are used for the variance transportation:
\begin{lemma}\label{KL transportation}
Let $p, q \in \Sigma_S$ and $f$ is a function defined on $\mathcal{S}$ such that $0 \le f(s) \le b$ for all $s \in \mathcal{S}$. If $\mathrm{KL}(p, q) \le \alpha$ then
\begin{align*}
\operatorname{Var}_q(f) &\le 2\mathrm{Var}_p(f) + 4b^2\alpha \quad \text{and} \\
\operatorname{Var}_p(f) &\le 2\mathrm{Var}_q(f) + 4b^2\alpha
\end{align*}
\end{lemma}
\begin{lemma}\label{transportation} For $p, q \in \Sigma_S$, for $f, g$ two functions defined on $\mathcal{S}$ such that $0 \le g(s), f(s) \le b$ for all $s \in \mathcal{S}$, we have that
\begin{align*}
\mathrm{Var}_p(f) &\le 2\mathrm{Var}_p(g) + 2bp|f - g| \quad \text{and} \\
\mathrm{Var}_q(f) &\le \mathrm{Var}_p(f) + 3b^2\|p - q\|_1,
\end{align*}
where we denote the absolute operator by $|f|(s) = |f(s)|$ for all $s \in \mathcal{S}$.
\end{lemma}
We state the pseudo-counts lemma 7 that allows to go from counts to their mean the pseudo-counts and lemma 8 a standard inequality from \citep{Menard2021Fast}
\begin{lemma}\label{pseudo-counts}
On event $\mathcal{E}^{\mathrm{cnt}}$, for any $\alpha(\cdot,\delta)$ such that $x \mapsto \beta(\delta, x)/x$ is non-increasing for $x \ge 1$, $x \mapsto \beta(x,\delta)$ is non-decreasing $\forall h \in \{1,...,H\}, (s, a) \in \mathcal{S} \times \mathcal{A}$.
$$
\forall t \in \mathbb{N}^*, \quad \frac{\alpha(n_h^t(s,a),\delta)}{n_h^t(s,a)} \wedge 1 \le 4 \frac{\alpha(\bar{n}_h^t(s,a)),\delta}{\bar{n}_h^t(s,a) \vee 1}
$$
\end{lemma}

\begin{lemma}
For $T \in \mathbb{N}^*$ and $(u_t)_{t \in \mathbb{N}^*}$ for a sequence where $u_t \in [0, 1]$ and $U_t \triangleq \sum_{i=1}^t u_i$, we get
$$
\sum_{t=0}^T \frac{u_{t+1}}{U_t \vee 1} \le 4 \log(U_{T+1} + 1).
$$
\end{lemma}
Finally we state lemma 13 from \citep{Menard2021Fast}
\begin{lemma}\label{inequality}
Let $A, B, C, D, E$, and $\alpha$ be positive scalars such that $1 \le B \le E$ and $\alpha \ge e$. If $\tau \ge 0$ satisfies
\begin{equation}
\tau \le C\sqrt{\tau \left(A \log(\alpha\tau) + B \log(\alpha\tau)^2\right)} + D\left(A \log(\alpha\tau) + E \log(\alpha\tau)^2\right) 
\end{equation}
then
$$
\tau \le C^2(A+B)C_1^2 + \left(D + 2\sqrt{D}C\right)(A+E)C_1^2 + 1
$$
where
$$
C_1 = \frac{8}{5}\log\left(11\alpha^2(A+E)(C+D)\right)
$$
\end{lemma}

\end{document}